\documentclass[journal,10pt,twocolumn]{IEEEtran}
\usepackage{cite}
\usepackage{amsmath, amssymb, amsfonts, amsthm}
\usepackage{algorithm}
\usepackage[noend]{algorithmic}
\usepackage{graphicx}
\usepackage{textcomp}
\usepackage{xcolor}
\usepackage{color}
\usepackage{bm}
\usepackage{booktabs}
\usepackage{autobreak}
\usepackage{multirow}
\usepackage{booktabs}
\usepackage{authblk}
\usepackage{mathrsfs}
\usepackage{array}
\usepackage{makecell} 
\usepackage{float}
\usepackage{soul}
\usepackage{stmaryrd}
\usepackage{dsfont}
\usepackage{bbm}
\usepackage[hidelinks]{hyperref}
\usepackage{cleveref}
\usepackage{pifont}
\usepackage{enumitem}

\usepackage{placeins}

\usepackage{atbegshi}

\newif\ifshowappendix
\showappendixtrue

\newtheorem{theorem}{Theorem}
\newtheorem{proposition}{Proposition}

\newtheorem{corollary}{Corollary}

\definecolor{blue}{rgb}{0,0,0}
 
\begin{document}

\bstctlcite{BSTcontrol}

\title{Beyond Fixed False Discovery Rates: Post-Hoc Conformal Selection with E-Variables}

\author{Meiyi Zhu and Osvaldo Simeone \IEEEmembership{Fellow, IEEE}

\vspace{-0.5cm}

\thanks{The work of M. Zhu and O. Simeone was supported by an Open Fellowship of the EPSRC (EP/W024101/1). The work of O. Simeone was also supported by EPSRC (EP/X011852/1) and ERC (No. 101198347)

Meiyi Zhu is with the Department of Engineering, King's College London, WC2R 2LS, London, U.K. (e-mail: meiyi.1.zhu@kcl.ac.uk).

Osvaldo Simeone is with the Institute for Intelligent Networked Systems, Northeastern University London, E1 8PH London, U.K. (e-mail: o.simeone@northeastern.edu).}
}

\maketitle

\begin{abstract}
Conformal selection (CS) uses calibration data to identify test inputs whose unobserved outcomes are likely to satisfy a pre-specified minimal quality requirement, while controlling the false discovery rate (FDR). Existing methods fix the target FDR level before observing data, which prevents the user from adapting the balance between number of selected test inputs and FDR to downstream needs and constraints based on the available data. For example, in genomics or neuroimaging, researchers often inspect the distribution of test statistics, and decide how aggressively to pursue candidates based on observed evidence strength and available follow-up resources. To address this limitation, we introduce {post-hoc CS} (PH-CS), which generates a path of candidate selection sets, each paired with a data-driven false discovery proportion (FDP) estimate. PH-CS lets the user select any operating point on this path by maximizing a user-specified utility, arbitrarily balancing selection size and FDR. Building on conformal e-variables and the e-Benjamini-Hochberg (e-BH) procedure, PH-CS is proved to provide a finite-sample post-hoc reliability guarantee whereby the ratio between estimated FDP level and true FDP is, on average, upper bounded by $1$, so that the average estimated FDP is, to first order, a valid upper bound on the true FDR. PH-CS is extended to control quality defined in terms of a general risk. Experiments on synthetic and real-world datasets demonstrate that, unlike CS, PH-CS can consistently satisfy user-imposed utility constraints while producing reliable FDP estimates and maintaining competitive FDR control.
\end{abstract}

% \vspace{-4mm}
\begin{IEEEkeywords}
    Conformal selection, false discovery rate, post-hoc, e-variables, multiple hypothesis testing
\end{IEEEkeywords}

\begin{figure}[t]
    \centering
    {\includegraphics[width = 0.48\textwidth]{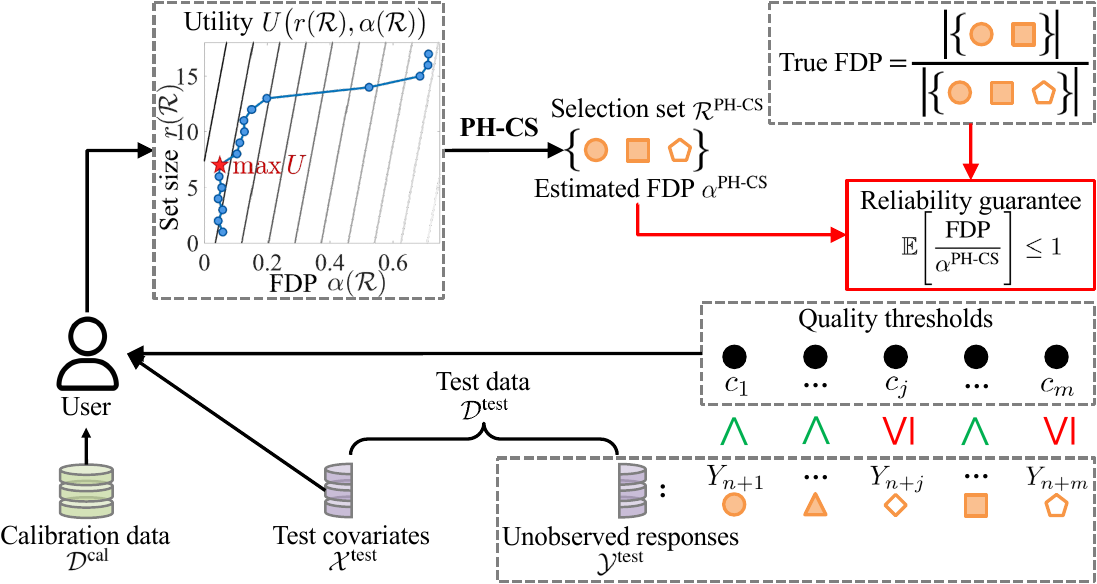}}
    \vspace{-1.2mm} % 0.68
    \caption{{\color{blue}Illustration of the PH-CS problem and framework.} A user has access to labeled calibration data $\mathcal{D}^{\textrm{cal}}$ and unlabeled test covariates $\mathcal{X}^{\textrm{test}}$, while the test responses $\mathcal{Y}^{\textrm{test}}$ remain unobserved. Each test input $X_{n+j}$ is associated with a threshold $c_j$ representing the minimum quality requirement \eqref{eq_perform_req}. Based on calibration and test data, PH-CS produces a selection set $\mathcal{R}^{\textrm{PH-CS}}$, together with an estimated FDP $\alpha^{\textrm{PH-CS}}$, by maximizing an arbitrary utility \eqref{eq_PHCS_choice} that balances selection set size and reliability (i.e., FDP). Contour lines of the utility function $U({\color{blue}r(\mathcal{R})}, \alpha(\mathcal{R}))$ are shown  with darker lines indicating higher utility. PH-CS optimizes the utility over a candidate path, shown via blue dots, which is determined post-hoc based on calibration and test data. The red star marks the selected operating point. The reliability guarantee \eqref{eq_intro_relia} ensures that the estimate $\alpha^{\textrm{PH-CS}}$ is, on average, an approximate upper bound on the true FDR (see \eqref{eq_intro_relia} and \eqref{eq_heur_FDR}).}
    \vspace{-5mm}
    \label{fig_sys_mod}
\end{figure}

% \vspace{-0.1cm}
\section{Introduction}\label{sec_intro}

\subsection{Motivation}\label{sec_intro_motiv}
\emph{Conformal selection} (CS) \cite{jin2023selection} provides a principled, distribution-free mechanism for choosing a subset of test inputs whose unobserved outcomes are likely to meet a given requirement. Examples of applications include early-stage drug discovery \cite{boldini2024machine, bai2025drug}, candidate gene selection in genomics \cite{korthauer2019practical}, and feature selection in machine learning \cite{dai2023false, stoica2022false}. CS is able to control the \emph{false discovery rate} (FDR), i.e., the fraction of incorrectly selected candidates, at a pre-specified level $\alpha_{\textrm{max}}\in[0,1]$. A key limitation of this framework is that the target FDR level $\alpha_{\textrm{max}}$ must be fixed \emph{before} observing the test and calibration data. In practice, however, the ``right'' operating point is often dictated by the quality of the candidates and by downstream constraints such as budget caps or resource availability, all of which only become clear once the calibration and test data have been collected. Being forced to commit to a single FDR level a priori can lead to selection sets that are either too conservative, missing promising candidates that could have been useful for downstream applications, or too liberal, exceeding acceptable constraints for downstream users, with no mechanism for readjustment.

This tension between a priori guarantees and post-hoc flexibility arises across a wide range of applied settings. For instance, in early-stage drug discovery, after observing enrichment patterns, it may be desirable to tighten or relax the selection threshold so as to fill a fixed number of slots for downstream validation, rather than being locked to a nominal FDR level chosen before the screen \cite{boldini2024machine, bai2025drug}. As another example, in genomics or neuroimaging, researchers may inspect the distribution of test statistics and decide how aggressively to pursue candidates based on observed evidence strength and available follow-up resources \cite{korthauer2019practical}. Finally, in feature selection for high-dimensional machine learning, users may wish to expand or contract the feature set after examining prediction accuracy or computational cost on a held-out sample \cite{dai2023false, stoica2022false}.

In each of these scenarios, as illustrated in Fig. \ref{fig_sys_mod}, the user faces a \emph{utility-driven} trade-off between the size of the selected set and the proportion of erroneous selections. Conventional CS cannot resolve this trade-off in a data-adaptive way, given that its FDR guarantee is tied to a single, pre-specified level.

To address this issue, this paper introduces \emph{post-hoc CS} (PH-CS). As sketched in Fig. \ref{fig_sys_mod}, PH-CS produces a path of candidate selection sets together with data-driven \emph{false discovery proportion} (FDP) estimates. The FDP is the fraction of incorrectly selected candidates, and its average is the FDR. Based on the candidate selection sets, PH-CS allows the user to select the operating point that maximizes an arbitrary non-negative utility of set size and estimated FDP. Importantly, the estimated FDP results in reliability guarantee that holds \emph{uniformly} over the path, so the choice of operating point can be made \emph{after} the data have been observed without violating statistical validity. In particular, the estimated FDP is, on average, an approximately valid upper bound on the FDR (see Sec. \ref{sec_intro_contrib} for a precise statement).

\subsection{Related Work} \label{sec_intro_related}

\subsubsection{Conformal Prediction and Conformal Selection}
Conformal prediction constructs distribution-free prediction sets with finite-sample coverage guarantees under exchangeability assumptions \cite{vovk2005algorithmic, angelopoulos2023conformal, xu2023conformal, park2023few, jensen2024ensemble}. Conformal prediction builds on \emph{conformal p-variables}, statistics that target the null hypothesis of exchangeability \cite{barber2025unifying}. Reference \cite{jin2023selection} recast selection of test data units as the multiple testing of random null hypotheses, leveraging conformal p-variables and the Benjamini-Hochberg (BH) procedure \cite{benjamini1995controlling, benjamini2001control} to control the FDR at a fixed level. Other applications of conformal p-variables include anomaly detection \cite{bates2023testing}.

Since the introduction of CS \cite{jin2023selection}, numerous extensions have been proposed. In terms of the selection criterion, the work \cite{jin2023sensitivity} extended CS to sensitivity analysis of individual treatment effects under covariate shift, reference \cite{bai2025multivariate} generalized CS to multivariate responses via a regional monotonicity condition, and the paper \cite{hao2025multicondition} handled conjunctive and disjunctive selection conditions. Targeting the statistical power, reference \cite{bai2024optimized} improved power through data-reuse strategies for score optimization, and the work \cite{qin2025revamping} achieved asymptotically optimal power via a likelihood-ratio-based rule inspired by the Neyman-Pearson paradigm. Addressing the decision protocol, CS has been adapted to online settings with irrevocable decisions \cite{liu2026online} or with real-time feedback \cite{lu2025feedback}, and to interactive settings that support adaptive model updates during screening \cite{gui2025acs}. Other extensions include promoting diversity in the selected set \cite{nair2025diversifying} and applications to compound screening in drug discovery \cite{bai2025drug}. In all these studies, however, the FDR level must be fixed a priori.

\subsubsection{E-variables and E-BH}
E-variables \cite{vovk2021evalues, grunwald2024safe, koning2023post} offer an alternative to p-variables for quantifying evidence against a null hypothesis, with large values indicating stronger evidence. Reference \cite{wang2022false} developed the e-Benjamini-Hochberg (e-BH) procedure, which controls the FDR using e-variables under arbitrary dependence among the test statistics, a property that the classical BH procedure does not enjoy in general. In the conformal setting, the work \cite{balinsky2024enhancing} introduced conformal e-variables based on a score-ratio statistic that replaces the rank-based conformal p-values introduced in \cite{vovk2005algorithmic}, and the paper \cite{bashari2023derandomized} combined conformal e-values with e-BH to achieve derandomized novelty detection with provable FDR control. Further developments have focused on improving the power of e-BH, including conditional calibration strategies \cite{lee2024boosting}, compound and weighted e-values that concentrate the testing budget on promising hypotheses \cite{ignatiadis2024asymptotic}, and randomization-based enhancements \cite{xu2023more}. Our work also adopts conformal e-variables and e-BH, but targets post-hoc selection rather than prediction or novelty detection.

\subsubsection{Post-hoc and Post-selection Inference}
Post-hoc inference in multiple testing aims to provide valid error guarantees for rejected sets chosen after observing the data. In the p-value framework, reference \cite{goeman2011multiple} established a foundational approach by bounding the number of false discoveries uniformly over all possible rejected sets. This framework was further developed by \cite{blanchard2020post} and \cite{hemerik2019permutation}, who provided tighter post-hoc bounds on the FDP via reference families and permutation-based methods, respectively. In the e-value framework, reference \cite{xu2024post} demonstrated the flexibility of e-values for post-selection inference through e-value-based confidence intervals. References \cite{gauthier2025backward, gauthier2025values} introduced backward conformal prediction, which fixes the prediction set size and adapts the coverage level accordingly, yielding a reliability condition that directly inspires the post-hoc guarantee pursued in this work. {\color{blue}A related direction bounds the FDP with high probability, simultaneously over all rejection thresholds \cite{gazin2024transductive, song2026everywhere}, which we adopt as one of our baselines.}

\subsubsection{Conformal Selection with E-variables}
Recent work has combined conformal e-variables with e-BH for selection and related tasks. Reference \cite{lee2025selection} applied this combination to selection from hierarchical data, constructing count-based e-variables from threshold exceedances after a data-dependent cutoff. More recently, the work \cite{bai2026conformal} extended CS to settings where the quality of each selection is measured by a continuous loss rather than a binary indicator, using risk-adjusted e-variables to control a generalized notion of FDR.
% In both works, however, the e-variable construction depends on the nominal level $\alpha_{\textrm{max}}$, so that changing the target $\alpha_{\textrm{max}}$ requires recomputing the e-variables. This prevents direct application to post-hoc selection, where the operating point is chosen after observing the data.
{\color{blue}In both works, however, the e-variable construction depends on a thresholding level that must be fixed before the data are seen, so that changing it requires recomputing the e-variables.} This prevents direct application to post-hoc selection, where the operating point is chosen after observing the data.

\subsection{Main Contributions} \label{sec_intro_contrib}

The main contributions of this paper are summarized as follows.
\begin{itemize}
    \item \textbf{Post-hoc conformal selection:} We propose PH-CS, a novel CS framework that generates a path of candidate selection sets with associated FDP estimates (see Fig. \ref{fig_sys_mod}). Among the candidate {\color{blue}selection} sets, PH-CS selects the operating point that maximizes a user-specified utility function balancing set size and reliability (Algorithm \ref{algo_PHCS}). Like \cite{lee2025selection} and \cite{bai2026conformal}, PH-CS replaces conformal p-variables and the BH procedure with conformal e-variables \cite{balinsky2024enhancing} and the e-BH procedure \cite{wang2022false}. However, unlike these prior works, whose e-variable constructions are tied to a specific nominal level, PH-CS exploits the level-uniform property of e-BH to ensure simultaneous validity across all nominal levels, enabling post-hoc selection of the operating point.

    \item \textbf{Finite-sample post-hoc reliability guarantee:} We prove that the FDP estimate $\alpha^{\textrm{PH-CS}}$ produced by PH-CS for any test batch satisfies, on average, the reliability condition
    \begin{align}\label{eq_intro_relia}
        \mathbb{E}\bigg[\frac{\textrm{FDP}}{\alpha^{\textrm{PH-CS}}}\bigg]\leq 1,
    \end{align}
    implying that the average of the FDP estimate $\alpha^{\textrm{PH-CS}}$, i.e., $\mathbb{E}[\alpha^{\textrm{PH-CS}}]$, is an approximate upper bound on the true FDR. In fact, this condition coincides, up to second-order terms, with the inequality $\mathbb{E}[\alpha^{\textrm{PH-CS}}]\geq \textrm{FDR}$ \cite{koning2023post, gauthier2025backward}. This guarantee holds for \emph{any} data-dependent choice of operating point along the e-BH solution path, under the only assumptions of exchangeability and
    score monotonicity.

    \item \textbf{Flexible utility designs:} PH-CS accommodates post-hoc selection based on an arbitrary utility function obtained from FDP and set size. This formulation supports, e.g., selection sets that control the size of the selected set, as well as additive trade-off designs, enabling the user to tailor the size-reliability balance to the application at hand.

    \item \textbf{Extensions to risk control and weighted selection:} We extend PH-CS to settings where the quality of each selection is measured by a continuous loss rather than a binary indicator \cite{bai2026conformal}, yielding \textit{post-hoc risk-controlled selection} (PH-RCS) with the same reliability guarantee. We further incorporate priority weights that allow the user to assign higher importance to more promising test inputs while preserving post-hoc validity.

    \item \textbf{Empirical validation:} Through experiments on three synthetic settings and three real-world datasets, we demonstrate that PH-CS consistently satisfies user-imposed size and utility constraints that conventional CS cannot enforce, while providing reliable FDP estimates and maintaining competitive FDR performance.
\end{itemize}

The remainder of the paper is organized as follows. Sec. \ref{sec_problem} formulates the post-hoc selection problem and the utility-based objective, and Sec. \ref{sec_conv_cs} reviews conventional CS. The proposed PH-CS framework and its theoretical guarantee are developed in Sec. \ref{sec_PHCS}, followed by extensions to risk control and weighted selection in Sec. \ref{sec_ph_rcs}. Experimental results are presented in Sec. \ref{sec_exp}, and Sec. \ref{sec_conclusion} concludes this paper. Proofs and additional experiments are deferred to the appendices.

\section{Problem Formulation}\label{sec_problem}
As illustrated in Fig. \ref{fig_sys_mod}, we study a selection problem in which a user leverages a labeled calibration sample to select a subset of unlabeled inputs with desirable properties from a batch of candidates \cite{jin2023selection, xu2024post, lee2025selection}. The goal of this selection is to balance selection size and reliability: selecting more inputs is generally preferable, but it can increase the fraction of selected inputs that fail to meet the pre-specified requirements. The rest of this section formalizes setting and requirements.

\vspace{-1mm}
\subsection{Calibration and Test Data}
The system has access to a labeled calibration sample $\mathcal{D}^{\textrm{cal}} = \{(X_i, Y_i)\}_{i = 1}^n$, collected from past operations and used to calibrate selection decisions. At a given epoch, the system observes a batch of $m$ unlabeled test covariates $\mathcal{X}^{\textrm{test}} = \{X_{n+j}\}_{j = 1}^m$, generated together with unobserved responses $\mathcal{Y}^{\textrm{test}} = \{Y_{n+j}\}_{j=1}^m$. We assume that the combined collection of calibration and test pairs $\mathcal{D}^{\textrm{cal}} \cup \mathcal{D}^{\textrm{test}}$, with $\mathcal{D}^{\textrm{test}} = \{(X_{n+j}, Y_{n+j})\}_{j = 1}^m$, is exchangeable. This is the case when all samples $(X,Y) \in \mathcal{D}^{\textrm{cal}} \cup \mathcal{D}^{\textrm{test}}$ are independent.

Larger values of the response variable correspond to more desirable outcomes. Accordingly, for each test input $j=1,\ldots,m$, the user specifies a threshold $c_j$ that represents a minimum target performance level. The goal is to identify test inputs $X_{n+j}$ whose (unobserved) outcome $Y_{n+j}$ strictly exceeds the minimum value $c_j$, i.e.,
\begin{align}\label{eq_perform_req}
    Y_{n+j} > c_j.
\end{align}
A generalization of this requirement that allows for the control of a general loss $L(X_{n+j}, Y_{n+j})$ is presented in Sec. \ref{sec_ph_rcs} following \cite{bai2026conformal}.

\vspace{-1mm}
\subsection{Selection Rule and False Discovery Rate}
A selection rule $R$ maps the observed data $(\mathcal{D}^{\textrm{cal}},\mathcal{X}^{\textrm{test}})$ to a data-dependent selection subset
\begin{align}\label{eq_sel_set}
    \mathcal{R} = R(\mathcal{D}^{\textrm{cal}}, \mathcal{X}^{\textrm{test}})\subseteq\{1, \ldots, m\},
\end{align}
whose indices are interpreted as the test inputs in set $\mathcal{X}^{\textrm{test}}$ the user deems likely to satisfy the requirement in \eqref{eq_perform_req}.

A \textit{false discovery} occurs when a test point is selected even though it fails to satisfy the requirement in \eqref{eq_perform_req}, i.e., when one has the inclusion $j\in\mathcal{R}$ and also the inequality $Y_{n+j}\leq c_j$. We quantify the realized error using the FDP, i.e.,
\begin{align}\label{eq_FDP_def}
    \textrm{FDP}(\mathcal{R}, \mathcal{Y}^{\textrm{test}}) = \frac{\sum_{j = 1}^m \mathds{1}\{j \in \mathcal{R}, Y_{n+j} \leq c_j\}}{\max\{1, |\mathcal{R}|\}},
\end{align}
and define the FDR as its expectation, i.e.,
\begin{align}\label{eq_FDR_def}
    \textrm{FDR}(R) = \mathbb{E}\big[\textrm{FDP}(\mathcal{R}, \mathcal{Y}^{\textrm{test}})\big],
\end{align}
where the expectation is taken over the joint distribution of the selected set $\mathcal{R}$ and of the true test labels $\mathcal{Y}^{\textrm{test}}$.

\vspace{-1mm}
\subsection{Utility-Based Objective} \label{sec_utility_obj}
The design of the selection rule $R$ entails a trade-off between set size and reliability. On the one hand, selecting a larger set $\mathcal{R}$ is generally preferable, as it may produce a larger yield in a production system, give more options to downstream applications \cite{jin2023sensitivity, de2024towards}, or require fewer queries to further evaluate unselected units using high-cost validation systems \cite{richens2020improving, carracedo2021review}. On the other hand, a larger set $\mathcal{R}$ may also increase the fraction of selected inputs that fail to meet requirement \eqref{eq_perform_req}. Conversely, a more conservative rule that produces a smaller set $\mathcal{R}$ may reduce the number of erroneous selections, but it may also exclude desirable inputs satisfying the condition \eqref{eq_perform_req}.

CS does not provide any control for the set size $|\mathcal{R}|$, targeting only a maximum FDR value $\alpha_{\textrm{max}}$ via the inequality $\textrm{FDR}(\mathcal{R})\leq \alpha_{\textrm{max}}$ (see Sec. \ref{sec_conv_cs}). In contrast, in this work we assume that, given the dataset $\mathcal{D}^{\textrm{cal}}$ and the unlabeled test data $\mathcal{X}^{\textrm{test}}$, the user wishes to balance size and reliability in the selection of set $\mathcal{R}$.

{\color{blue}To elaborate, we write as
\begin{align}\label{eq_w_size}
    r = r(\mathcal{R}) = \sum_{j\in\mathcal{R}}\pi_j
\end{align}
the weighted size of set $\mathcal{R}$, where $0\leq \pi_j \leq 1$ are arbitrary weights. With $\pi_j=1$, we obtain the set cardinality $r(\mathcal{R})=|\mathcal{R}|$. Alternatively, if the weight $\pi_j$ is an estimate of the probability of the requirement \eqref{eq_perform_req}, the size \eqref{eq_w_size} becomes an estimate of the average number of correct selections that satisfy condition \eqref{eq_perform_req}.}

Let $U: {\color{blue}[0,m]} \times (0,1) \mapsto [0, +\infty)$ be a user-specified non-negative utility $U(r,\alpha)$ that is non-decreasing in the set size $r$ {\color{blue}in \eqref{eq_w_size}} and non-increasing in the FDP $\alpha$. The specific ways in which the utility $U(r, \alpha)$ increases with respect to set size $r$ and decreases with the FDP $\alpha$ characterize the user-defined trade-off between these two criteria. For example, the utility function $U(r,\alpha)$ may be chosen in one of the following ways:
\begin{itemize}
    \item \emph{Constrained-size design:} Given a minimum selection size $r_{\textrm{min}}\in\{0,1,\ldots,m\}$, the user can prioritize smaller FDP levels among rules with size $r \geq r_{\textrm{min}}$ by adopting the utility
    \begin{align}\label{eq_utility_size_first}
        U(r, \alpha) = (1-\alpha) \cdot \mathds{1}\big\{r \geq r_{\textrm{min}}\big\}.
    \end{align}
    {\color{blue}Such a lower bound arises when a downstream stage must process a guaranteed number of candidates, as in filling a screening budget in drug discovery \cite{boldini2024machine, bai2025drug} or shortlisting genes for follow-up \cite{korthauer2019practical}.}

    \item \emph{Additive trade-off:} A general trade-off between size $r$ and FDP values can be expressed via the additive utility
    \begin{align}\label{eq_u_add}
        U(r, \alpha) = u(r) - \lambda v(\alpha) + C,
    \end{align}
    where $u: {\color{blue}[0, m]} \to [0, +\infty)$ is a non-decreasing function; $v: (0,1) \to [0, +\infty)$ is a non-decreasing function; hyperparameter $\lambda>0$ controls the relative importance of set size $r$ and FDP $\alpha$; and the constraint $C\geq 0$ is chosen to ensure the non-negativity of utility \eqref{eq_u_add}. Varying the hyperparameter $\lambda$ yields a continuum of possible operating points between size-oriented and reliability-oriented behavior.

    \item {\color{blue}\emph{Multiplicative trade-off:} Alternatively, the size $r$ and reliability $1-\alpha$ can be combined via the multiplicative utility
    \begin{align}\label{eq_u_mult}
        U(r, \alpha) = u(r) \cdot v(1-\alpha),
    \end{align}
    where function $u(\cdot)$ and $v(\cdot)$ are non-decreasing as in \eqref{eq_u_add}. With the choices $u(r) = r$ and $v(x) = x$, together with weights $\pi_j=1$ in \eqref{eq_w_size}, the utility \eqref{eq_u_mult} captures the expected number of correct selections among the $r$ selected inputs, i.e., $U(r,\alpha)=r(1-\alpha)$.}
\end{itemize}

We are ideally interested in supporting any selection rule \eqref{eq_sel_set} that addresses the problem
\begin{align}\label{eq_utility_obj}
    \max_{\mathcal{R} \subseteq \{1, \ldots, m\}} U\big({\color{blue}r(\mathcal{R})}, \alpha (\mathcal{R})\big),
\end{align}
where {\color{blue} $r(\mathcal{R})$ is the weighted set size \eqref{eq_w_size} and} $\alpha(\mathcal{R}) = \textrm{FDP}(\mathcal{R}, \mathcal{Y}^{\textrm{test}})$ is the FDP in \eqref{eq_FDP_def} for set $\mathcal{R}$. However, the FDP $\alpha(\mathcal{R})$ in \eqref{eq_FDP_def} depends on the true responses $\mathcal{Y}^{\textrm{test}}$, which are unknown. Therefore, we propose to study selection rules that use only data $\mathcal{D}^{\textrm{cal}}$ and $\mathcal{X}^{\textrm{test}}$, taking the form
\begin{align}\label{eq_set_rule}
    \mathcal{R} = R(\mathcal{D}^{\textrm{cal}}, \mathcal{X}^{\textrm{test}}) = \arg\max_{\mathcal{R} \subseteq \{1, \ldots, m\}} U\big({\color{blue}r(\mathcal{R})}, \hat{\alpha}(\mathcal{R})\big),
\end{align}
where $\hat{\alpha}(\mathcal{R})$ is an estimate of the quantity $\textrm{FDP}(\mathcal{R}, \mathcal{Y}^{\textrm{test}})$ based on the available data. Note that, throughout this paper, the notation $\arg\max$ is used to denote any of the maximizers of the given function.

In order for problem \eqref{eq_set_rule} to be a useful approximation of the original optimization \eqref{eq_utility_obj}, we impose that the estimate $\hat{\alpha}(\mathcal{R})$ used in problem \eqref{eq_set_rule} satisfy the inequality
\begin{align}\label{eq_post_FDP_first}
    \mathbb{E}\bigg[\frac{\textrm{FDP}(\mathcal{R}, \mathcal{Y}^{\textrm{test}})}{\hat{\alpha}(\mathcal{R})}\bigg] \leq 1,
\end{align}
where the expectation is taken over the joint distribution of the selected set $\mathcal{R}$ and of the true test labels $\mathcal{Y}^{\textrm{test}}$. This inequality, inspired by \cite{gauthier2025backward}, stipulates that, on average, the ratio between the true FDP and the estimate $\hat{\alpha}(\mathcal{R})$ does not exceed $1$. {\color{blue}The ratio \eqref{eq_post_FDP_first} compares the true FDP with its estimate within each run, and it reduces to the standard requirement $\mathbb{E}[\textrm{FDP}]\leq\alpha$ when $\hat{\alpha}(\mathcal{R})$ is a fixed level $\alpha$.

For a general data-dependent $\hat{\alpha}(\mathcal{R})$,} a first-order Taylor expansion of $1 / \hat{\alpha}(\mathcal{R})$ around the mean $\mathbb{E}[\hat{\alpha}(\mathcal{R})]$, as shown in Appendix \ref{apdx_Taylor}, implies that condition \eqref{eq_post_FDP_first} is approximately, up to a quadratic error of order $O\big((\hat{\alpha}(\mathcal{R}) - \mathbb{E}[\hat{\alpha}(\mathcal{R})])^2\big)$, equivalent to the inequality \cite{koning2023post, gauthier2025backward}
\begin{align}\label{eq_heur_FDR}
    \textrm{FDR}(R) \leq \mathbb{E}[\hat{\alpha}(\mathcal{R})].
\end{align}
Accordingly, the estimate $\hat{\alpha}(\mathcal{R})$ provides, on average, an approximately valid upper bound on the true FDR \eqref{eq_FDR_def}. {\color{blue}An alternative criterion based on a probably approximately correct (PAC) requirement \cite{gazin2024transductive, song2026everywhere} will be explored in Sec. \ref{sec_exp_baselines}.}

Overall, as illustrated in Fig. \ref{fig_sys_mod}, our goal is to design a selection rule that addresses problem \eqref{eq_set_rule}, while leveraging an FDR estimate $\hat{\alpha}(\mathcal{R})$ satisfying the inequality \eqref{eq_post_FDP_first}. The resulting class of methods can adjust set size {\color{blue}$r(\mathcal{R})$} and FDP $\alpha(\mathcal{R})$, in a post-hoc manner, as a function of the observed data $\mathcal{D}^{\textrm{cal}}$ and $\mathcal{X}^{\textrm{test}}$, while providing a reliable estimate of the FDR.

\section{Preliminaries: Conventional Conformal Selection} \label{sec_conv_cs}
To provide the necessary background, in this section we briefly review the conventional CS framework based on conformal p-variables and the BH procedure introduced in \cite{jin2023selection}. CS is a selection rule of the form \eqref{eq_sel_set} that imposes a constraint on the FDR \eqref{eq_FDR_def}, i.e.,
\begin{align}\label{eq_FDR_gua}
    \textrm{FDR}(R) \leq \alpha_{\textrm{max}}
\end{align}
for a pre-specified target maximum level $\alpha_{\textrm{max}} \in (0,1)$. Unlike the general utility-based rule \eqref{eq_set_rule} studied in this work, CS cannot adjust set size $|\mathcal{R}|$ and FDP level $\alpha$ as a function of the observed data $\mathcal{D}^{\textrm{cal}}$ and $\mathcal{X}^{\textrm{test}}$. Rather, it can only ensure the average guarantee \eqref{eq_FDR_gua}. CS serves as a baseline, highlighting the limitations in the state of the art that motivate the proposed class of selection methods.

\subsection{Selection as Multiple Hypothesis Testing} \label{sec_conv_prelim}
For each test input $X_{n+j} \in \mathcal{X}^{\textrm{test}}$, CS considers the \emph{random} null hypothesis
\begin{align}\label{eq_hypo}
    H_j: Y_{n+j}\leq c_j,
\end{align}
so that rejecting hypothesis $H_j$ corresponds to selecting input $X_{n+j}$. With this formulation, false discoveries coincide with selected indices $j\in\mathcal{R}$ for which the null hypothesis $H_j$ is true. Unless stated otherwise, we assume that the target levels are fixed constraints, not dependent on data, although our results can be generalized (see Sec. \ref{sec_ph_rcs}).

To test the random nulls \eqref{eq_hypo}, conventional CS uses a predictive model $\mu: \mathcal{X} \rightarrow \mathbb{R}$, {\color{blue}whose output $\mu(X)$ estimates the response for input $X$,} to define a non-negative \textit{conformity score} $S: \mathcal{X} \times \mathcal{Y} \rightarrow \mathbb{R}_+$. The purpose of the score $S(X,Y)$ is to quantify how likely it is for the true response associated with input $X$ under the data distribution to exceed, i.e., to be better than, level $Y$. For instance, one may use the scores $S(X,Y) = \max\{\mu(X) - Y, 0\}$ or $S(X,Y) = \exp(\mu(X)-Y)$, which are non-negative and non-increasing in $Y$ \cite{jin2023selection}.

The score $S(X,Y)$ is assumed to be monotone non-increasing in the response $Y$, i.e., for every $X\in\mathcal{X}$, we have the inequality 
\begin{align}\label{eq_sco_mon}
    S(X,Y)\geq S(X,Y')
\end{align}
for $Y \leq Y'$. This condition captures the desired property that the likelihood of exceeding a higher (better) value $Y' \geq Y$ cannot be larger than the corresponding likelihood for level $Y$.

On the calibration samples, CS computes the scores $S_i = S(X_i,Y_i)$ for all data points $i = 1, \ldots, n$. Moreover, for each test input $X_{n+j}$, CS evaluates the score at the required output threshold $c_j$ as $\hat{S}_{n+j} = S(X_{n+j},c_j)$ for $j = 1, \ldots, m$. Note that the score $\hat{S}_{n+j}$ is computable, since it depends only on the test input $X_{n+j}$. Furthermore, given the assumed monotonicity of the score function, it provides the highest score for responses compatible with the condition \eqref{eq_perform_req}.

CS converts the calibration scores $\{S_i\}_{i=1}^n$ and the test scores $\hat{S}_{n+j}$ into a \emph{conformal p-variable} $P_j$ for each random null hypothesis $H_j$ in \eqref{eq_hypo}, with $j=1,\ldots,m$. This is given by
\begin{align}\label{eq_conf_pval}
    P_j = \frac{1 + \sum_{i=1}^n \mathds{1}\{S_i > \hat{S}_{n+j}\}}{n+1},
    \quad j=1,\ldots,m.
\end{align}
Note that if the test score $\hat{S}_{n+j}$ coincides with some calibration scores $S_i$, ties are resolved by randomized tie-breaking: among the equal-score calibration points, a uniformly random subset is counted as exceeding $\hat{S}_{n+j}$ \cite{jin2023selection}. The quantity \eqref{eq_conf_pval} effectively counts the fraction of calibration points with higher scores $S_i$ than the score $\hat{S}_{n+j}$ for the test input at the target response $c_j$.

A key subtlety in CS is that the null hypothesis $H_j$ is random because it involves the unobserved test outcome $Y_{n+j}$. Consequently, the quantity $P_j$ is not a classical p-variable for the null hypothesis $H_j$. Instead, the quantity $P_j$ is a conformal p-variable for the related null hypothesis that the multiset $\{S_i\}_{i=1}^n \cup \{\hat{S}_{n+j}\}$ is exchangeable \cite{jin2023selection}. Intuitively, if this exchangeability hypothesis does not hold, the threshold $c_j$ is incompatible with test input $X_{n+j}$, and one should reject the null hypothesis $H_j$. Accordingly, a small value of the quantity $P_j$ provides evidence against exchangeability and, in turn, against the random null hypothesis $H_j$ \cite{jin2023selection}.

\subsection{BH Selection} \label{sec_conv_bh}
Given the conformal p-variables $\{P_j\}_{j=1}^m$ and a pre-specified nominal level $\alpha_{\textrm{max}} \in (0, 1)$, CS evaluates the rejection set $\mathcal{R}$ via the BH procedure \cite{jin2023selection, benjamini1995controlling, benjamini2001control}. Specifically, sorting the p-variables $P_{(1)} \leq \cdots \leq P_{(m)}$, BH computes the maximum index $k$ for which the p-variable $P_{(k)}$ does not exceed the corrected threshold $\alpha_{\textrm{max}} k/m$, i.e.,
\begin{align}\label{eq_BH_k}
    k(\alpha_{\textrm{max}}) = \max\Big\{k\in\{1, \ldots, m\}: P_{(k)} \leq \frac{\alpha_{\textrm{max}} k}{m}\Big\},
\end{align}
with the convention $\max(\varnothing) = 0$. Then it outputs the selected set
\begin{align}\label{eq_BH_set}
    R^{\textrm{CS}}_{\alpha_{\textrm{max}}}(\mathcal{D}^{\textrm{cal}}, \mathcal{X}^{\textrm{test}}) = \Big\{j\in\{1, \ldots, m\}: P_j \leq \frac{\alpha_{\textrm{max}} k(\alpha_{\textrm{max}})}{m}\Big\}
\end{align}
if $k(\alpha_{\textrm{max}})>0$ and $R^{\textrm{CS}}_{\alpha_{\textrm{max}}}(\mathcal{D}^{\textrm{cal}}, \mathcal{X}^{\textrm{test}}) = \varnothing$ if $k(\alpha_{\textrm{max}}) = 0$.

\subsection{Theoretical Guarantees}
CS satisfies the following guarantee on the FDR.
\begin{theorem}[\textbf{\!\!\!\cite[Theorem 3]{jin2023selection}: Fixed-level FDR control}]\label{theo_conv_fdr}
    Assume that (i) the combined calibration and test pairs $\mathcal{D}^{\text{\rm{cal}}} \cup \mathcal{D}^{\text{\rm{test}}}$ are i.i.d; and (ii) the score function $S(X,Y)$ is monotone non-increasing in $Y$ as in \eqref{eq_sco_mon}. Let the levels $\{c_j\}_{j=1}^m$ and the target FDR $\alpha_{\text{\rm{max}}}\in(0,1)$ be fixed in advance, independently of the observed data. Then, the CS selection rule $R^{\text{\rm{CS}}}_{\alpha_{\text{\rm{max}}}}$ defined in \eqref{eq_BH_set} satisfies the inequality
    \begin{align}\label{eq_gua_CS}
        \text{\rm{FDR}} \big(R^{\text{\rm{CS}}}_{\alpha_{\text{\rm{max}}}}\big) \leq \alpha_{\text{\rm{max}}}.
    \end{align}
\end{theorem}

\section{Post-Hoc Conformal Selection} \label{sec_PHCS}
We now develop PH-CS, a data-driven selection method capable of addressing the post-hoc utility optimization in \eqref{eq_utility_obj} based on an FDR estimate $\hat{\alpha}(\mathcal{R})$ satisfying the requirement \eqref{eq_post_FDP_first}. As in the conventional CS reviewed in Sec. \ref{sec_conv_cs}, PH-CS views selection as multiple hypothesis testing of the random nulls $H_j: Y_{n+j}\leq c_j$ in \eqref{eq_hypo}. Unlike CS, however, PH-CS can address the post-hoc optimization \eqref{eq_utility_obj} rather than being constrained to the FDR requirement \eqref{eq_FDR_gua}. Technically, this is done by replacing conformal p-variables \eqref{eq_conf_pval} and the BH procedure adopted by CS with conformal e-variables and e-BH \cite{wang2022false}.

\subsection{Scores and Conformal E-variables}\label{sec_sco_eval}
Like CS, in order to construct candidate selection sets, PH-CS relies on a predictive model {\color{blue}$\mu(\cdot)$} to obtain conformity scores $S(X,Y)$, where larger values indicate stronger evidence that the outcome exceeds level $Y$ for input $X$. As discussed in Sec. \ref{sec_conv_prelim}, the score $S(X,Y)$ satisfies the monotonicity condition \eqref{eq_sco_mon}. Furthermore, as in CS, PH-CS computes the scores $S_i=S(X_i,Y_i)$ for the calibration data $i=1,\ldots,n$, and the scores $\hat{S}_{n+j} = S(X_{n+j},c_j)$ for the test data $j = 1, \ldots, m$. We recall that the score $\hat{S}_{n+j}$ can be interpreted as quantifying the evidence, for test input $X_{n+j}$, in favor of meeting the requirement $Y_{n+j}>c_j$.

While CS adopts the p-variables \eqref{eq_conf_pval}, PH-CS adopts e-variables targeting the random null hypothesis \eqref{eq_hypo}. An e-variable is a random variable whose average under the null does not exceed $1$. E-variables provide a more robust way to measure evidence against a null hypothesis as compared to a p-variable \cite{xu2024post}. In particular, as shown in \cite{koning2023post}, e-variables support the post-hoc selection of significance levels in hypothesis testing, while p-variables require significance levels to be specified in advance.

{\color{blue}While different choices are possible \cite{bashari2023derandomized, lee2025selection, bai2026conformal}, we focus here on the \textit{conformal e-variable} introduced in \cite{balinsky2024enhancing}. Unlike the e-variables presented in \cite{bashari2023derandomized, lee2025selection, bai2026conformal}, the adopted e-variable has the key advantage that it does not depend on a target FDR level, supporting the desired utility optimization step. For each test point $X_{n+j}$ with $j = 1, \ldots, m$, the {conformal e-variable} is defined as \cite{balinsky2024enhancing}}
\begin{align}\label{eq_e_vari}
    E_j = \frac{\hat{S}_{n+j}}{\frac{1}{n+1}\Big(\sum_{i=1}^n S_i + \hat{S}_{n+j}\Big)}.
\end{align}
This statistic is an e-variable for the null hypothesis that the multiset $\{S_i\}_{i=1}^n\cup\{\hat{S}_{n+j}\}$ is exchangeable \cite{balinsky2024enhancing}. In fact, under this assumption, the average of the random variable $E_j$ equals $1$ (see Appendix \ref{apx_prof_posthoc}). Accordingly, a large value of the statistic \eqref{eq_e_vari} provides evidence that exchangeability does not hold, which, in turn, supports rejecting the random null hypothesis $H_j:Y_{n+j}\leq c_j$ in \eqref{eq_hypo} by following the same arguments given in Sec. \ref{sec_conv_prelim}.

\subsection{Utility-Driven Selection}
PH-CS aims to output a data-dependent set $\mathcal{R}\subseteq\{1,\ldots,m\}$ that maximizes the utility $U\big({\color{blue}r(\mathcal{R})}, \hat{\alpha}(\mathcal{R})\big)$ as per problem \eqref{eq_set_rule}, where the FDP estimate $\hat{\alpha}(\mathcal{R})$, computed from data $(\mathcal{D}^{\textrm{cal}}, \mathcal{X}^{\textrm{test}})$, must serve, on average, as a reliable upper bound on the FDR as formalized by \eqref{eq_post_FDP_first}. To address problem \eqref{eq_set_rule}, PH-CS starts by restricting the optimization to a discrete set of possible subsets $\mathcal{R}_k\subseteq\{1,\ldots,m\}$ with $k\in\mathcal{K}$, where $\mathcal{K}$ is a discrete set of integers. To motivate our choice of candidate subsets $\{\mathcal{R}_k\}_{k\in\mathcal{K}}$, we review next the operation of e-BH \cite{wang2022false}, an FDR-controlling testing procedure based on e-variables.

For a given nominal FDR level $\alpha \in (0,1)$, e-BH selects indices $j\in\{1,\ldots,m\}$ whose e-variable $E_j$ exceeds a threshold $t(\alpha)$, i.e.,
\begin{align}\label{eq_set_alpha}
    \mathcal{R}(\alpha) = \{j\in\{1,\ldots,m\}: E_j\geq t(\alpha)\}.
\end{align}
In a manner similar to BH (see \eqref{eq_BH_set}), the threshold $t(\alpha)$ is obtained by imposing the self-consistency condition that, if $k$ points are selected, i.e., if $|\mathcal{R}(\alpha)| = k$, the cutoff must be given by \cite{wang2022false}
\begin{align}\label{eq_consist_thre}
    t(\alpha) = \frac{m}{\alpha k}.
\end{align}

Let $E_{(1)}\ge \cdots \ge E_{(m)}$ be the ordered e-variables. The e-BH set \eqref{eq_set_alpha} can change only when the cutoff $t(\alpha)$ crosses one of the order statistics $\{E_{(k)}\}_{k=1}^m$. Therefore, when adopting e-BH with possibly varying FDR levels $\alpha$, one can focus without loss of generality on the discrete family of candidate subsets
\begin{align}\label{eq_Rk_def}
    \mathcal{R}_k = \{j\in\{1,\ldots,m\}\hspace{-0.5mm}:\hspace{-0.5mm} E_j \geq E_{(k)}\},~ k=0,1,\ldots,m.
\end{align}
Note that the size of the set \eqref{eq_Rk_def} increases in the index $k=0,1,\ldots,m$, i.e., we have the inclusion $\mathcal{R}_0 \subseteq \mathcal{R}_1 \subseteq \cdots \subseteq \mathcal{R}_m$.

PH-CS optimizes problem \eqref{eq_set_rule} over the sets $\mathcal{R}_k$ in \eqref{eq_Rk_def}. For each candidate $\mathcal{R}_k$, the FDP estimate $\hat{\alpha}(\mathcal{R}_k)$ is obtained by inverting the equality \eqref{eq_consist_thre} as
\begin{align}\label{eq_alpha_k_def}
    \hat{\alpha}(\mathcal{R}_k)=\min\Big\{1,\frac{m}{kE_{(k)}}\Big\},
\end{align}
where the first term in the minimum ensures that the estimate $\hat{\alpha}(\mathcal{R}_k)$ does not exceed $1$. The resulting post-hoc optimization problem solved by PH-CS is then
\begin{align}\label{eq_PHCS_choice}
    \mathcal{R}^\textrm{PH-CS} &= R^{\textrm{PH-CS}}(\mathcal{D}^{\textrm{cal}}, \mathcal{X}^{\textrm{test}}) \nonumber\\
    &= \arg\max_{\mathcal{R}\in\{\mathcal{R}_0, \ldots, \mathcal{R}_m\}} U\big({\color{blue}r(\mathcal{R})}, \hat{\alpha}(\mathcal{R})\big),
\end{align}
and reports the associated FDP level as \begin{align}\label{eq_PHCS_alpha}
    \alpha^{\textrm{PH-CS}} = \hat{\alpha}(\mathcal{R}^{\textrm{PH-CS}}).
\end{align}

\subsection{Summary of PH-CS}
\begin{algorithm}[t]
  \caption{PH-CS} \label{algo_PHCS}
  \hspace*{\algorithmicindent}\parbox[t]{\dimexpr\linewidth-\algorithmicindent}{\textbf{Input:} Calibration data $\mathcal{D}^{\textrm{cal}} = \{(X_i,Y_i)\}_{i=1}^n$, unlabeled test covariates $\mathcal{X}^{\textrm{test}} = \{X_{n+j}\}_{j=1}^m$, thresholds $\{c_j\}_{j=1}^m$, and utility $U(\cdot,\cdot)$}
  \begin{algorithmic}[1]
    \STATE{Compute calibration scores $\{S_i\}_{i=1}^n$, and test scores $\{\hat{S}_{n+j}\}_{j=1}^m$}
    \STATE{Construct conformal e-variables $\{E_j\}_{j=1}^m$ via \eqref{eq_e_vari}}
    \STATE{Sort $\{E_j\}_{j=1}^m$ in non-increasing order as $E_{(1)} \geq E_{(2)}\geq \dots \geq E_{(m)}$}
    \STATE{For each $k=1,\ldots, m$, define set $\mathcal{R}_{k} = \{j: E_j\geq E_{(k)}\}$ and FDP estimate $\hat{\alpha}(\mathcal{R}_k) = \min\{1, m / (k E_{(k)})\}$ with the convention $(\alpha_{0}, \mathcal{R}_{0})=(0,\varnothing)$}
    \STATE{Among all candidate sets $\{\mathcal{R}_k\}_{k=0}^m$, choose $\mathcal{R}^{\textrm{PH-CS}} = \arg\max_{\mathcal{R} \in \{\mathcal{R}_0,\ldots, \mathcal{R}_m\}} U({\color{blue}r(\mathcal{R})}, \hat{\alpha}(\mathcal{R}))$.}
    \STATE{Set $\alpha^{\textrm{PH-CS}} = \hat{\alpha}(\mathcal{R}^{\textrm{PH-CS}})$}
  \end{algorithmic}
  \hspace*{\algorithmicindent}\parbox[t]{\dimexpr\linewidth-\algorithmicindent}{\textbf{Output:} Selected set $\mathcal{R}^{\textrm{PH-CS}}$ and error level $\alpha^{\textrm{PH-CS}}$}
\end{algorithm}

Algorithm \ref{algo_PHCS} summarizes the proposed PH-CS procedure. Starting from the calibration sample and an unlabeled test batch, PH-CS computes conformity scores $S(X_i, Y_i)$ on the calibration pairs, and evaluates the scores $S(X_{n+j}, c_j)$ at the requirement thresholds $c_j$ on the test covariates. These quantities are then converted into conformal e-variables $\{E_j\}_{j=1}^m$ using \eqref{eq_e_vari}. Next, the e-variables are sorted in non-increasing order to generate the finite collection of candidate sets $\mathcal{R}_k$ in \eqref{eq_Rk_def}. Finally, PH-CS selects the set $\mathcal{R}_k$ by maximizing the utility \eqref{eq_PHCS_choice}, producing the selected set $\mathcal{R}^{\textrm{PH-CS}}$ in \eqref{eq_PHCS_choice} together with the chosen nominal level $\alpha^{\textrm{PH-CS}}$ in \eqref{eq_PHCS_alpha}. In the next subsection, we will show that this construction yields a valid post-hoc reliability estimate in the sense of inequality \eqref{eq_post_FDP_first}.

\subsection{Reliability Guarantee} \label{sec_theo_guar}
The following theorem shows that PH-CS, which is defined in Algorithm \ref{algo_PHCS}, returns an FDP estimate $\alpha^{\textrm{PH-CS}}$ for the selected set $\mathcal{R}^{\textrm{PH-CS}}$ that satisfies the desired reliability property \eqref{eq_post_FDP_first}.

\begin{theorem}[\textbf{Post-hoc reliability guarantee}]\label{theo_posthoc}
    Assume that (i) the combined calibration and test pairs $\mathcal{D}^{\text{\rm{cal}}} \cup \mathcal{D}^{\text{\rm{test}}}$ are exchangeable, and (ii) the score $S(X,Y)$ is monotone non-increasing in $Y$ as in \eqref{eq_sco_mon}. Then, the PH-CS output $(\mathcal{R}^{\text{\rm{PH-CS}}}, \alpha^{\text{\rm{PH-CS}}})$ defined in \eqref{eq_PHCS_choice} satisfies the average requirement
    \begin{align}\label{eq_post_FDP}
        \mathbb{E}\Bigg[\frac{\text{\rm{FDP}} \big(\mathcal{R}^{\text{\rm{PH-CS}}}, \mathcal{Y}^{\text{\rm{test}}}\big)}{\alpha^{\text{\rm{PH-CS}}}}\Bigg] \leq 1,
    \end{align}
    where the average is evaluated with respect to the joint distribution of the selected set $\mathcal{R}^{\text{\rm{PH-CS}}}$ and of the true test labels $\mathcal{Y}^{\text{\rm{test}}}$.
\end{theorem}

\textit{Proof:} See Appendix \ref{apx_prof_posthoc}.

As explained in Sec. \ref{sec_problem}, the inequality \eqref{eq_post_FDP} implies, up to a first-order approximation, the condition  \eqref{eq_heur_FDR}, showing that the estimate $\alpha^{\textrm{PH-CS}}$ is, on average, a valid upper bound on the FDR. Following \cite{gauthier2025backward}, this property can also be leveraged to estimate the FDR from the calibration sample via batch resampling.

As a final remark, Theorem \ref{theo_posthoc} can be shown to hold even if the target levels $\{c_j\}_{j=1}^m$ depend on the available data $(\mathcal{D}^{\textrm{cal}}, \mathcal{X}^{\textrm{test}})$ (see Appendix \ref{apx_prof_posthoc}). In contrast, extending Theorem \ref{theo_conv_fdr} to data-dependent thresholds requires the additional {positive regression dependence on a subset} condition \cite{jin2023selection}.

{\color{blue}
Theorem~\ref{theo_posthoc} guarantees the reliability of the estimate $\alpha^{\textrm{PH-CS}}$ in terms of the average ratio in \eqref{eq_post_FDP}. As explained in Sec. \ref{sec_problem}, the condition \eqref{eq_post_FDP} is related to \eqref{eq_heur_FDR}, i.e., to the inequality $\mathbb{E}[\textrm{FDP}(\mathcal{R}^{\textrm{PH-CS}},\mathcal{Y}^{\textrm{test}})] \leq \mathbb{E}[\alpha^{\textrm{PH-CS}}]$. A natural question is whether one can establish the desired inequality $\textrm{FDP}(\mathcal{R}^{\textrm{PH-CS}}, \mathcal{Y}^{\textrm{test}}) \leq \alpha^{\textrm{PH-CS}}$ pointwise with high probability, rather than on average. The following result shows that, as the calibration and test sets grow, this is indeed the case.

\begin{proposition}[\textbf{Asymptotic reliability}]\label{prop_asymp}
    Assume that (i) the calibration and test pairs $\mathcal{D}^{\text{\rm{cal}}}\cup\mathcal{D}^{\text{\rm{test}}}$ are i.i.d.; (ii) the score $S(X,Y)$ is monotone non-increasing in $Y$ as per inequality \eqref{eq_sco_mon}; and (iii) the score has a finite second moment $\mathbb{E}[S(X,Y)^2]<\infty$ and a positive mean $\mathbb{E}[S(X,Y)]>0$. Then, for every $\varepsilon>0$, the PH-CS output $(\mathcal{R}^{\text{\rm{PH-CS}}}, \alpha^{\text{\rm{PH-CS}}})$ in \eqref{eq_PHCS_choice}--\eqref{eq_PHCS_alpha} satisfies the limit
    \begin{align}\label{eq_prop_asymp}
        \lim_{n,m\to\infty}\mathbb{P}\Big(\text{\rm{FDP}}\big(\mathcal{R}^{\text{\rm{PH-CS}}}, \mathcal{Y}^{\text{\rm{test}}}\big) \leq \alpha^{\text{\rm{PH-CS}}} + \varepsilon\Big) = 1.
    \end{align}
\end{proposition}

\textit{Proof:} See Appendix~\ref{apx_prop_asymp}.

 Sec. \ref{sec_exp_baselines} will discuss an alternative baseline scheme that guarantees a finite-sample counterpart of condition \eqref{eq_prop_asymp} formulated as a PAC requirement \cite{gazin2024transductive, song2026everywhere}.}

% \section{Extensions}\label{sec_extension}

\section{Post-Hoc Risk-Controlled Selection}\label{sec_ph_rcs}
The PH-CS framework developed in Sec. \ref{sec_PHCS} measures the quality of each selection through the binary cost $\mathds{1}\{Y_{n+j}\leq c_j\}$ appearing in the numerator of the FDP \eqref{eq_FDP_def}: a selected input either meets the requirement \eqref{eq_perform_req}, incurring a zero cost, or not, incurring a cost equal to $1$. In many applications, however, the consequence of selecting an input is more naturally quantified by a \emph{continuous} loss $\mathcal{L}(X,Y)\in[0,1]$, such as a squared prediction error or a semantic distance to a reference output \cite{bai2026conformal}. We show below that the post-hoc mechanism in Algorithm \ref{algo_PHCS} extends directly to this setting, yielding a generalized strategy that we refer to as PH-RCS.

\subsubsection{Generalized Risk and Error Metric}
For each test input $X_{n+j}$, the cost of selection is captured by the loss
\begin{equation}\label{eq_gen_risk}
    L_j = \mathcal{L}(X_{n+j}, Y_{n+j}) \in [0,1].
\end{equation}
Replacing the binary indicator in \eqref{eq_FDP_def} with the continuous risk $L_j$ in \eqref{eq_gen_risk}, the \emph{generalized} FDP, defined as
\begin{equation}\label{eq_gfdp}
    \textrm{FDP}^{\textrm{g}}(\mathcal{R}, \mathcal{Y}^{\textrm{test}}) = \frac{\sum_{j=1}^{m} L_j\mathds{1}\{j\in\mathcal{R}\}}{\max\{1, |\mathcal{R}|\}},
\end{equation}
measures the per-selected unit risk, and its average yields the \emph{generalized} FDR
\begin{equation}
    \textrm{FDR}^{\textrm{g}}(R) = \mathbb{E} \big[\textrm{FDP}^{\textrm{g}}(\mathcal{R}, \mathcal{Y}^{\textrm{test}})\big].
\end{equation}
When we set $L_j = \mathds{1}\{Y_{n+j} \leq c_j\}$, the quantities $\textrm{FDP}^{\textrm{g}}$ and $\textrm{FDR}^{\textrm{g}}$ reduce to the FDP in \eqref{eq_FDP_def} and the FDR in \eqref{eq_FDR_def}, respectively.

\subsubsection{Risk-Adjusted E-Variables}
In Sec. \ref{sec_PHCS}, the conformal e-variable $E_j$ in \eqref{eq_e_vari} is defined by the inequality $\mathbb{E}$$[\mathds{1}\{H_j\} \cdot E_j] \leq 1$, where $H_j$ is the binary null event defined in \eqref{eq_hypo}. For a more general risk function, reference \cite{bai2026conformal} introduces \emph{risk-adjusted e-variables} as random variables satisfying the inequality
\begin{equation}\label{eq_risk_adj_eval}
    \mathbb{E}\big[L_j E_j^{\textrm{g}}\big] \leq 1.
\end{equation}
Intuitively, a large value of the generalized e-variable $E_j^{\textrm{g}}$ provides evidence that the risk $L_j$ is small, since a simultaneously large risk would violate the mean constraint \eqref{eq_risk_adj_eval}. Concrete constructions of risk-adjusted e-variables are provided in \cite{bai2026conformal}.

\subsubsection{Algorithm and Guarantee}
Given risk-adjusted e-variables $\{E_j^{\textrm{g}}\}_{j=1}^m$ satisfying condition \eqref{eq_risk_adj_eval}, PH-RCS applies Algorithm \ref{algo_PHCS} with the e-variable $E_j$ replaced by the generalized e-variable $E_j^{\textrm{g}}$ throughout. The nested candidate sets $\mathcal{R}_k$ in \eqref{eq_Rk_def}, the FDP estimate $\hat{\alpha}(\mathcal{R}_k)$ in \eqref{eq_alpha_k_def}, and the utility optimization in \eqref{eq_PHCS_choice} all retain the same form. We denote the resulting output by $(\mathcal{R}^{\textrm{PH-RCS}}, \alpha^{\textrm{PH-RCS}})$ to distinguish it from the binary setting in Sec. \ref{sec_PHCS}. The reported level $\alpha^{\textrm{PH-RCS}}$ now serves as a reliability certificate for the generalized FDP \eqref{eq_gfdp} rather than the standard FDP \eqref{eq_FDP_def}. The following result establishes the post-hoc reliability of this generalized procedure.

\begin{theorem}[\textbf{Post-hoc risk-control guarantee}]\label{thm_ph_rcs}
Assume that, for each $j = 1, \dots, m$, the statistic $E_j^{\text{\rm{g}}} \geq 0$ satisfies \eqref{eq_risk_adj_eval}. Then the PH-RCS output $(\mathcal{R}^{\text{\rm{PH-RCS}}}, \alpha^{\text{\rm{PH-RCS}}})$ of Algorithm \ref{algo_PHCS}, with $E_j$ replaced by $E_j^{\text{\rm{g}}}$, satisfies the average requirement
\begin{equation}\label{eq_ph_rcs_guarantee}
    \mathbb{E}\left[\frac{\text{\rm{FDP}}^{\text{\rm{g}}}(\mathcal{R}^{\text{\rm{PH-RCS}}}, \mathcal{Y}^{\text{\rm{test}}})}{\alpha^{\text{\rm{PH-RCS}}}}\right] \leq 1,
\end{equation}
where the average is evaluated with respect to the joint distribution of the selected set $\mathcal{R}^{\text{\rm{PH-RCS}}}$ and of the true test labels $\mathcal{Y}^{\text{\rm{test}}}$.
\end{theorem}

\textit{Proof:} {\color{blue}See Appendix \ref{apdx_proof_rcs}.}

% {\color{blue}The framework also accommodates priority weighting of test inputs, which preserves the post-hoc guarantee. See Appendix \ref{apdx_proof_weighted} for details.}

\section{Experiments}\label{sec_exp}
{\color{blue}In this section, we evaluate PH-CS under the constrained-size utility \eqref{eq_utility_size_first}, the additive trade-off utility \eqref{eq_u_add}, and the multiplicative trade-off utility \eqref{eq_u_mult}. We compare PH-CS against conventional CS \cite{jin2023selection} and three data-driven baselines that instead estimate the FDP from p-variables, from a data split, or from a PAC bound in lieu of the average ratio condition \eqref{eq_heur_FDR}.} {\color{blue}Further results on additional datasets and regressors are collected in Appendix~\ref{apdx_add_exp}, and Appendix~\ref{apdx_asymp_behavior} empirically validates the large-sample behavior of Proposition~\ref{prop_asymp}.}

\subsection{Settings}
We consider two different experimental settings, including both synthetic and real data.

\subsubsection{Synthetic Data}
Following reference \cite{jin2023selection}, in each run, in the synthetic-data setting, we independently generate a training sample, a calibration sample, and a test batch with sizes $1000$, $n = 1000$, and $m = 100$, respectively, using the same data-generating mechanism. According to this mechanism, the covariate vector $X=[X^{(1)}, \ldots, X^{(20)}]$ has i.i.d. entries with uniform distribution $\textrm{U}([-1,1])$, and the response follows the relationship $Y = f(X) + \varepsilon$ with function $f(X) = 5\big(X^{(1)} X^{(2)} + e^{X^{(4)} - 1}\big)$ and additive noise satisfying one of the following models: \textit{(i)} homoscedastic noise: $\varepsilon \sim \mathcal{N}\big(0,(0.15)^2\big)$; and \textit{(ii)} heteroscedastic noise: $\varepsilon \mid X \sim \mathcal{N}\big(0, \tilde{\sigma}(X)^2\big)$, where $\tilde{\sigma}(X) = 0.1(5.5 - |f(X)|)/2$. All results are averaged over $100$ independent trials. As in \cite{jin2023selection}, we use gradient boosting as the regressor to obtain the prediction function $\mu(X)$. We set the target outcome to $c_j = c = 0$ for all test inputs $j = 1, \ldots, m$.

% {\color{blue}While the realized FDP \eqref{eq_FDP_def} measures reliability, we also report the \emph{power} to measure effectiveness, defined as the fraction of test points satisfying the requirement \eqref{eq_perform_req} that are selected, i.e.,
% \begin{align}\label{eq_power_def}
%     \textrm{Power}(\mathcal{R}, \mathcal{Y}^{\textrm{test}}) = \frac{\sum_{j=1}^m \mathds{1}\{j\in\mathcal{R}, Y_{n+j} > c_j\}}{\max\{1, \sum_{j=1}^m \mathds{1}\{Y_{n+j} > c_j\}\}}.
% \end{align}
% Also known as the true discovery rate, the power complements the FDP, with larger values preferable.}

\subsubsection{Real Data}
We also consider three real-world datasets, which are randomly split into training, calibration, and test subsets at each run:
\begin{itemize}
    \item \textbf{Recruitment}: The campus recruitment dataset \cite{jin2023selection} has $215$ samples, where $Y=1$ indicates a successful placement, and the covariates $X\in\mathbb{R}^{21}$ include features such as exam percentages, degree type, and specialization. In each run, the data are split into $45$ training samples, $85$ calibration samples, and $85$ test samples. We use a gradient boosting classifier as the prediction model $\mu(X)$.
    \item \textbf{Musk}: The Musk dataset \cite{bashari2023derandomized} contains $6598$ molecular conformations with covariates $X\in\mathbb{R}^{166}$ describing spatial shape features, and a binary label $Y$ indicating musk activity ($Y=1$) or not ($Y=0$). In each run, the data are split into training, calibration, and test subsets with proportions $(0.8, 0.1, 0.1)$. We use a support vector machine as the prediction model $\mu(X)$.
    \item \textbf{Shuttle}: The Shuttle dataset \cite{bashari2023derandomized} contains $58000$ samples with covariates $X\in\mathbb{R}^{9}$ representing sensor readings, and label $Y$ binarized by setting $Y=1$ for all classes other than class $1$, and $Y=0$ for class $1$. In each run, the data are split into training, calibration, and test subsets with proportions $(0.5,0.25,0.25)$. We use logistic regression as the prediction model $\mu(X)$.
\end{itemize}
{\color{blue}For all three real-data classification datasets, we use the same threshold $c_j=c=0$, so that the requirement $Y_{n+j}>c_j$ identifies the positive class $Y_{n+j}=1$.}

\begin{figure*}[t]
    \centering
    {
    \includegraphics[width = 0.294\textwidth]{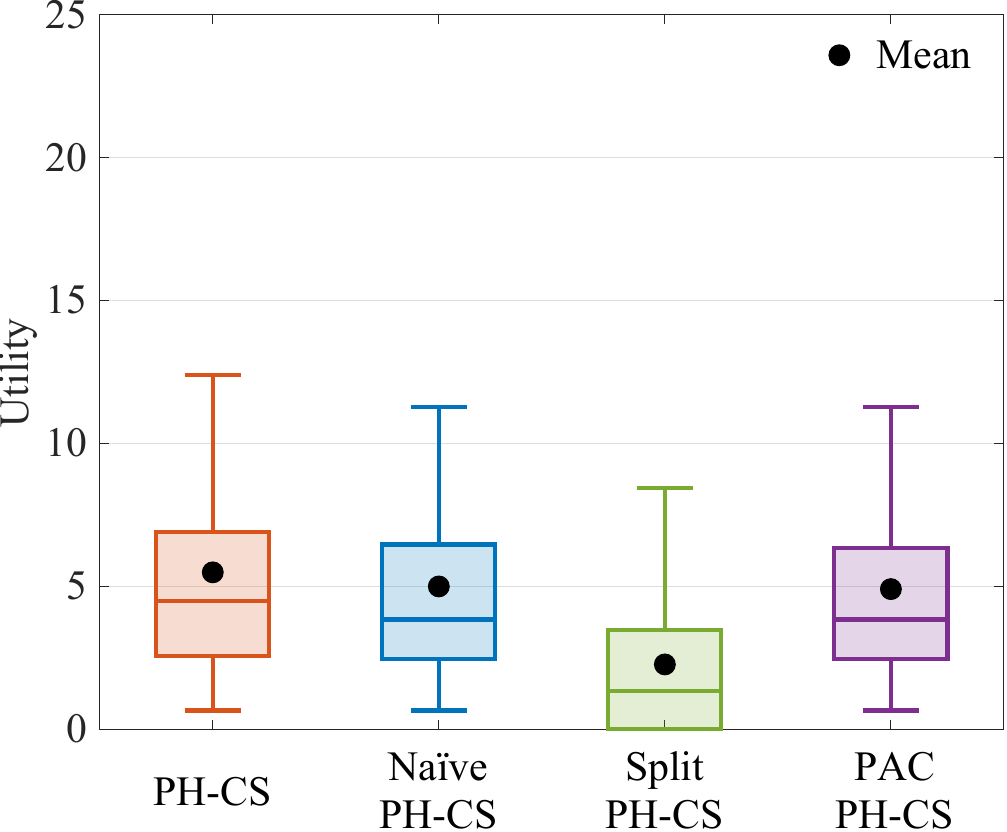}
    \includegraphics[width = 0.3\textwidth]{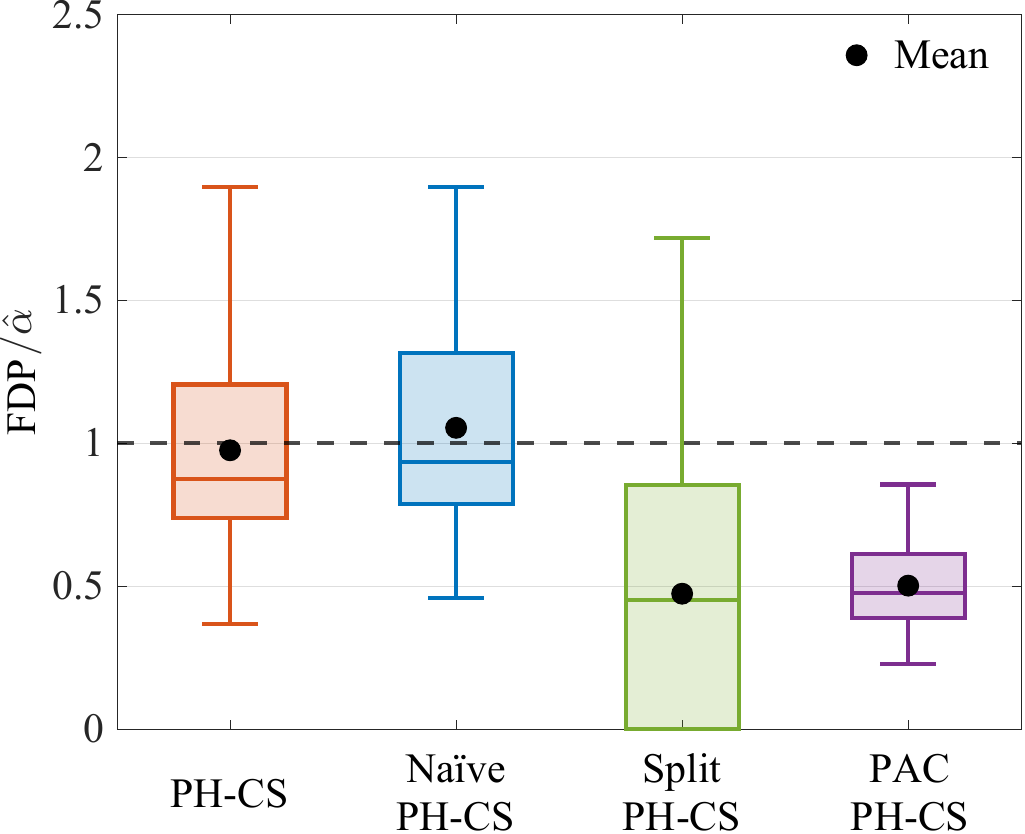}
    % \\ \includegraphics[width = 0.32\textwidth]{./figure/F_rec_add_FDP_baselines.pdf}
    \includegraphics[width = 0.3\textwidth]{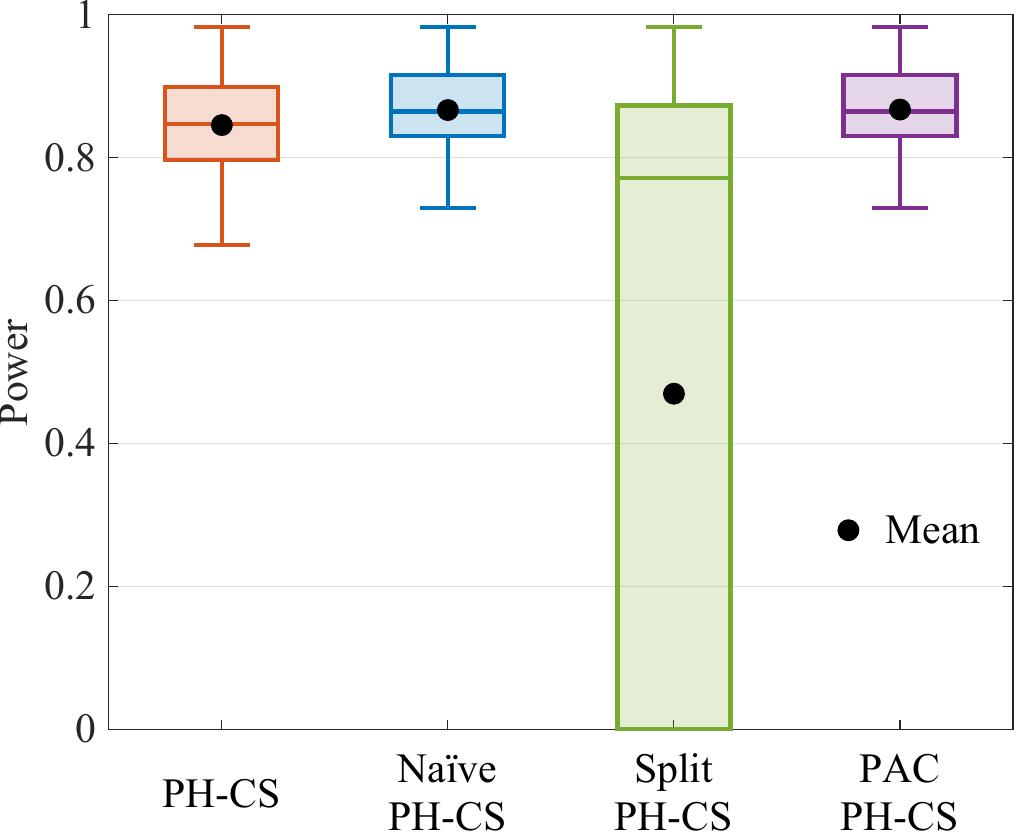}
    }
    \vspace{-1.2mm}
    \caption{{\color{blue}Box plots of the realized utility (left), the ratio between the realized FDP \eqref{eq_FDP_def} and the declared level $\hat{\alpha}$ (middle), and the power (right) under the additive trade-off utility \eqref{eq_u_add} on the Recruitment dataset, for PH-CS and the three baselines of Sec.~\ref{sec_exp_baselines}.}}
    \vspace{-5mm}
    \label{fig_baselines_rec}
\end{figure*}
\begin{figure*}[t]
    \centering
    {
    \hspace{-2mm}\includegraphics[width = 0.305\textwidth]{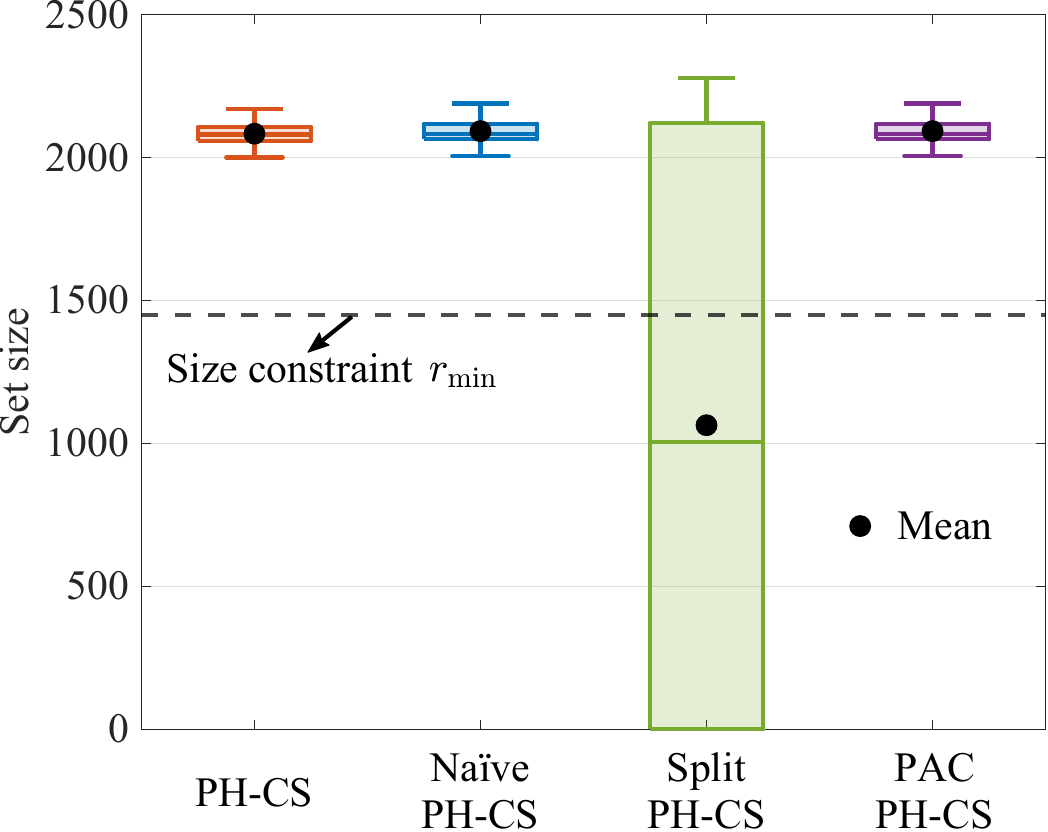}
    \includegraphics[width = 0.3\textwidth]{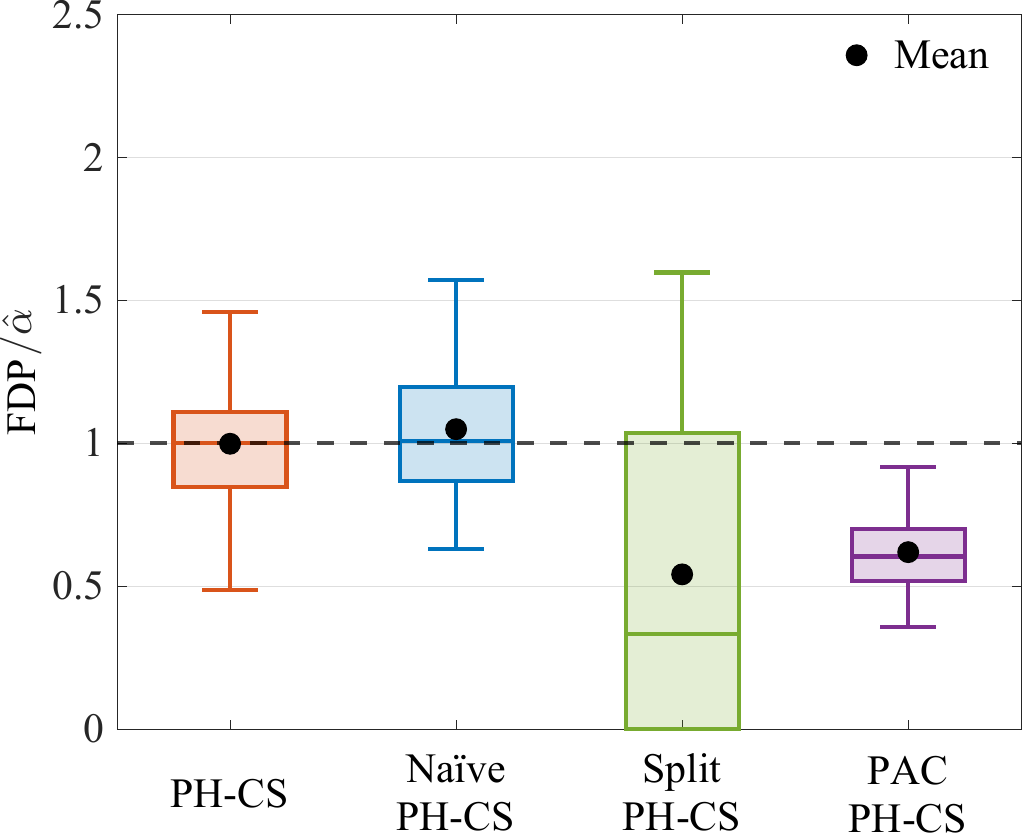}
    % \\ \includegraphics[width = 0.32\textwidth]{./figure/F_shuttle_Csize_FDP_baselines.pdf}
    \includegraphics[width = 0.3\textwidth]{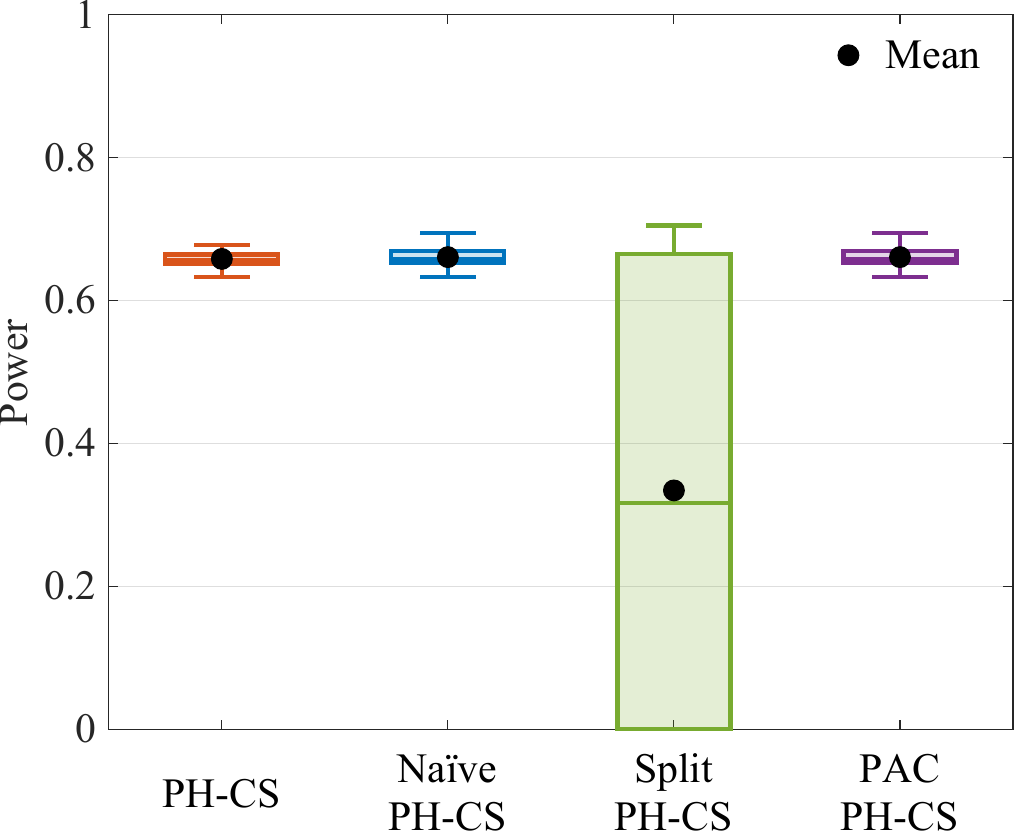}
    }
    \vspace{-1.2mm}
    \caption{{\color{blue}Box plots of the selected set size (left), the ratio between the realized FDP \eqref{eq_FDP_def} and the declared level $\hat{\alpha}$ (middle), and the power (right) under the constrained-size utility \eqref{eq_utility_size_first} on the Shuttle dataset, for PH-CS and the three baselines of Sec.~\ref{sec_exp_baselines}.}}
    \vspace{-5mm}
    \label{fig_baselines_shuttle}
\end{figure*}

\subsubsection{Score Function}
For the real-data classification tasks, the prediction $\mu(X)$ represents the predicted probability for class $Y=1$, while for the synthetic regression setting, the model output is first mapped to the interval $(0,1)$ via min-max normalization using the range observed on the training set.
As discussed in Sec. \ref{sec_conv_prelim}, the conformity score $S(X,Y)$ should capture the likelihood that the outcome exceeds $Y$, and must be non-negative and non-increasing in $Y$ (see \eqref{eq_sco_mon}). Following \cite{jin2023selection}, given the target $c$, we adopt the piece-wise constant function 
\begin{equation}\label{eq_score}
    S(X, Y) = \begin{cases}
    \Big(\dfrac{\mu(X)}{1 - \mu(X)}\Big)^{\gamma}, & \text{if } Y \leq c, \\[6pt]
    \delta, & \text{if } Y > c,
    \end{cases}
\end{equation}
where $\gamma > 0$ is a tuning parameter and $\delta = 10^{-6}$ is a small constant.
When $Y \leq c$, the score increases with the predicted quality $\mu(X)$, and a larger hyperparameter $\gamma$ amplifies this contrast, improving the discriminative power of the e-BH procedure for well-trained predictors. When $Y > c$, the score is clipped to $\delta$, reducing the denominator of the e-variable \eqref{eq_e_vari} and boosting the e-values for promising inputs.
{\color{blue}A larger parameter $\gamma$ sharpens the score \eqref{eq_score} toward the hard threshold used by indicator-type e-variables \cite{bashari2023derandomized, lee2025selection, bai2026conformal}, which yield a single selection set, whereas the continuous score generates the entire path \eqref{eq_Rk_def}. We use $\gamma=3$ on the well-separated synthetic data and $\gamma=50$ on the harder real data.
% As validity is independent of $\gamma$ by Theorem~\ref{theo_posthoc} and power is robust to it, an ablation is provided in Appendix~\ref{...}.
}

{\color{blue}
\subsubsection{Performance Metrics}
We report the selected set size $|\mathcal{R}|$, the realized FDP \eqref{eq_FDP_def}, its declared estimate $\alpha^{\textrm{PH-CS}}$ \eqref{eq_PHCS_alpha}, and the utility \eqref{eq_utility_obj}, all introduced earlier, together with the \emph{power}, i.e., the fraction of test points satisfying the requirement \eqref{eq_perform_req} that are selected:
\begin{align}\label{eq_power_def}
    \textrm{Power}(\mathcal{R}, \mathcal{Y}^{\textrm{test}}) = \frac{\sum_{j=1}^m \mathds{1}\{j\in\mathcal{R}, Y_{n+j} > c_j\}}{\max\{1, \sum_{j=1}^m \mathds{1}\{Y_{n+j} > c_j\}\}}.
\end{align}
Also known as the true discovery rate, the power  complements the FDP, with larger values preferable.}

{\color{blue}
\subsection{Baseline Methods}\label{sec_exp_baselines}
Beside conventional CS \cite{jin2023selection}, which targets a pre-specified FDR level $\alpha_{\textrm{max}}$, we compare PH-CS against three data-driven baselines that also address the post-hoc objective \eqref{eq_utility_obj}. All three build on the BH counterpart of the e-BH path \eqref{eq_Rk_def} used by PH-CS, spanning the same range of set sizes, and differ only in how the FDP estimate of the selected set is constructed.

\begin{itemize}
    \item[\textit{1)}] \textit{Naïve PH-CS} takes the nominal BH level of each candidate set as its FDP estimate and keeps the most useful operating point. Since this level is chosen after inspecting the data, it carries no validity guarantee.
    \item[\textit{2)}] \textit{Split PH-CS} restores a guarantee through data splitting. It selects the level on one half of the calibration data, and runs conventional CS at that fixed level on the other half. Thus, the guarantee of Theorem \ref{theo_conv_fdr} applies, but at the cost of calibrating on only part of the data.
    \item[\textit{3)}] \textit{PAC PH-CS} replaces the in-expectation reliability of PH-CS with a high-probability bound that upper bounds the realized FDP of all candidate sets simultaneously \cite{song2026everywhere, gazin2024transductive}. The bound holds with probability $1-\eta$, but requires the extra confidence level $\eta$ and, being simultaneous over all sets, may be loose at the selected operating point.
\end{itemize}

The constructions of the three FDP estimates are detailed in Appendix \ref{apx_baselines}.
}

\begin{figure*}[t]
    \centering
    {\hspace{-5mm}
    \includegraphics[width = 0.305\textwidth]{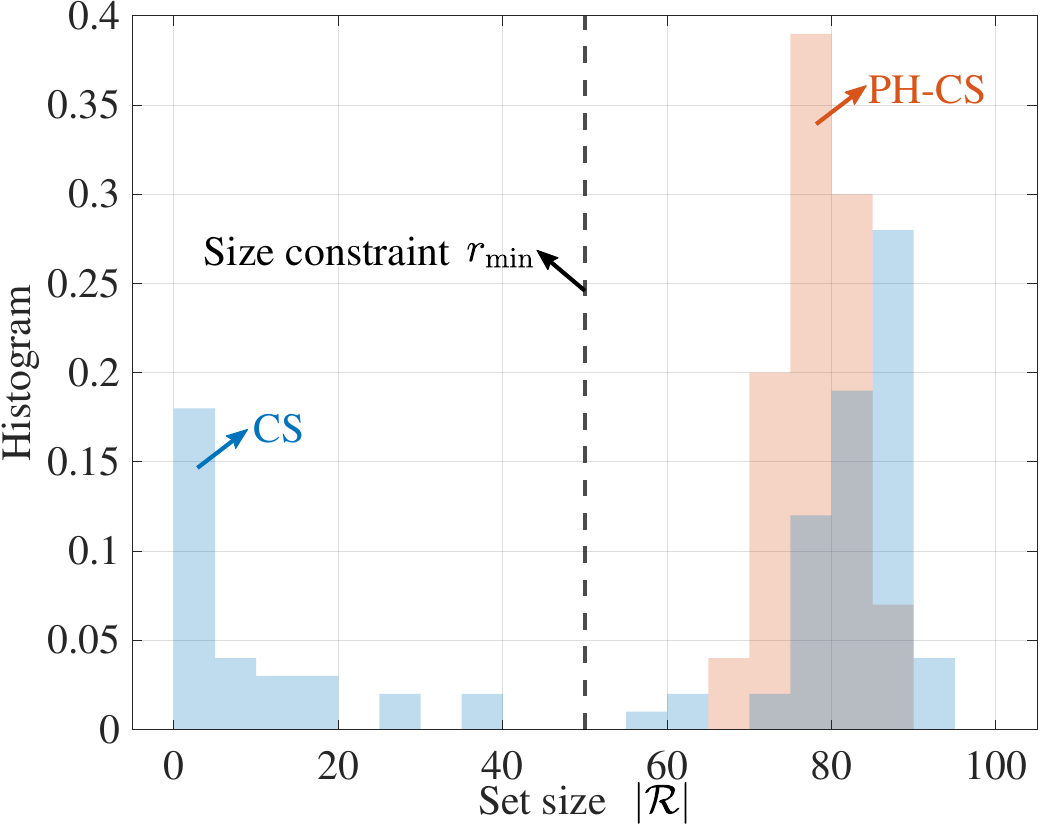}
	\includegraphics[width = 0.3\textwidth]{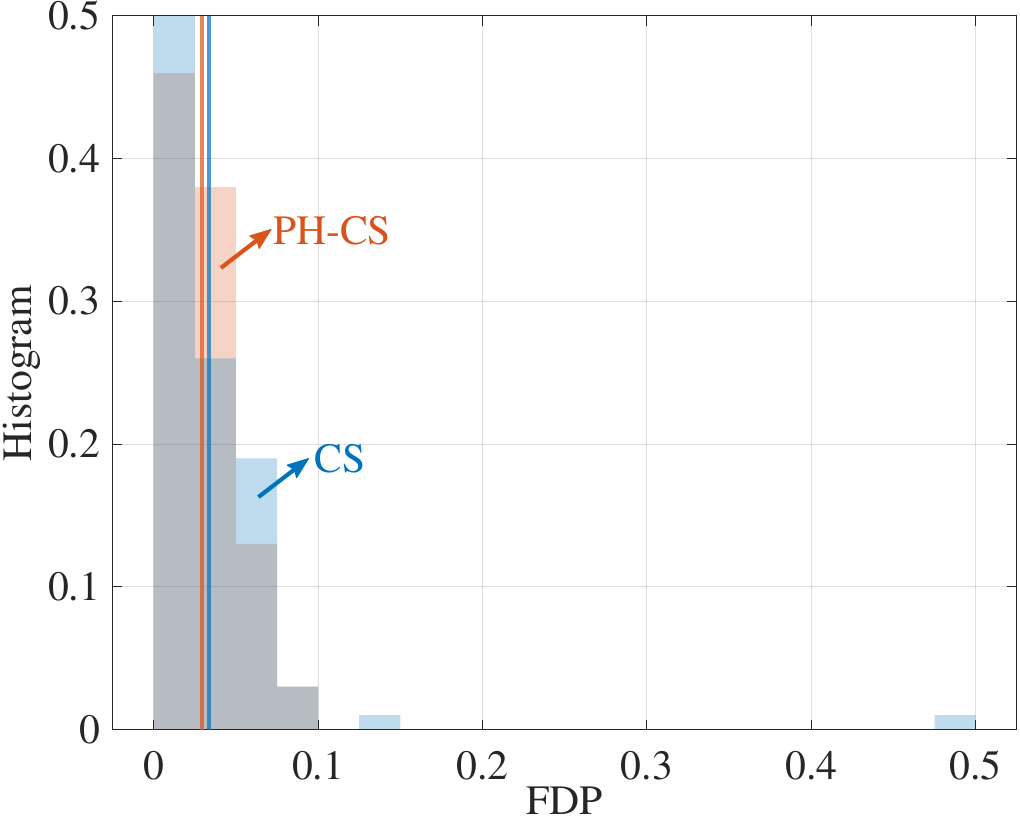}
	\includegraphics[width = 0.3\textwidth]{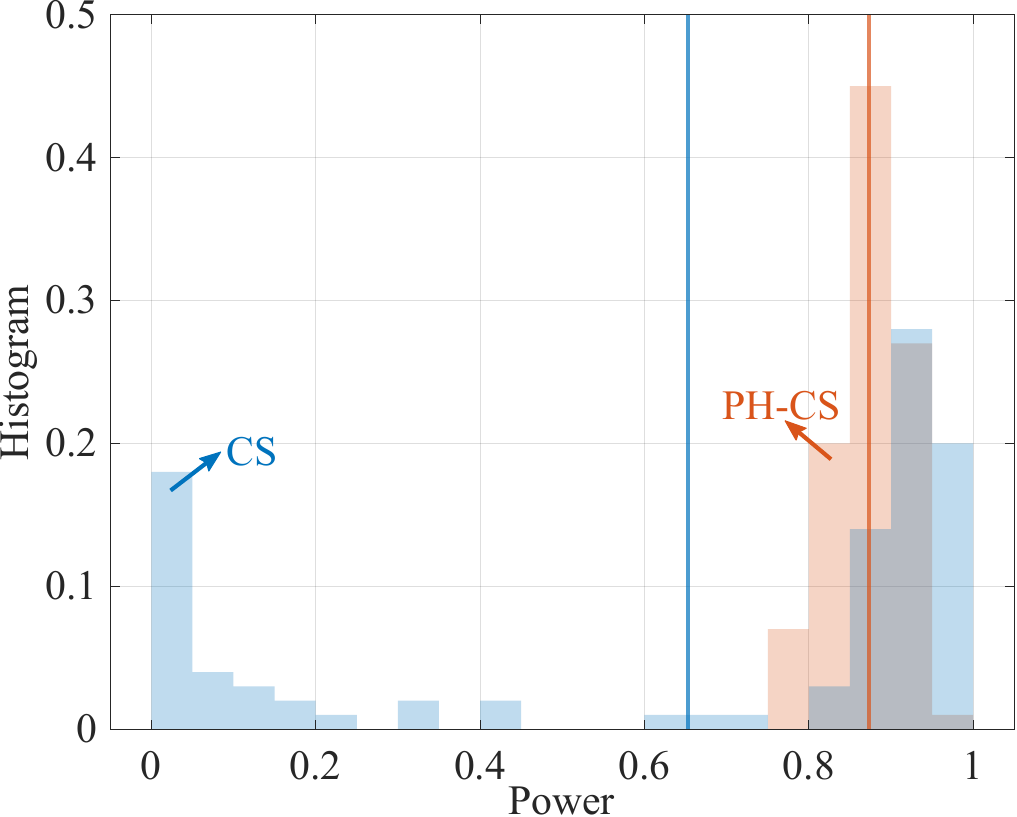}\\

    \includegraphics[width = 0.3\textwidth]{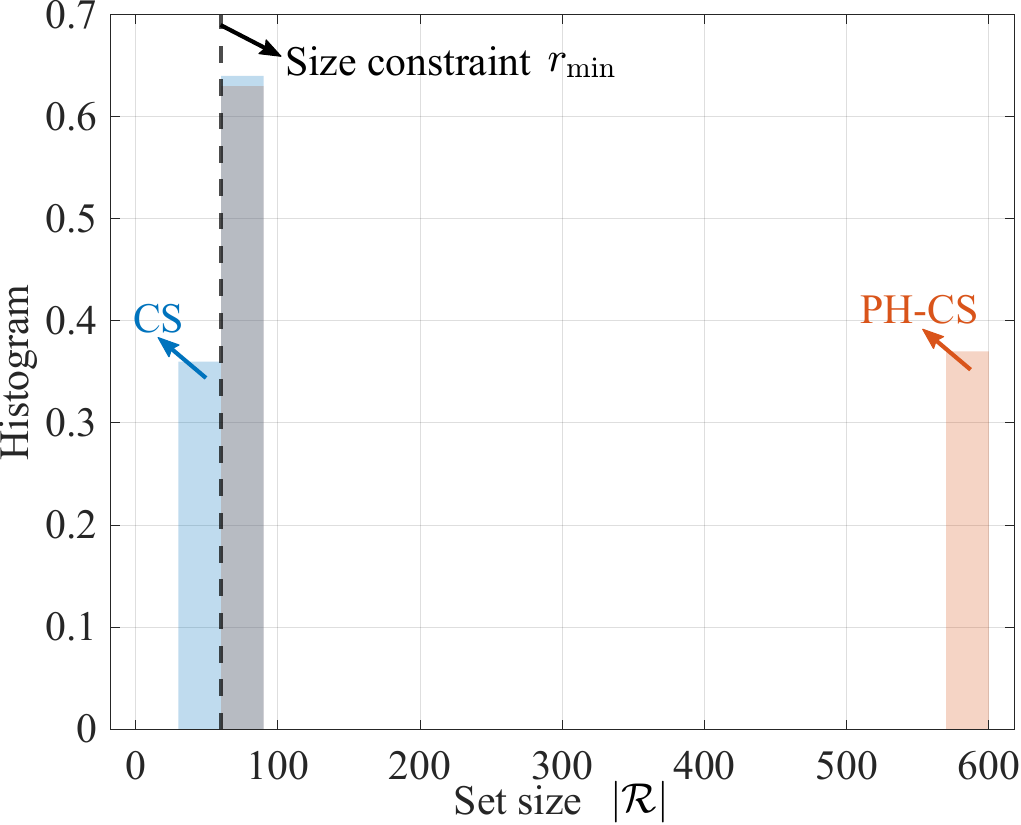}
    \includegraphics[width = 0.3\textwidth]{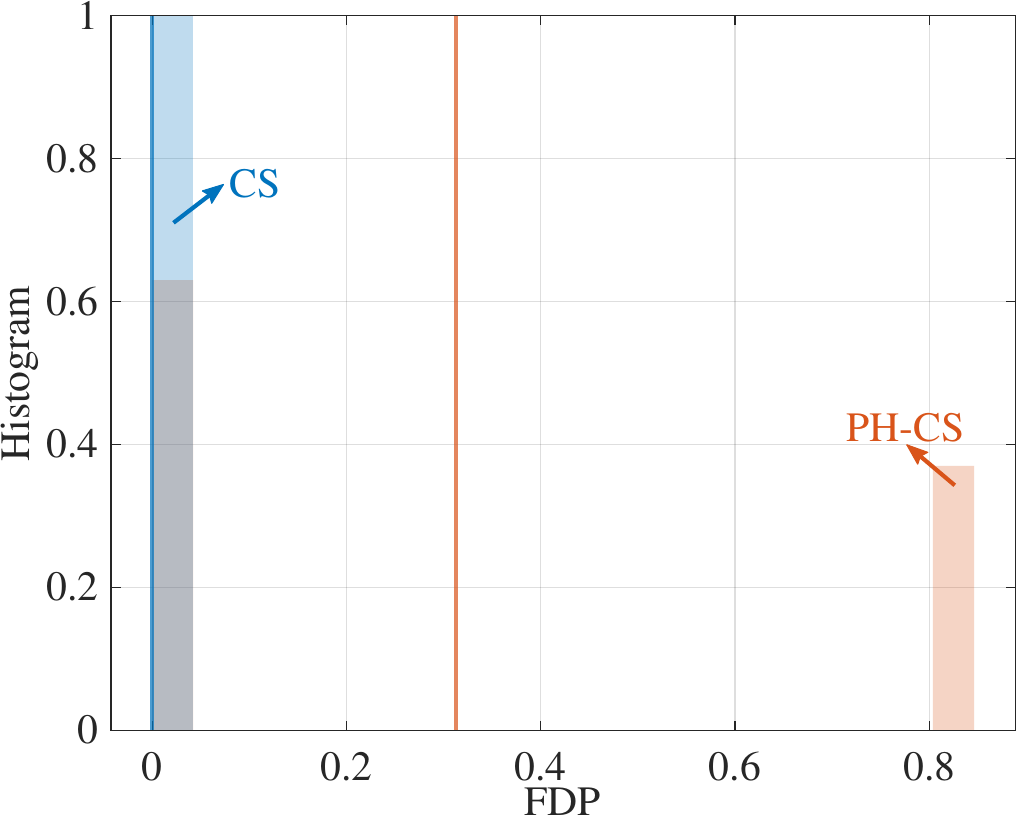}
    \includegraphics[width = 0.3\textwidth]{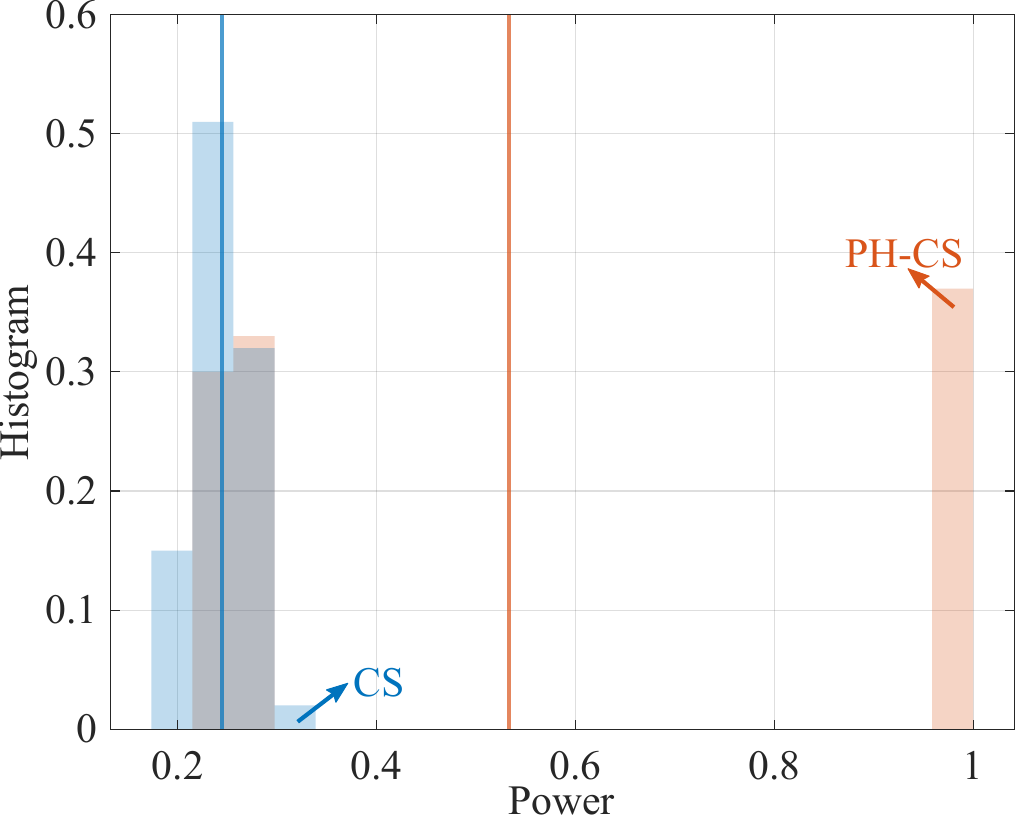}
    }
    \vspace{-1.2mm}
    \caption{{\color{blue}Histograms of the selected set size (left), realized FDP (middle), and power (right) under the constrained-size utility \eqref{eq_utility_size_first} on the synthetic data with homoscedastic noise (top row) and the Musk dataset (bottom row). In the middle and right panels, the blue and red vertical lines indicate the average values for CS and PH-CS, respectively.}}
    \vspace{-5mm}
    \label{fig_Csize}
\end{figure*}

{\color{blue}
\subsection{Comparison with Baselines}\label{sec_exp_baselines_results}
We compare PH-CS against the three post-hoc baselines of Sec.~\ref{sec_exp_baselines} on the Recruitment dataset under the additive trade-off utility \eqref{eq_u_add} with $U(r,\alpha)=\log r-\lambda\log(1/(1-\alpha))$, $r=|\mathcal{R}|$, and $\lambda=15$, and on the Shuttle dataset under the constrained-size utility \eqref{eq_utility_size_first} with $r_{\textrm{min}}=0.1m$. The PAC PH-CS approach uses confidence level $\eta=0.1$.

Fig.~\ref{fig_baselines_rec} and Fig.~\ref{fig_baselines_shuttle} report the achieved utility or selected set size, the ratio between the realized FDP \eqref{eq_FDP_def} and the declared level $\hat{\alpha}$, and the power \eqref{eq_power_def}.
PH-CS attains the highest utility, meets the size constraint in every run, and remains competitive in power. Its average ratio between the realized FDP and the declared level stays closest to one, so its FDP estimate satisfies the post-hoc reliability \eqref{eq_post_FDP_first} and is also the tightest among all methods. 

Na\"ive PH-CS instead reports the nominal level of the selected set as its FDP estimate, which is chosen by maximizing the utility and is no longer valid in advance. Its average ratio exceeds one, so the post-hoc reliability \eqref{eq_post_FDP_first} no longer holds. Split PH-CS and PAC PH-CS keep the average ratio below one and remain reliable, but are overly conservative. Split PH-CS uses only half of the data for calibration, which weakens its selection, yielding the lowest utility on the Recruitment dataset and frequently violating the size constraint on the Shuttle dataset. PAC PH-CS upper bounds the FDP of all candidate sets at once, so its estimate is loose at the selected set. Overall, PH-CS is the only method that is both reliable and effective for the selection task.

As PH-CS outperforms all three post-hoc baselines, we adopt it as our method and compare it with conventional CS \cite{jin2023selection} in the remaining experiments.}

\subsection{Constrained-Size Utility}\label{sec_exp_csize}
Under the constrained-size utility in \eqref{eq_utility_size_first} {\color{blue}with $r=|\mathcal{R}|$}, the goal is to select, for each test batch, a set of test samples of size at least $r_{\textrm{min}}$, while keeping the realized FDP as small as possible. In our experiments, we set $r_{\textrm{min}}=0.5m$ for synthetic data and $r_{\textrm{min}}=0.1m$ for real data, corresponding to selecting at least $50\%$ and $10\%$ of the test batch, respectively. {\color{blue}The resulting histograms of the selected set size, realized FDP, and power are shown in Fig.~\ref{fig_Csize} for the synthetic data with homoscedastic noise and the Musk dataset.} To enable a fair comparison with CS \cite{jin2023selection}, we set the target FDR level $\alpha_{\textrm{max}}$ in \eqref{eq_BH_k}--\eqref{eq_BH_set} to the empirical average of the declared levels $\alpha^{\textrm{PH-CS}}$ produced by PH-CS over the $100$ trials. This way, both PH-CS and CS guarantee the same FDR level, while PH-CS also controls the selected set size for test run.

% A consistent trend is observed{\color{blue}}.
In the {\color{blue}left panels} of Fig.~\ref{fig_Csize}, the selected sets produced by PH-CS are seen to always satisfy the minimum-size requirement, whereas the selected sets produced by CS fall below this threshold in a non-negligible fraction of realizations. At the same time, the {\color{blue}middle panels} show that PH-CS is generally competitive in terms of FDP, although the capacity for size control may come at the cost of a larger FDP on test batches, most notably on the Musk dataset. {\color{blue}The right panels further show that PH-CS attains higher power, as it reliably selects a set of the required size while CS, tied to a single pre-specified level, often selects far fewer points.} Overall, these results highlight the practical advantage of the post-hoc construction in \eqref{eq_PHCS_choice}. Unlike CS, which is tied to a single pre-specified FDR level, PH-CS can adaptively select an operating point so as to satisfy the size requirement $r_{\textrm{min}}$ in each realization, while obtaining competitive FDR performance as compared to CS and offering reliable estimates of the FDP. This latter property is illustrated in Sec.~\ref{sec_FDP_estimate}.

\begin{figure*}[t]
    \centering
    {
	\includegraphics[width = 0.294\textwidth]{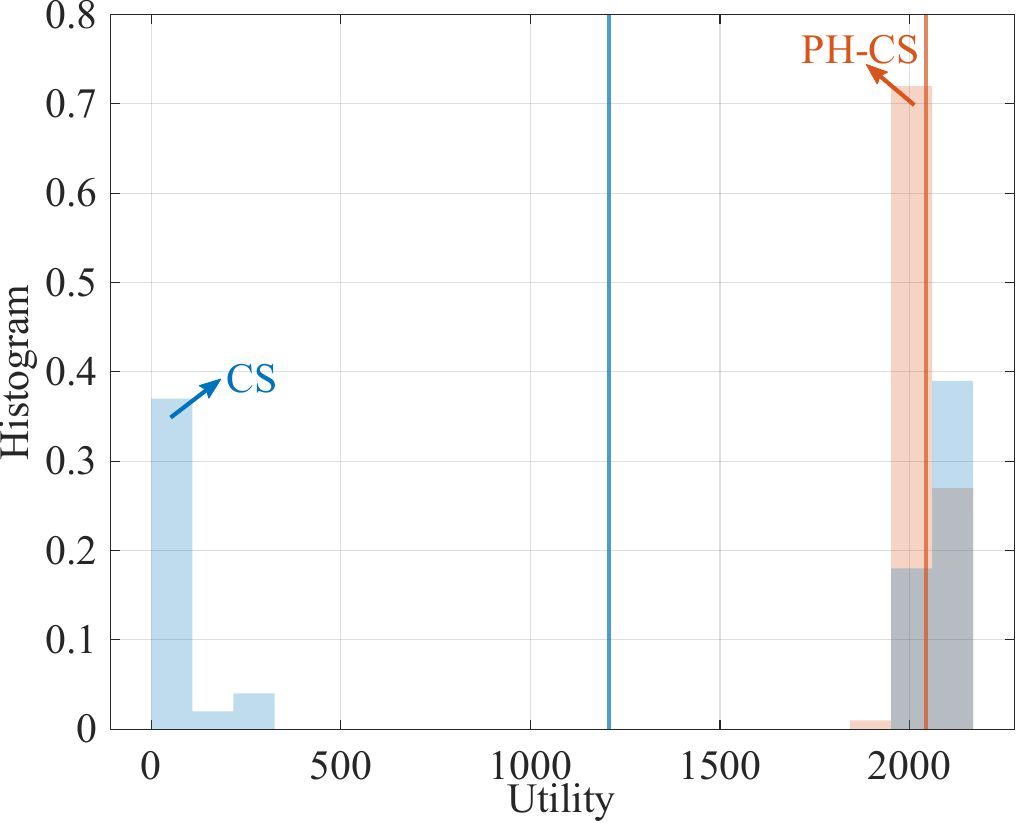}
    \includegraphics[width = 0.295\textwidth]{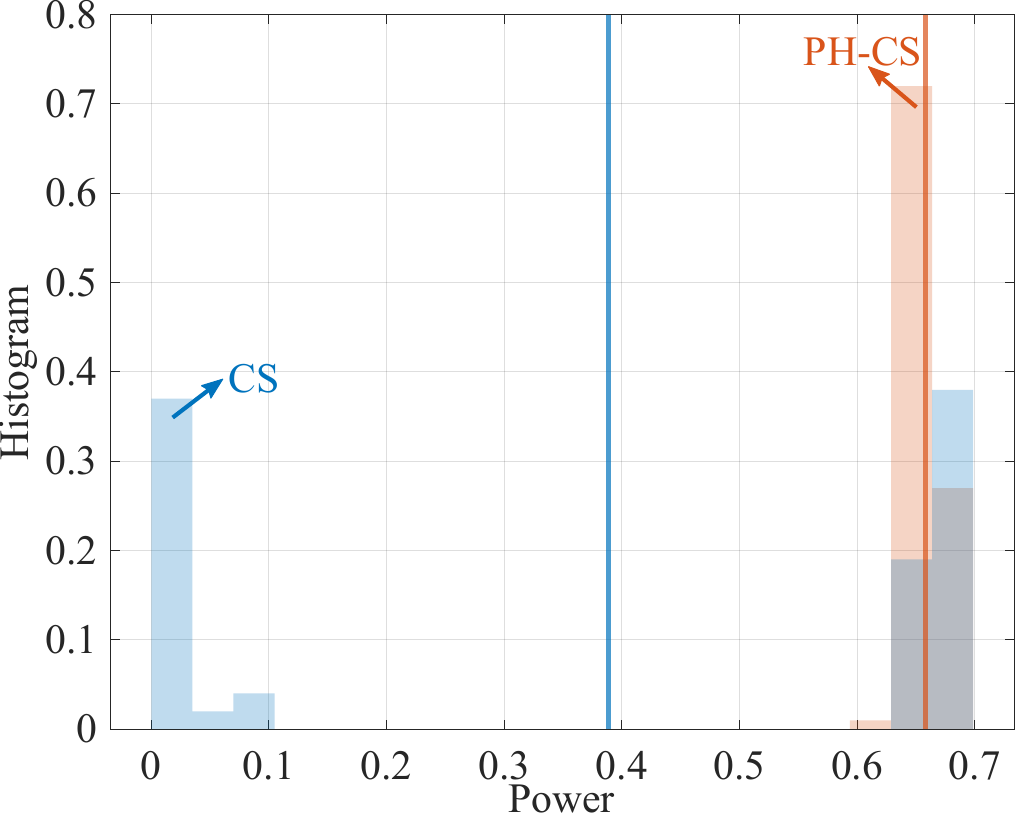}
    \includegraphics[width = 0.31\textwidth]{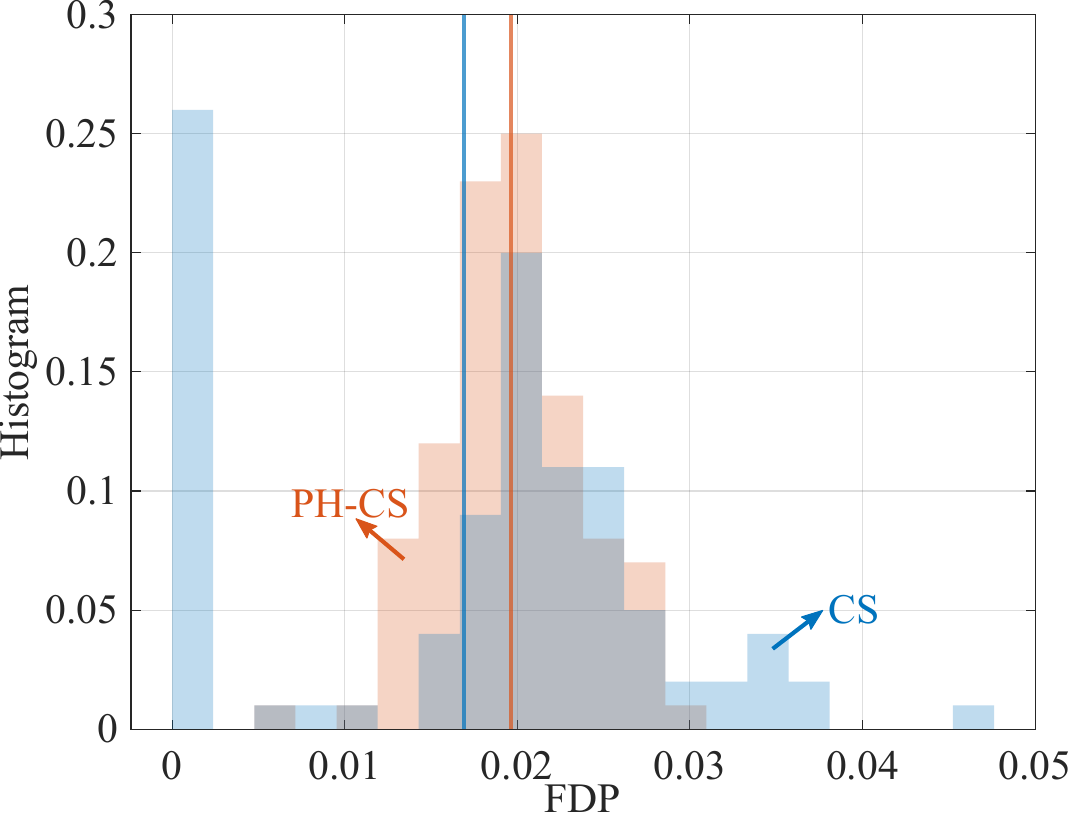}
	\includegraphics[width = 0.3\textwidth]{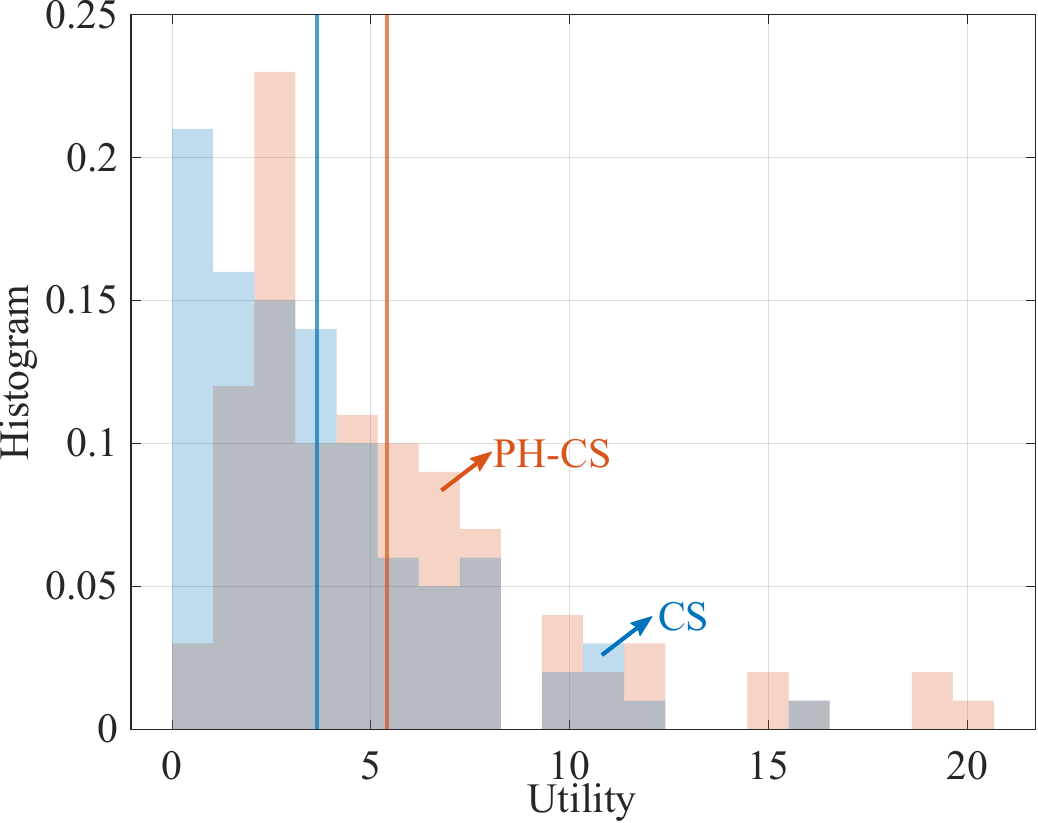}
    \includegraphics[width = 0.3\textwidth]{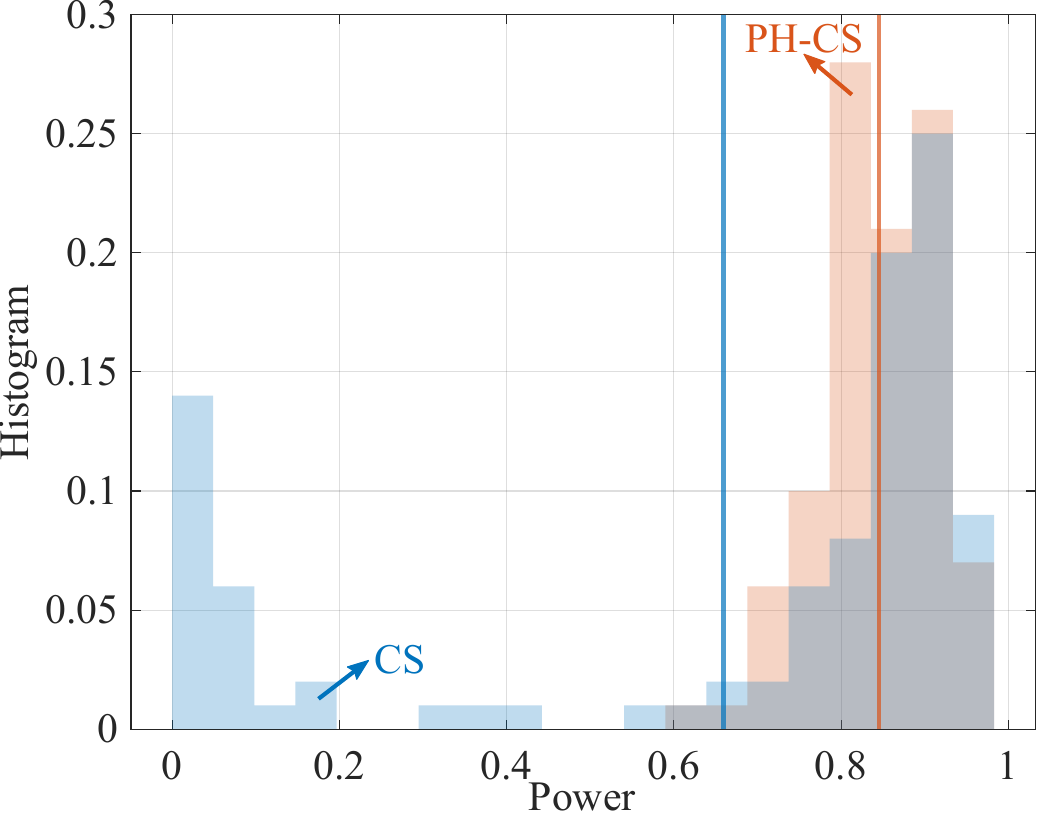}
    \includegraphics[width = 0.3\textwidth]{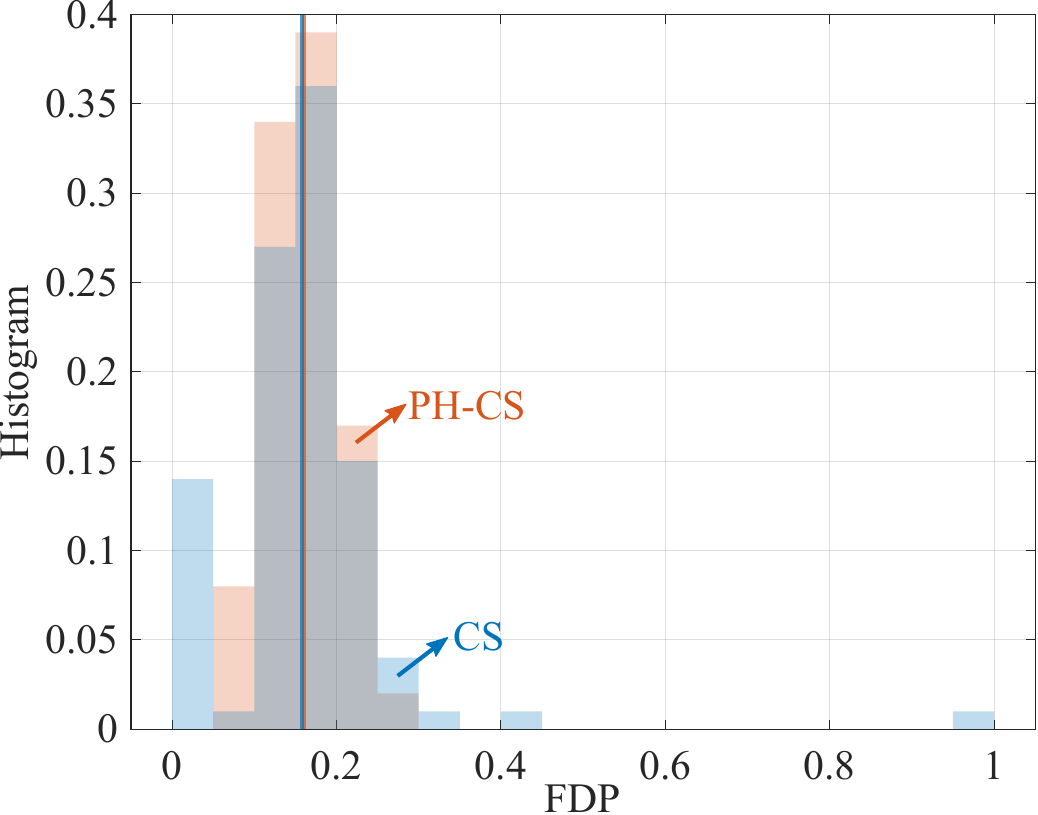}
    }
    \vspace{-1.2mm}
    \caption{{\color{blue}Histograms of the realized utility (left), power (middle), and FDP (right) under the multiplicative utility \eqref{eq_u_mult} on the Shuttle dataset (top row) and the additive utility \eqref{eq_u_add} with weighted set size \eqref{eq_w_size} on the Recruitment dataset (bottom row). In each panel, the blue and red vertical lines indicate the average values for CS and PH-CS, respectively.}}
    \vspace{-5mm}
    \label{fig_power_aware}
\end{figure*}

{\color{blue}
\subsection{Power-Aware Utility}\label{sec_power_aware}
We now consider two utilities that target the correct selections rather than the raw set size: (\emph{i}) the multiplicative utility \eqref{eq_u_mult} with $U(r,\alpha)=r(1-\alpha)$ and $r=|\mathcal{R}|$; and (\emph{ii}) the additive utility \eqref{eq_u_add} with $U(r,\alpha)=\log r-\lambda\log(1/(1-\alpha))$, $\lambda=15$, and the weighted size $r$ given by \eqref{eq_w_size}, where $\pi_j$ estimates the probability of the event $\{Y_{n+j}>c_j\}$ based on the model. Specifically, for classification, we set $\pi_j = \mu(X_{n+j})$, and for regression, we model $Y_{n+j} \mid X_{n+j} \sim \mathcal{N}(\mu(X_{n+j}), \sigma(X_{n+j})^2)$, with $\sigma(X)$ estimated from the training residuals, and set
\begin{align}\label{eq_pi_gaussian}
    \pi_j = 1 - \Phi\left(\frac{c_j - \mu(X_{n+j})}{\sigma(X_{n+j})}\right),
\end{align}
where $\Phi$ is the standard normal cumulative distribution function.

Fig.~\ref{fig_power_aware} reports the realized utility, power, and FDP for the multiplicative utility on the Shuttle dataset and the weighted utility on the Recruitment dataset, with CS matched to the average declared level as before. Under both utilities, PH-CS attains a higher average utility and a higher power than CS at a comparable FDP. This shows that, by rewarding the estimated correct selections, the utility can be steered toward power while the FDP remains controlled.
}

\subsection{Empirical Evaluation of the FDP Estimate}\label{sec_FDP_estimate}
\begin{figure*}[t]
    \centering
    {
    \includegraphics[width = 0.305\textwidth]{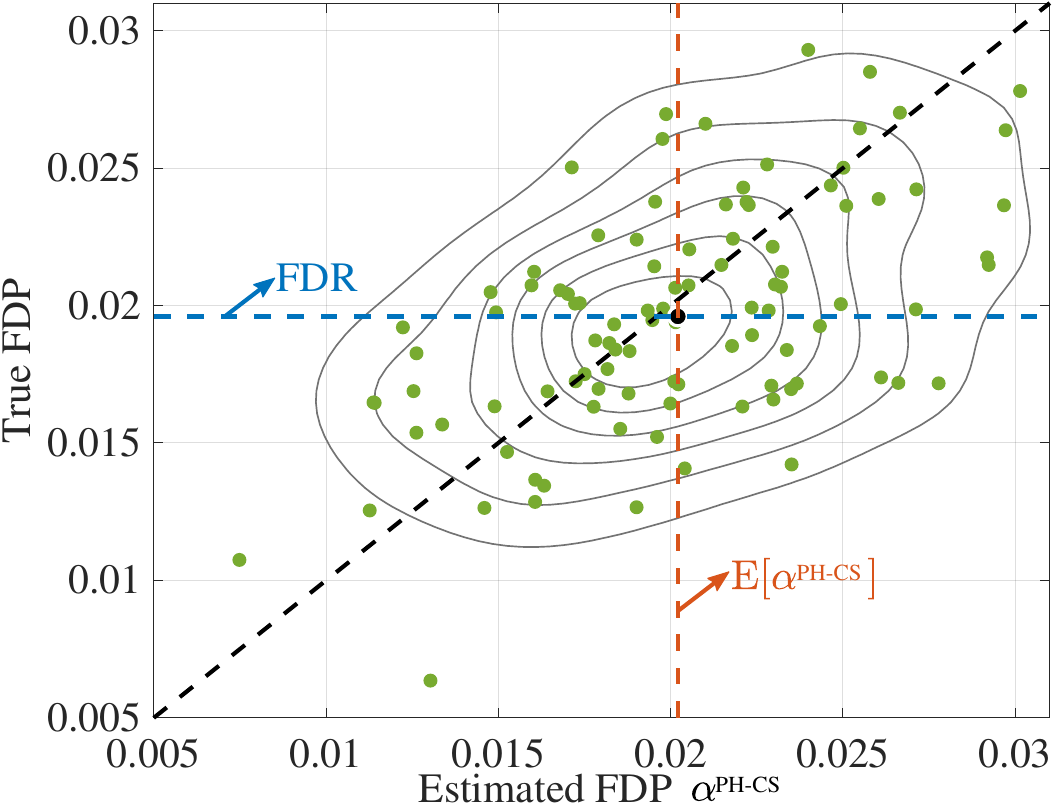}
    \includegraphics[width = 0.31\textwidth]{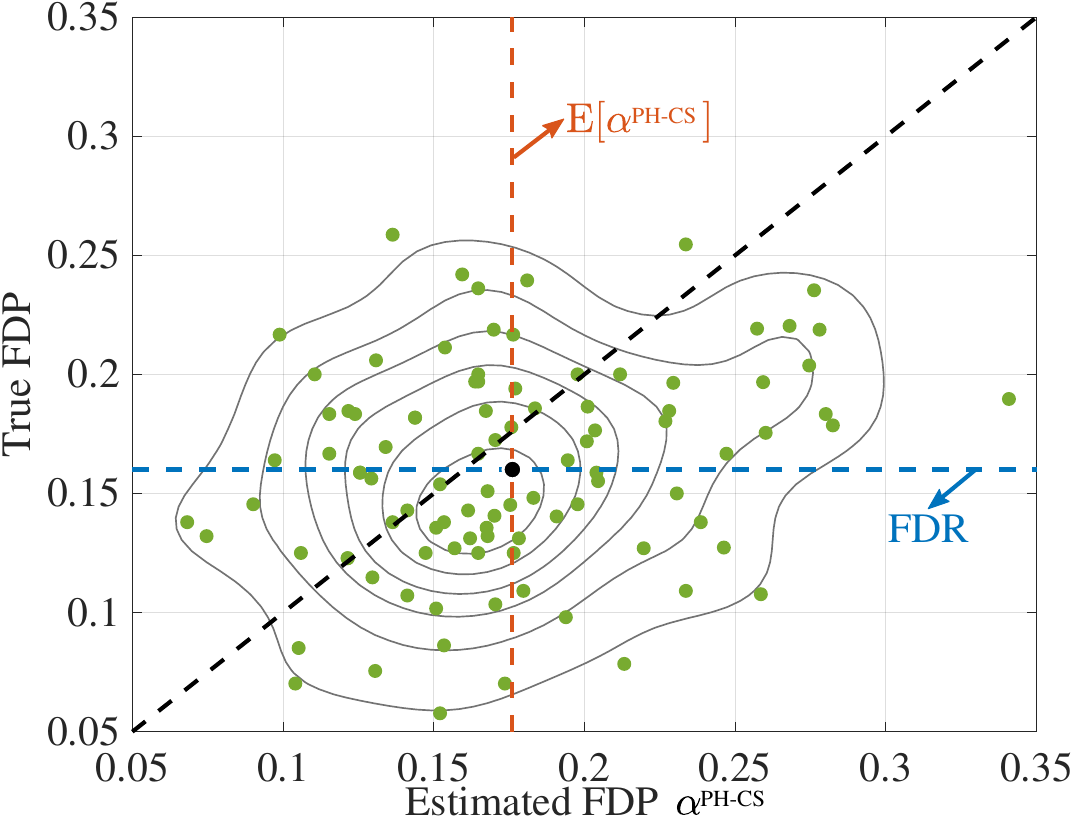}
    }
    \vspace{-1.2mm}
    \caption{Scatter plots of the realized FDP, $\textrm{FDP}(\mathcal{R}^{\textrm{PH-CS}}, \mathcal{Y}^{\textrm{test}})$ in \eqref{eq_FDP_def}, versus the estimated level $\alpha^{\textrm{PH-CS}}$ \eqref{eq_PHCS_alpha} over $100$ random seeds under the constrained-size utility \eqref{eq_utility_size_first} on the Shuttle dataset (left) and the additive trade-off utility \eqref{eq_u_add} on the Recruitment dataset (right). Contour lines show the kernel density of the joint distribution. The black dashed line is the reference at which estimated and true FDP coincide, while the vertical red and horizontal blue dashed lines mark the average estimate $\mathbb{E}[\alpha^{\textrm{PH-CS}}]$ and the true FDR, $\textrm{FDR}(R^\textrm{PH-CS})$, respectively.}
    \vspace{-5mm}
    \label{fig_FDP_estimate}
\end{figure*}

A key property of PH-CS is its capacity to provide reliable conservative estimates $\alpha^{\textrm{PH-CS}}$ of the FDP, as formalized by \eqref{eq_post_FDP}--\eqref{eq_heur_FDR}. To illustrate this, Fig. \ref{fig_FDP_estimate} reports the relationship between the realized $\textrm{FDP}(\mathcal{R}^{\textrm{PH-CS}},\mathcal{Y}^{\textrm{test}})$ and the declared level $\alpha^{\textrm{PH-CS}}$ over $100$ trials via a scatter plot evaluated on the Shuttle dataset under the constrained-size utility \eqref{eq_utility_size_first} and on the Recruitment dataset under the additive trade-off utility \eqref{eq_u_add}.

The main observation is that, in both cases, the joint distribution of true and estimated FDP is concentrated around the diagonal, indicating that the declared level $\alpha^{\textrm{PH-CS}}$ tracks the realized FDP reasonably well across realizations. Some points lie above the diagonal, meaning that the realized FDP can occasionally exceed the declared level, but the average estimated FDP $\mathbb{E}[\alpha^{\textrm{PH-CS}}]$ {\color{blue}is larger than} the true FDR $\textrm{FDR}(R^{\textrm{PH-CS}})$, confirming the theoretical results in \eqref{eq_post_FDP} and \eqref{eq_heur_FDR}.

\section{Conclusion} \label{sec_conclusion}
This paper has introduced post-hoc conformal selection (PH-CS), a framework that extends conventional CS by allowing the user to choose the operating point after observing the data while retaining a finite-sample reliability guarantee. PH-CS replaces conformal p-variables and the BH procedure with conformal e-variables and the e-BH procedure, exploiting the level-uniform property of e-BH to generate a path of candidate selection sets, each paired with a data-driven FDP estimate. The user then selects the operating point that maximizes an arbitrary utility balancing set size and reliability. We proved that the reported FDP estimate $\alpha^{\textrm{PH-CS}}$ satisfies the post-hoc reliability condition $\mathbb{E}[\textrm{FDP}/\alpha^{\textrm{PH-CS}}]\leq 1$ under exchangeability and score monotonicity, and extended the framework to continuous risk control (PH-RCS) and priority-weighted selection. Experiments on synthetic and real-world datasets confirmed that PH-CS consistently satisfies user-imposed size and utility constraints that CS cannot enforce, while maintaining competitive FDR control and providing reliable FDP estimates across different regressors. Overall, PH-CS presents a principled solution for utility-driven selection with post-hoc flexibility, establishing a foundation for further research in data-adaptive conformal inference.

Future work could explore extensions of PH-CS to settings with distribution shift between calibration and test data \cite{jin2023sensitivity, barber2025unifying, bai2026conformal}, the use of optimized e-variable designs \cite{koning2025fuzzy} and improved e-BH procedures \cite{xu2025bringing}, and the design of level-independent boosting strategies for conformal e-variables that improve power while preserving post-hoc validity.

\bibliographystyle{IEEEtran}
\bibliography{cite.bib}

\ifshowappendix
\appendix

\subsection{Link Between Post-hoc Reliability and FDR Control} \label{apdx_Taylor}

For completeness, we provide a short derivation supporting the relation \eqref{eq_heur_FDR}, which links the post-hoc reliability condition \eqref{eq_post_FDP_first} to an approximate upper bound on the FDR in terms of $\mathbb{E}[\hat{\alpha}(\mathcal{R})]$. The approximation is most informative when the estimate $\hat{\alpha}(\mathcal{R})$ has a sufficiently small variance.

A first-order Taylor expansion of $1/\hat{\alpha}(\mathcal{R})$ around $\mathbb{E}[\hat{\alpha}(\mathcal{R})]$ yields
\begin{align}\label{eq_taylor_alpha_star}
    \frac{1}{\hat{\alpha}(\mathcal{R})} =& \frac{1}{\mathbb{E}[\hat{\alpha}(\mathcal{R})]} - \frac{\hat{\alpha}(\mathcal{R}) - \mathbb{E}[\hat{\alpha}(\mathcal{R})]}{\mathbb{E}[\hat{\alpha}(\mathcal{R})]^2} \nonumber\\
    &+ O\big((\hat{\alpha}(\mathcal{R}) - \mathbb{E}[\hat{\alpha}(\mathcal{R})])^2\big).
\end{align}
Substituting \eqref{eq_taylor_alpha_star} into \eqref{eq_post_FDP_first} and neglecting the higher-order term yields
\begin{align}
    \mathbb{E}\Bigg[\frac{\textrm{FDP}(\mathcal{R}, \mathcal{Y}^{\textrm{test}})}{\hat{\alpha}(\mathcal{R})}\Bigg]
    \hspace{-1mm}\approx & \frac{\mathbb{E}\big[\textrm{FDP}(\mathcal{R}, \mathcal{Y}^{\textrm{test}})\big]}{\mathbb{E}[\hat{\alpha}(\mathcal{R})]} \nonumber\\
    & - \frac{\mathbb{E} \big[\textrm{FDP}(\mathcal{R}, \mathcal{Y}^{\textrm{test}})(\hat{\alpha}(\mathcal{R}) - \mathbb{E}[\hat{\alpha}(\mathcal{R})])\big]}{\mathbb{E}[\hat{\alpha}(\mathcal{R})]^2}\nonumber\\
    \approx & \frac{\mathbb{E}\big[\textrm{FDP}(\mathcal{R}, \mathcal{Y}^{\textrm{test}})\big]}{\mathbb{E}[\hat{\alpha}(\mathcal{R})]} \hspace{-1mm} = \hspace{-1mm} \frac{\textrm{FDR}(R)}{\mathbb{E}[\hat{\alpha}(\mathcal{R})]} \hspace{-1mm} \leq 1,
\end{align}
which yields the approximate bound \eqref{eq_heur_FDR}.

\subsection{Proof of Theorem \ref{theo_posthoc}} \label{apx_prof_posthoc}
We first introduce an \emph{oracle} counterpart of the conformal e-variable in \eqref{eq_e_vari}, obtained by evaluating the score at the unobserved test label $Y_{n+j}$, i.e.,
\begin{align}\label{eq_oracle_e}
    E^*_j = \frac{S_{n+j}}{\frac{1}{n+1}(\sum_{i=1}^n S_i + S_{n+j})}, ~ j  = 1, \ldots, m,
\end{align}
where $S_{n+j} = S(X_{n+j}, Y_{n+j})$ for $j = 1, \ldots, m$, and we assume $\sum_{i=1}^n S_i + S_{n+j}>0$ almost surely since the score is non-negative as defined in Sec. \ref{sec_conv_prelim}.

\subsection*{\textbf{Step 1: $\mathbb{E}[E_j^*] = 1$}}
By assumption \textit{(i)} in Theorem \ref{theo_posthoc}, for each fixed $j$, the collection $\{S_i\}_{i=1}^n\cup\{S_{n+j}\}$ is exchangeable. Hence
\begin{align}\label{eq_exp_E_star}
    \mathbb{E}[E^*_j] &= \mathbb{E}\bigg[\frac{(n+1) S_{n+j}}{\sum_{i=1}^n S_i + S_{n+j}}\bigg]\nonumber\\
    &=\frac{1}{n+1}\mathbb{E}\bigg[\sum_{k\in\{1,\ldots,n,n+j\}}\frac{(n+1) S_k}{\sum_{i=1}^n S_i + S_{n+j}} \bigg]\nonumber\\
    &=\frac{1}{n+1}\mathbb{E}[n+1]=1,
\end{align}
where the second equality uses exchangeability to replace the expectation of the single term involving $S_{n+j}$ by the average of all $n+1$ symmetric terms $S_1,\ldots, S_n, S_{n+j}$. Based on this, we apply the following steps.

\subsection*{\textbf{Step 2: Under the null $H_j$, $E_j\leq E^*_j$}}
To compare $E_j$ in \eqref{eq_e_vari} with its oracle counterpart $E_j^*$ in \eqref{eq_oracle_e}, we study how the e-variable changes with its score input while keeping the calibration sum fixed.

Fix $a = \sum_{i=1}^n S_i$ and define the scalar mapping
\begin{align}
    f(S) = \frac{(n+1)S}{a+S},
\end{align}
which is the functional form of the e-variable as a function of the score input $S$. A direct calculation gives
\begin{align}
    f^{\prime}(S) = \frac{a(n+1)}{(a+S)^2} \geq 0,
\end{align}
so $f(S)$ is non-decreasing in $S$. Consequently, for each $j$, we have
\begin{align}
    E_j=f(\hat S_{n+j}) ~ \text{and} ~ E_j^*=f(S_{n+j}).
\end{align}

By assumption \textit{(ii)} in Theorem \ref{theo_posthoc}, the score $S(X,Y)$ is monotone non-increasing in $Y$. Hence, under the null hypothesis $H_j$ in \eqref{eq_hypo}, i.e., on the event $\{Y_{n+j}\leq c_j\}$, we have
\begin{align}
    S_{n+j} = S(X_{n+j},Y_{n+j}) \geq S(X_{n+j},c_j)=\hat{S}_{n+j}.
\end{align}
Since function $f(S)$ is non-decreasing in score $S$, this implies
\begin{align}\label{eq_E_E_star}
    E_j\leq E^*_j~\textrm{under}~H_j~\textrm{for}~j=1,\ldots,m.
\end{align}

\subsection*{\textbf{Step 3: Post-hoc reliability via e-BH consistency}}
We now adapt the proof strategy of \cite[Theorem 3]{jin2023selection} to the e-variable and e-BH setting. 
% Building on Steps 1--2, we first establish the post-hoc bound for any nominal FDR level $\alpha \in (0,1)$ and the corresponding e-BH output $\mathcal{R}(\alpha)$ in \eqref{eq_set_alpha}, and then specialize the result to the data-dependent choice $(\mathcal{R}^{\textrm{PH-CS}},\alpha^{\textrm{PH-CS}})$ selected by \eqref{eq_PHCS_choice}.
Fix any $\alpha\in(0,1)$ and let $\mathcal{R} = \mathcal{R}(\alpha)$ denote the rejection set obtained by the e-BH procedure \eqref{eq_set_alpha}--\eqref{eq_consist_thre}. The selection set $\mathcal{R}$ is then self-consistent with threshold $m/(\alpha|\mathcal{R}|)$, in the sense that
\begin{align}\label{eq_self_consist}
    j \in \mathcal{R} \iff E_j\geq \frac{m}{\alpha |\mathcal{R}|}.
\end{align}
Expanding the FDP \eqref{eq_FDP_def} using \eqref{eq_self_consist}, and noting that $E_j\leq E^*_j$ on the event $\{Y_{n+j}\leq c_j\}$ by \eqref{eq_E_E_star} and that $\mathds{1}\{Y_{n+j}\leq c_j\}\leq 1$, we obtain
\begin{align}\label{eq_FDR_expand}
    \textrm{FDP}(\mathcal{R}, \mathcal{Y}^{\textrm{test}})
    &= \frac{\sum_{j=1}^m\mathds{1}\{Y_{n+j}\leq c_j\}\mathds{1}\{j\in\mathcal{R}\}}{\max\{1, |\mathcal{R}|\}}\nonumber\\
    &\leq \frac{1}{|\mathcal{R}|}\sum_{j=1}^m \mathds{1}\{Y_{n+j}\leq c_j\} \mathds{1} \bigg\{E^*_j \geq \frac{m}{\alpha |\mathcal{R}|}\bigg\}\nonumber\\
    &\leq \frac{1}{|\mathcal{R}|}\sum_{j=1}^m \mathds{1} \bigg\{E^*_j \geq \frac{m}{\alpha |\mathcal{R}|}\bigg\}.
\end{align}
Applying the elementary bound $\mathds{1}\{E\geq t\}\leq E/t$ for $E\geq 0$ and $t>0$ with $E = E^*_j$ and $t = m/(\alpha|\mathcal{R}|)$, and noting that $|\mathcal{R}|$ cancels, gives
\begin{align}
    \textrm{FDP}(\mathcal{R},\mathcal{Y}^{\textrm{test}})
    \leq \frac{1}{|\mathcal{R}|}\sum_{j=1}^m \frac{\alpha |\mathcal{R}|}{m} E^*_j
    = \frac{\alpha}{m}\sum_{j=1}^m E^*_j.
\end{align}
Dividing both sides by $\alpha$ yields the level-uniform bound
\begin{align}\label{eq_level_uniform}
    \frac{\textrm{FDP}(\mathcal{R}, \mathcal{Y}^{\textrm{test}})}{\alpha} \leq \frac{1}{m} \sum_{j=1}^m E_j^*
\end{align}
for all $\alpha\in(0,1)$. Taking expectations and using $\mathbb{E}[E^*_j]=1$ from \eqref{eq_exp_E_star} gives
\begin{align}
    \mathbb{E}\Bigg[\frac{\textrm{FDP}(\mathcal{R}, \mathcal{Y}^{\textrm{test}})}{\alpha}\Bigg] \leq 1.
\end{align}
Since the right-hand side of \eqref{eq_level_uniform} depends on neither $\alpha$ nor the thresholds $\{c_j\}_{j=1}^m$, the bound holds when $\alpha$ is replaced by any data-dependent choice measurable with respect to $(\mathcal{D}^{\textrm{cal}}, \mathcal{X}^{\textrm{test}})$, and remains valid even if the thresholds $\{c_j\}_{j=1}^m$ are themselves functions of $(\mathcal{D}^{\textrm{cal}}, \mathcal{X}^{\textrm{test}})$. In particular, taking $\alpha = \alpha^{\textrm{PH-CS}}$ and $\mathcal{R} = \mathcal{R}(\alpha^{\textrm{PH-CS}}) = \mathcal{R}^{\textrm{PH-CS}}$ yields \eqref{eq_post_FDP}.

The key step enabling post-hoc validity is the inequality $\mathds{1}\{E\geq t\}\leq E/t$, which yields a bound free of $\alpha$ (see \eqref{eq_level_uniform}). In contrast, conventional CS relies on the tail calibration $\mathbb{P}(P\leq t)\leq t$, whose analogue would require $\mathbb{E}[1/P]<\infty$, which fails even for $P\sim\textrm{U}(0,1)$. This is why CS guarantees are tied to a pre-specified level and cannot extend to post-hoc selection.

{\color{blue}
\subsection{Proof of Proposition~\ref{prop_asymp}} \label{apx_prop_asymp}
By the definition of the FDP \eqref{eq_FDP_def}, the claim \eqref{eq_prop_asymp} requires the realized FDP of the selected set $\mathcal{R}^{\textrm{PH-CS}}$ to stay below the reported level $\alpha^{\textrm{PH-CS}}$ in large samples. We obtain this from the level-uniform bound \eqref{eq_level_uniform} in Appendix~\ref{apx_prof_posthoc}, which holds for every level $\alpha\in(0,1)$ and for the data-dependent level $\alpha^{\textrm{PH-CS}}$. Evaluated at $\alpha=\alpha^{\textrm{PH-CS}}$ and $\mathcal{R}=\mathcal{R}^{\textrm{PH-CS}}$, it reads
\begin{align}\label{eq_fdp_le_avg}
    \textrm{FDP}\big(\mathcal{R}^{\textrm{PH-CS}}, \mathcal{Y}^{\textrm{test}}\big) \leq \alpha^{\textrm{PH-CS}}\cdot\frac{1}{m}\sum_{j=1}^m E_j^*,
\end{align}
where $E_j^*$ is the oracle e-variable \eqref{eq_oracle_e}. Since $\alpha^{\textrm{PH-CS}}\leq 1$ by \eqref{eq_alpha_k_def}, this yields the deterministic inequality
\begin{align}\label{eq_fdp_to_avg}
    \textrm{FDP}\big(\mathcal{R}^{\textrm{PH-CS}}, \mathcal{Y}^{\textrm{test}}\big)
    \leq \alpha^{\textrm{PH-CS}} + \bigg|\frac{1}{m}\sum_{j=1}^m E_j^* - 1\bigg|.
\end{align}
By \eqref{eq_fdp_to_avg}, the event $\{|\sum_{j=1}^m E_j^*/m - 1|\leq\varepsilon\}$ is contained in $\{\textrm{FDP}(\mathcal{R}^{\textrm{PH-CS}}, \mathcal{Y}^{\textrm{test}})\leq\alpha^{\textrm{PH-CS}}+\varepsilon\}$, so for every $\varepsilon>0$,
\begin{align}\label{eq_prob_reduction}
    &\mathbb{P}\Big(\textrm{FDP}\big(\mathcal{R}^{\textrm{PH-CS}}, \mathcal{Y}^{\textrm{test}}\big) \leq \alpha^{\textrm{PH-CS}} + \varepsilon\Big)\nonumber\\
    \geq &\mathbb{P}\bigg(\bigg|\frac{1}{m}\sum_{j=1}^m E_j^* - 1\bigg| \leq \varepsilon\bigg).
\end{align}
Since this probability is also at most one, it therefore suffices to prove that, for every $\varepsilon>0$,
\begin{align}\label{eq_avg_E_star_conv}
    \lim_{n,m \rightarrow \infty}\mathbb{P}\bigg(\bigg|\frac{1}{m}\sum_{j=1}^m E_j^* - 1\bigg| \leq \varepsilon\bigg) = 1.
\end{align}

By definition, the score is non-negative, so dropping the term $S_{n+j}$ from the denominator of the oracle e-variable \eqref{eq_oracle_e} only enlarges it. This yields, for each $j$,
\begin{align}\label{eq_E_star_sandwich}
    0 \leq \frac{S_{n+j}}{\frac{1}{n+1}\sum_{i=1}^n S_i} - E_j^*
    &= \frac{(n+1)S_{n+j}^2}{\big(\sum_{i=1}^n S_i\big)\big(\sum_{i=1}^n S_i + S_{n+j}\big)} \notag\\
    &\leq \frac{(n+1)S_{n+j}^2}{\big(\sum_{i=1}^n S_i\big)^2},
\end{align}
where the last step again drops $S_{n+j}$ from a denominator. Averaging \eqref{eq_E_star_sandwich} over $j=1,\ldots,m$, the gap
\begin{align}\label{eq_R_def}
    R_{n,m} = \frac{1}{m}\cdot\frac{\sum_{j=1}^m S_{n+j}}{\frac{1}{n+1}\sum_{i=1}^n S_i} - \frac{1}{m}\sum_{j=1}^m E_j^*
\end{align}
satisfies
\begin{align}\label{eq_R_bound}
    0\leq R_{n,m} \leq \frac{n+1}{\big(\sum_{i=1}^n S_i\big)^2}\cdot\frac{1}{m}\sum_{j=1}^m S_{n+j}^2.
\end{align}

By assumption \textit{(i)} in Proposition~\ref{prop_asymp}, the scores are i.i.d, so for every $\varepsilon>0$ the law of large numbers gives
\begin{subequations}\label{eq_lln}
    \begin{align}
        \lim_{n \rightarrow \infty}\mathbb{P}\bigg(\bigg|\frac{1}{n}\sum_{i=1}^n S_i - \mathbb{E}[S(X,Y)]\bigg| > \frac{\varepsilon}{2}\bigg) &= 0,\\
        \lim_{m \rightarrow \infty}\mathbb{P}\bigg(\bigg|\frac{1}{m}\sum_{j=1}^m S_{n+j} - \mathbb{E}[S(X,Y)]\bigg| > \frac{\varepsilon}{2}\bigg) &= 0,
    \end{align}
\end{subequations}
with $\mathbb{E}[S(X,Y)]>0$ by assumption \textit{(iii)}. Since the first term in \eqref{eq_R_def} is the ratio of these two averages, \eqref{eq_lln} gives
\begin{align}\label{eq_ratio_lim}
    \lim_{n,m \rightarrow \infty}\mathbb{P}\bigg(\bigg|\frac{1}{m}\cdot\frac{\sum_{j=1}^m S_{n+j}}{\frac{1}{n+1}\sum_{i=1}^n S_i} - 1\bigg| > \frac{\varepsilon}{2}\bigg) = 0.
\end{align}
For the gap, the finite second moment of assumption \textit{(iii)} keeps $\sum_{j=1}^m S_{n+j}^2/m$ bounded, while \eqref{eq_lln} makes $\sum_{i=1}^n S_i$ grow as $n\,\mathbb{E}[S(X,Y)]$, so the upper bound in \eqref{eq_R_bound} gives
\begin{align}\label{eq_R_lim}
    \lim_{n,m \rightarrow \infty}\mathbb{P}\bigg(R_{n,m} > \frac{\varepsilon}{2}\bigg) = 0.
\end{align}
By \eqref{eq_R_def} and the triangle inequality, $|\sum_{j=1}^m E_j^*/m - 1|$ is at most the sum of the two deviations bounded in \eqref{eq_ratio_lim} and \eqref{eq_R_lim}. These two limits therefore establish \eqref{eq_avg_E_star_conv}, which with \eqref{eq_prob_reduction} proves \eqref{eq_prop_asymp}.
}

\begin{figure*}[t]
    \centering
    {\hspace{-5mm}
    \includegraphics[width = 0.3\textwidth]{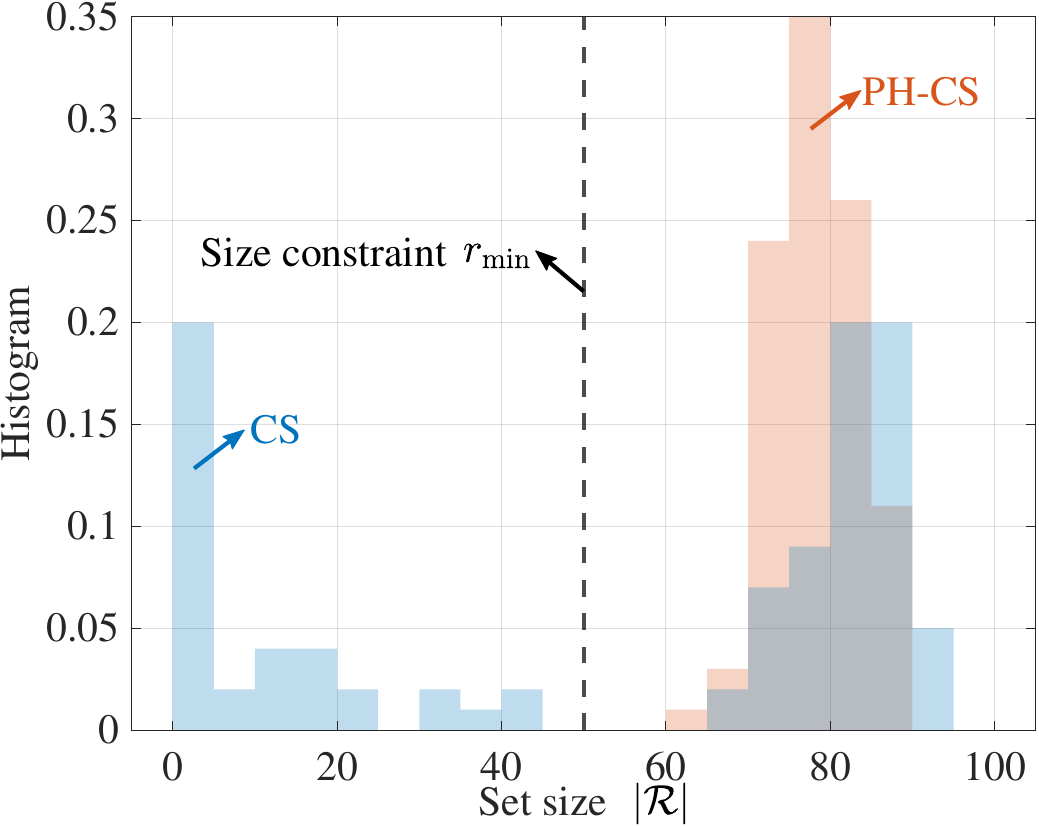}
    \includegraphics[width = 0.305\textwidth]{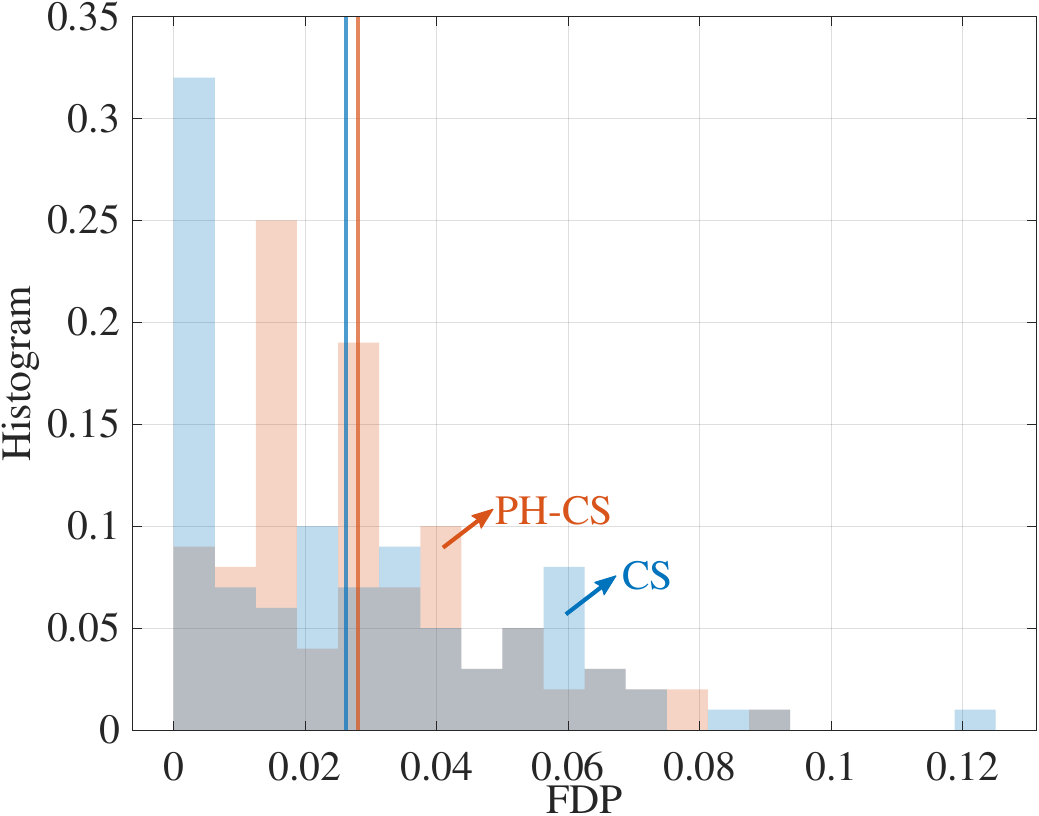}
    \includegraphics[width = 0.3\textwidth]{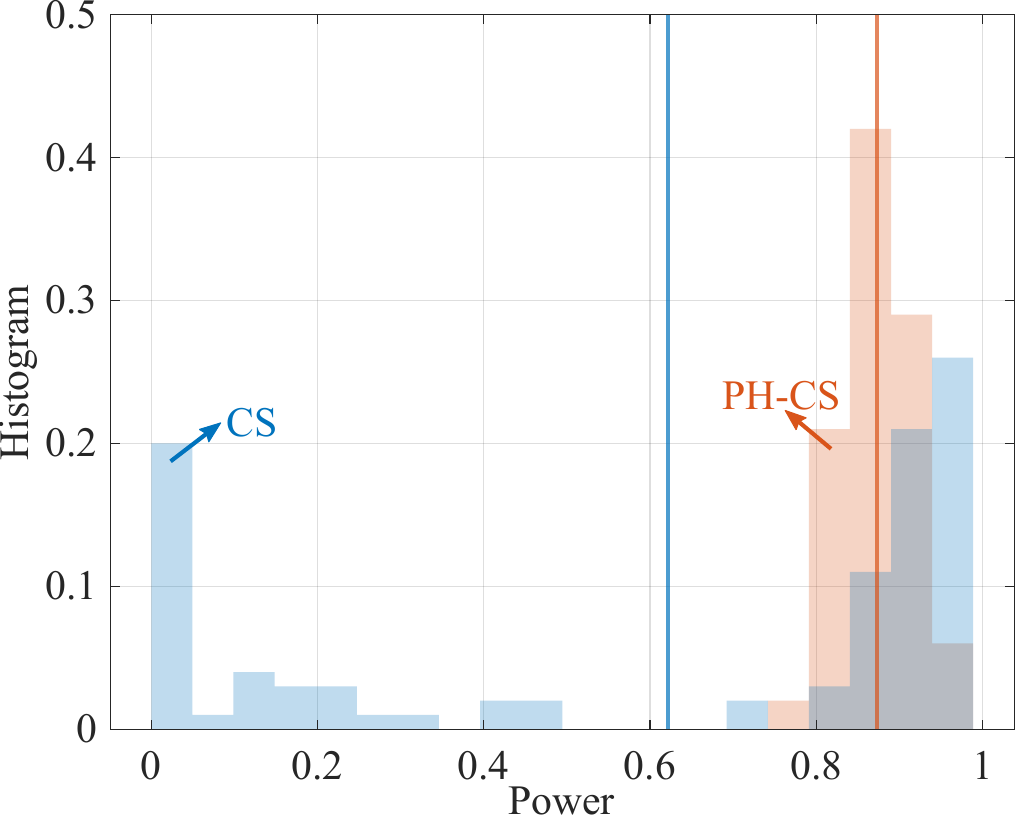}\\
    \hspace{-5.5mm}
    \includegraphics[width = 0.302\textwidth]{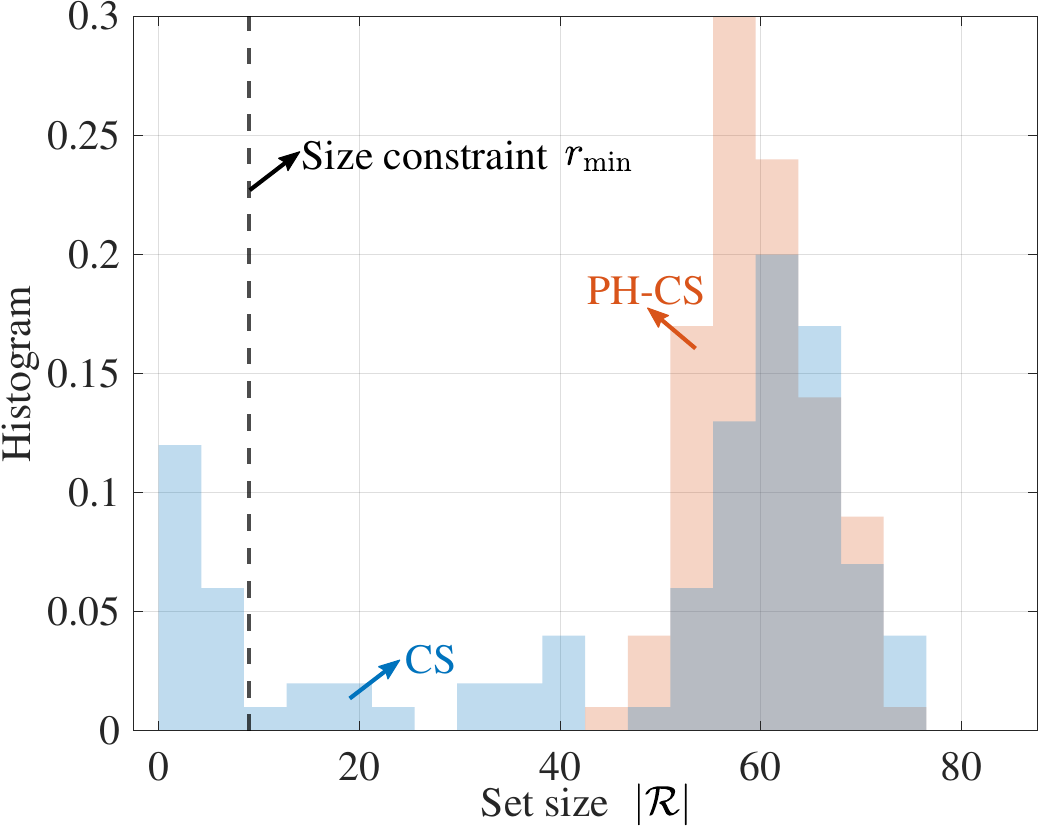}
    \includegraphics[width = 0.304\textwidth]{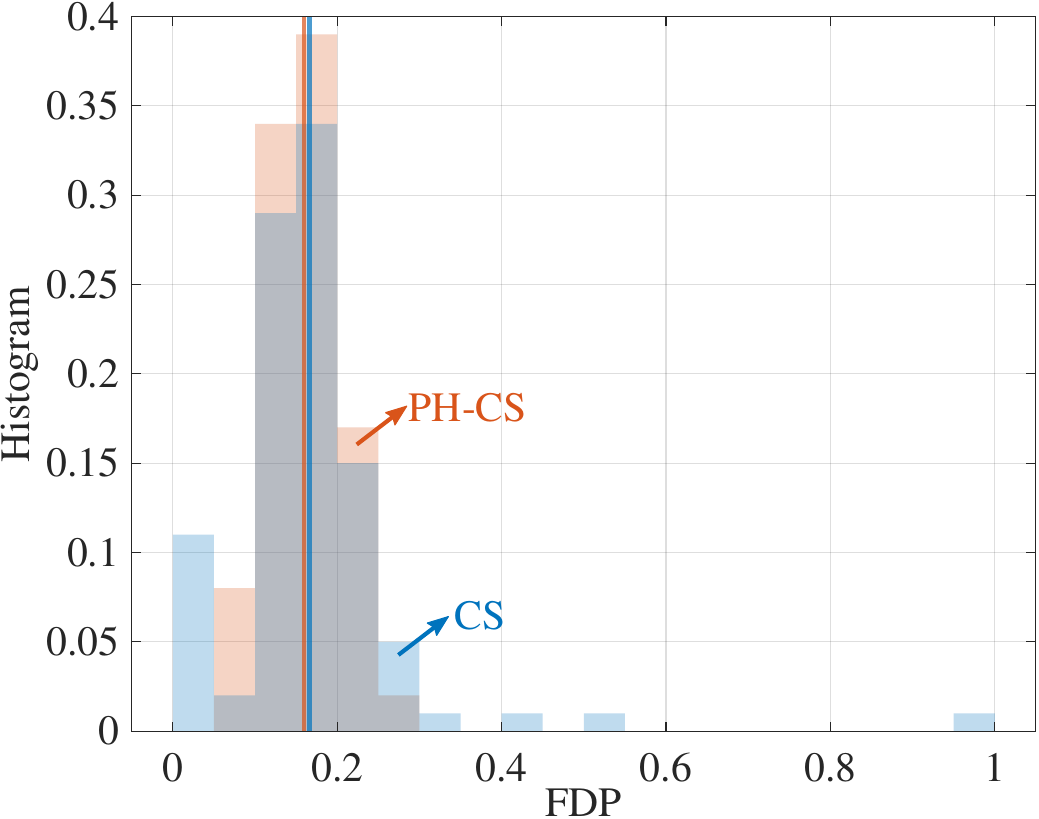}
    \includegraphics[width = 0.303\textwidth]{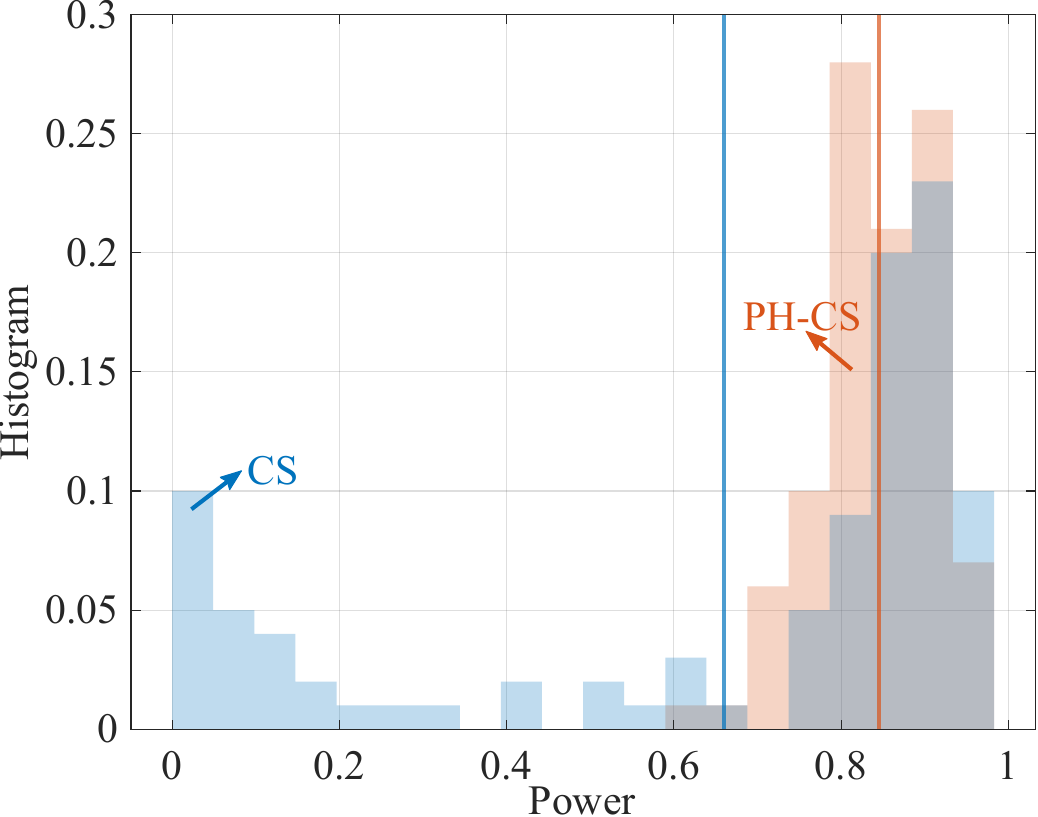}\\
    
    \includegraphics[width = 0.3\textwidth]{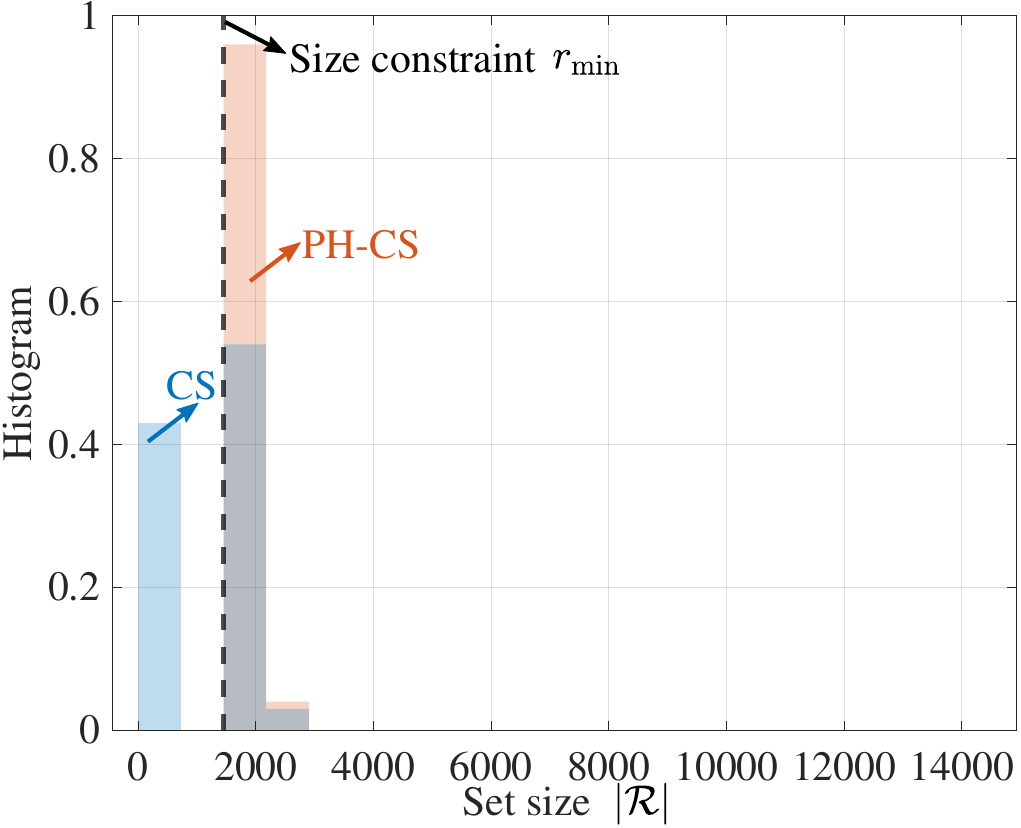}
    \includegraphics[width = 0.3\textwidth]{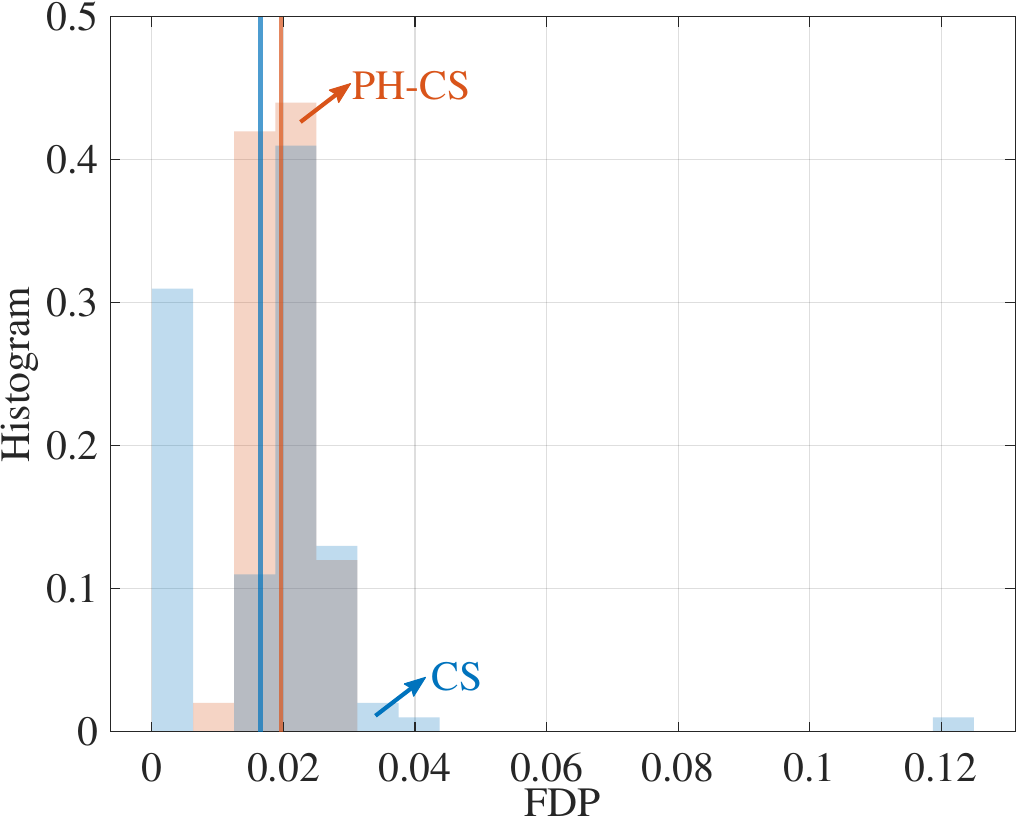}
    \includegraphics[width = 0.3\textwidth]{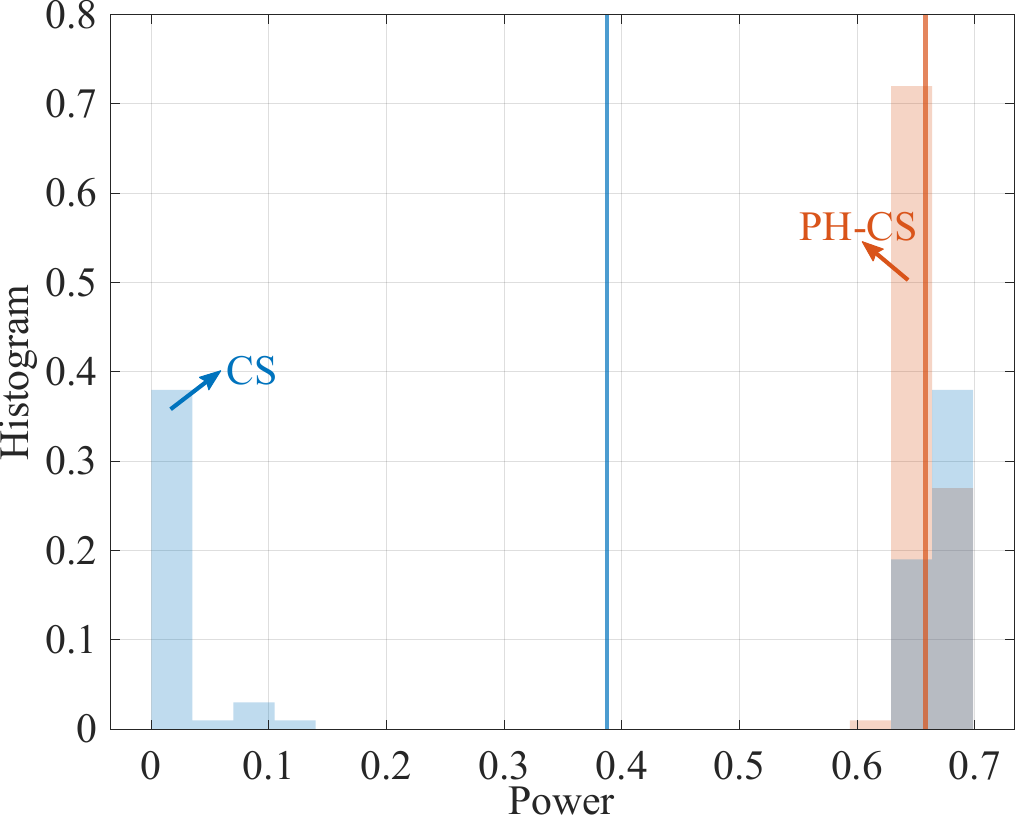}
    }
    \vspace{-1mm}
    \caption{{\color{blue}Histograms of the selected set size (left), realized FDP (middle), and power (right) under the constrained-size utility \eqref{eq_utility_size_first} on the synthetic data with heteroscedastic noise using gradient boosting (top row), the Recruitment dataset (middle row), and the Shuttle dataset (bottom row). In the middle and right panels, the blue and red vertical lines indicate the average values for CS and PH-CS, respectively.}}
    \vspace{-3.5mm}
    \label{fig_Csize_apdx}
\end{figure*}
\begin{figure*}[t]
    \centering
    {\hspace{-4mm}
	\includegraphics[width = 0.3\textwidth]{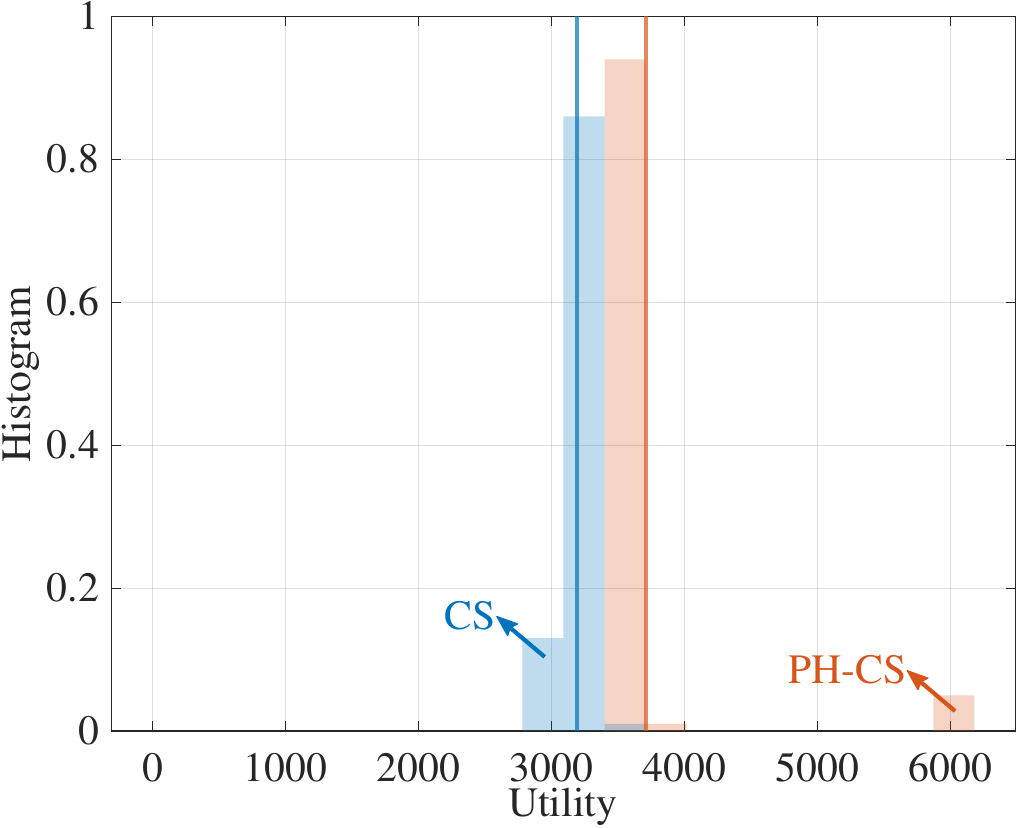}
    \includegraphics[width = 0.305\textwidth]{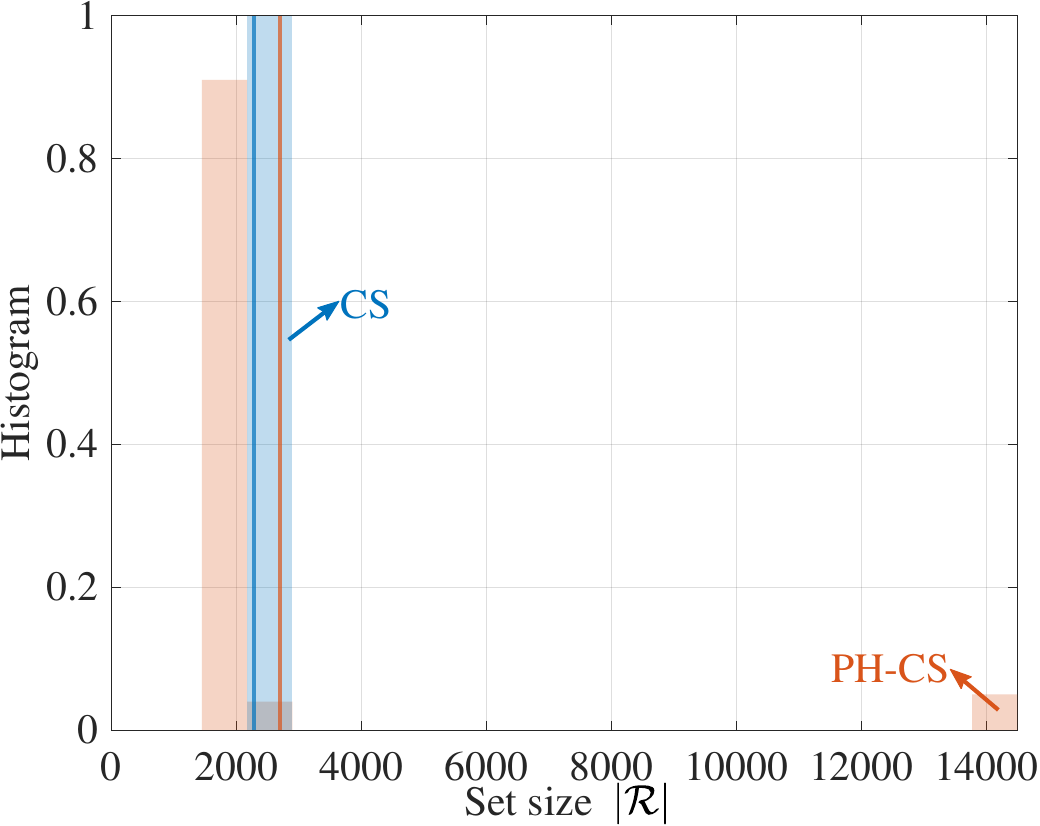}
    \includegraphics[width = 0.3\textwidth]{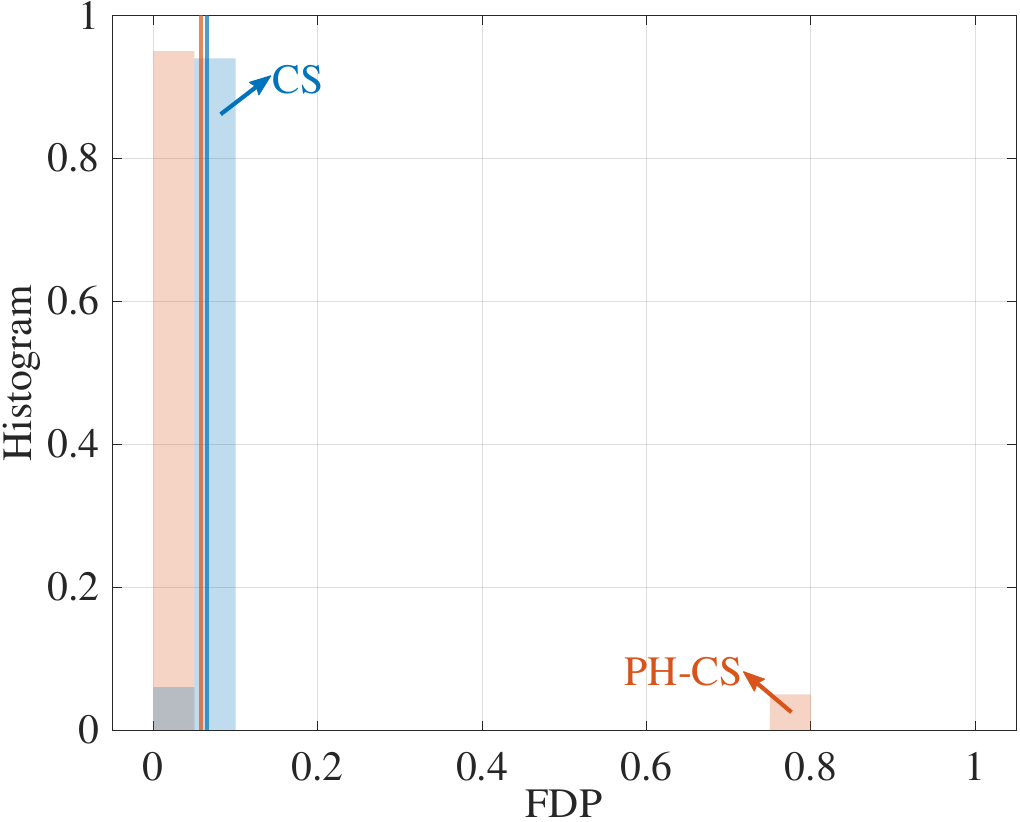}
	\includegraphics[width = 0.3\textwidth]{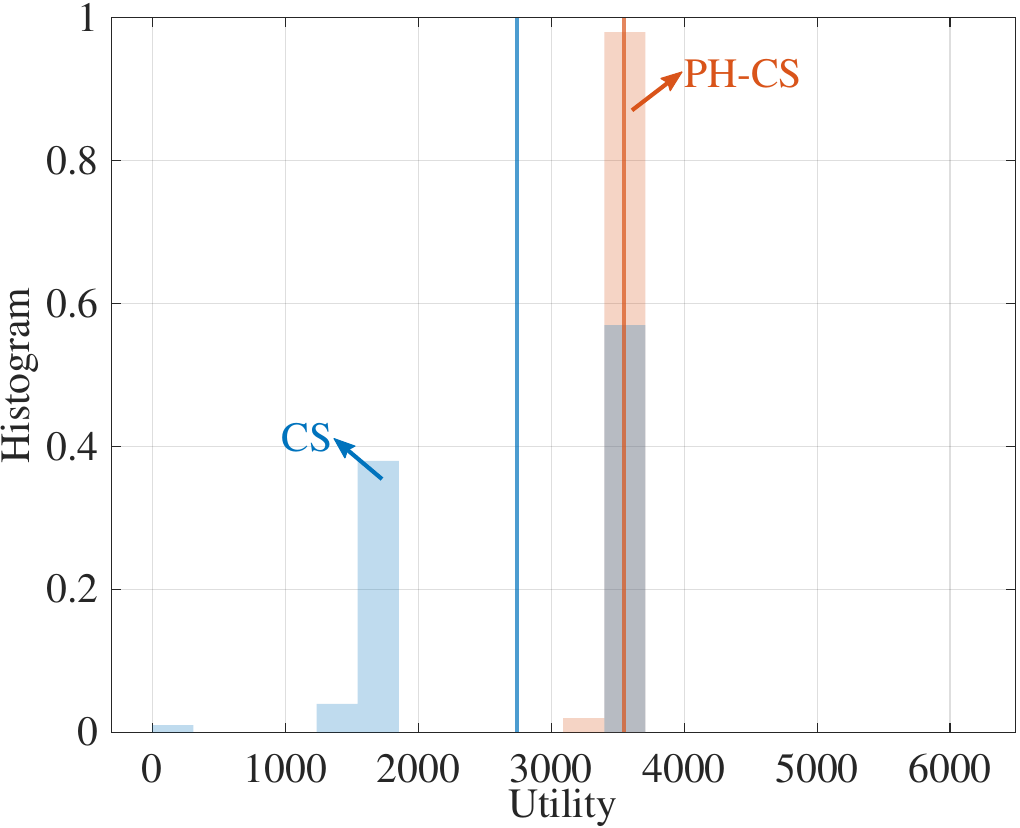}
    \includegraphics[width = 0.304\textwidth]{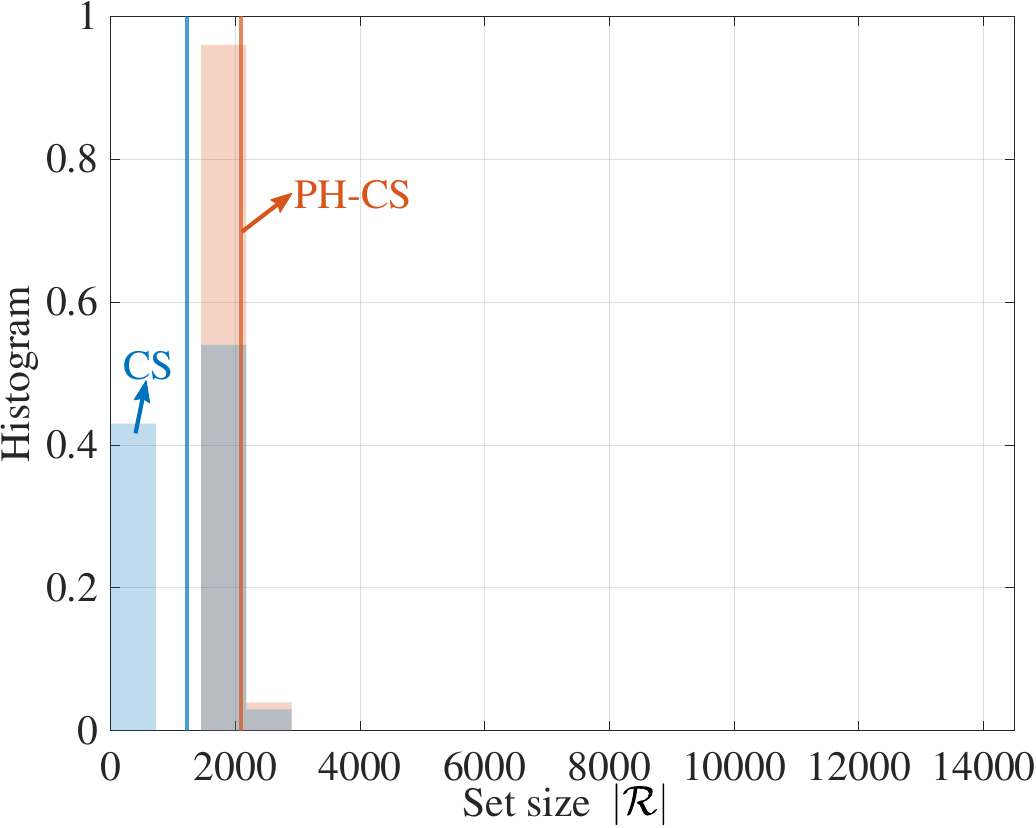}
    \includegraphics[width = 0.3\textwidth]{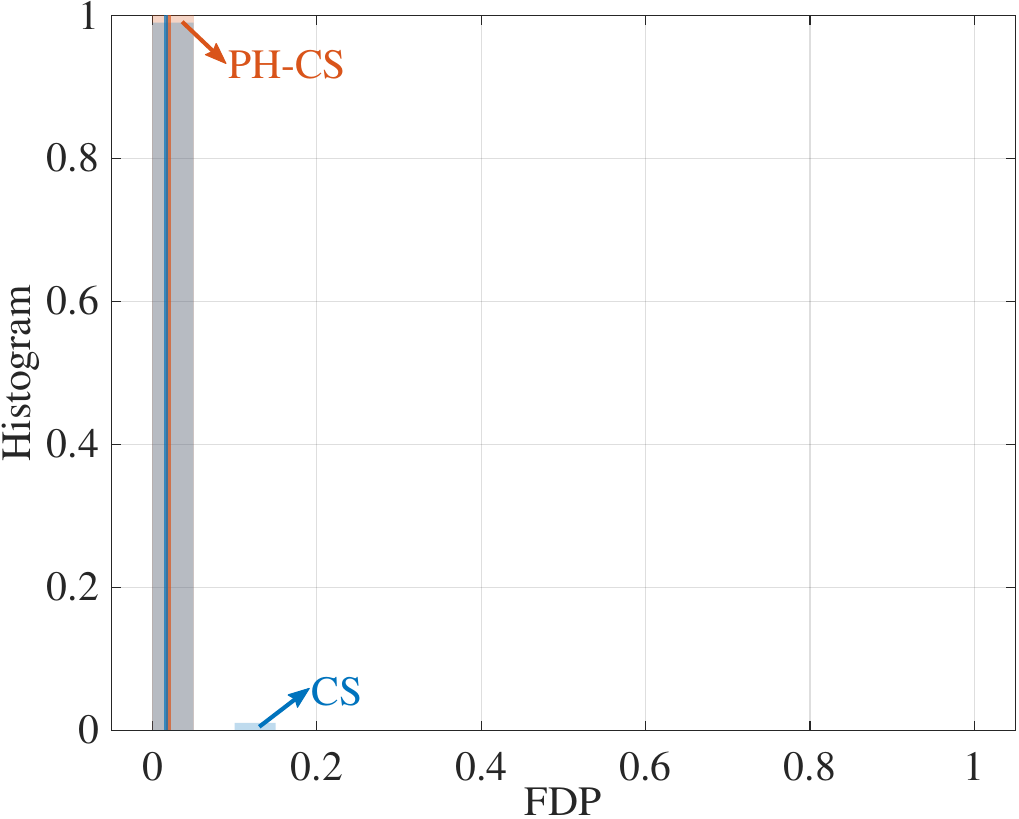}
    }
    \vspace{-1.2mm}
    \caption{Histograms of the realized utility (left), selected set size (middle), and FDP (right) under the additive trade-off utility \eqref{eq_u_add} on the Shuttle dataset, with $\lambda = 12800$ (top row) and $\lambda = 14000$ (bottom row). In each panel, the blue and red vertical lines indicate the corresponding average values of CS and PH-CS, respectively.}
    \vspace{-5mm}
    \label{fig_real_utility}
\end{figure*}
\begin{figure*}[t]
    \centering
    {
	\includegraphics[width = 0.3\textwidth]{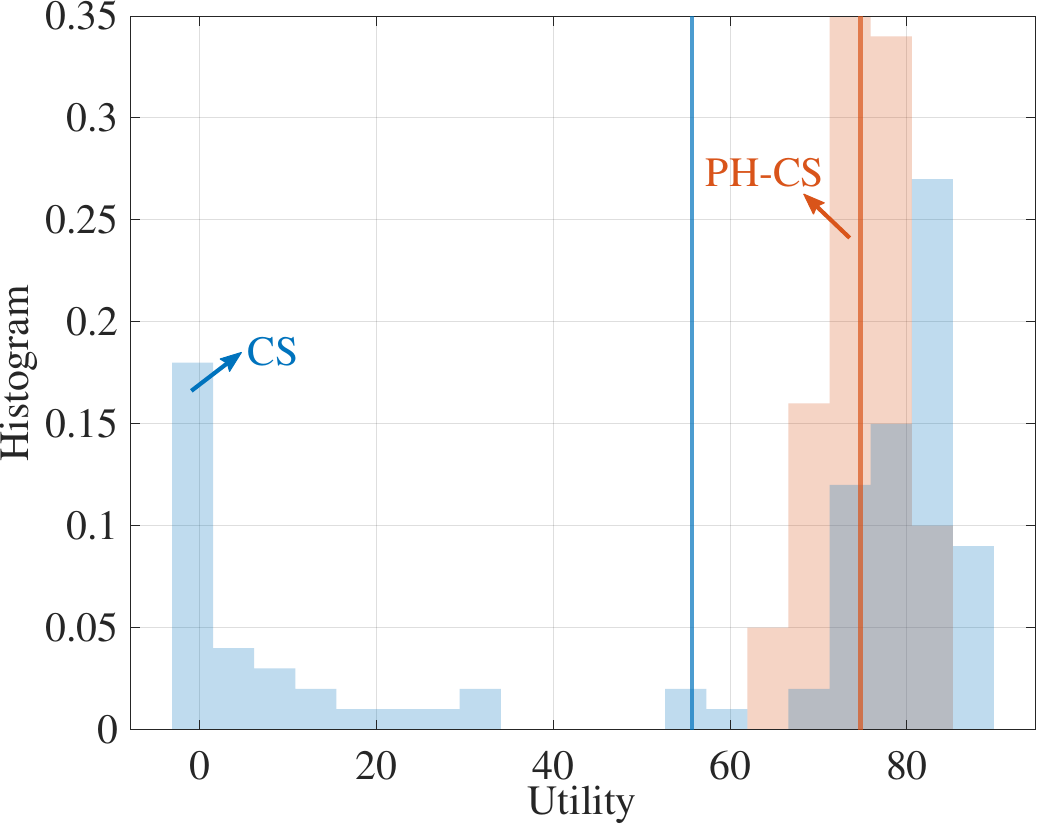}
    \includegraphics[width = 0.3\textwidth]{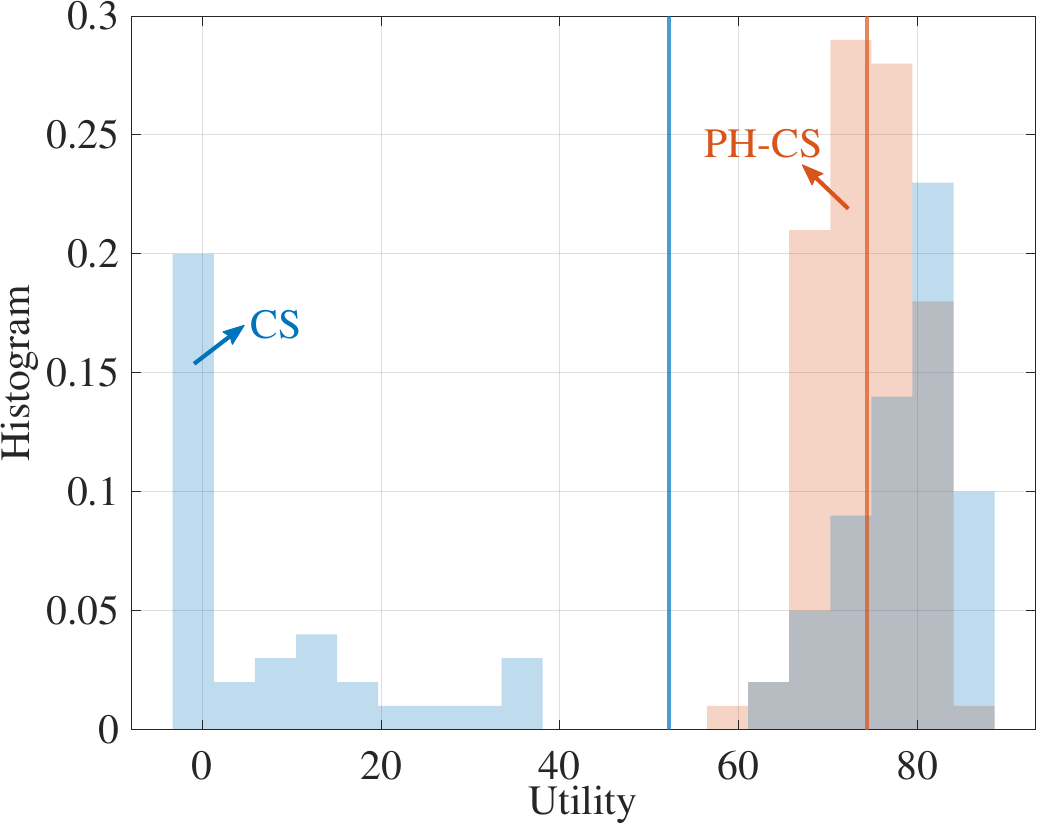}\\
    \includegraphics[width = 0.3\textwidth]{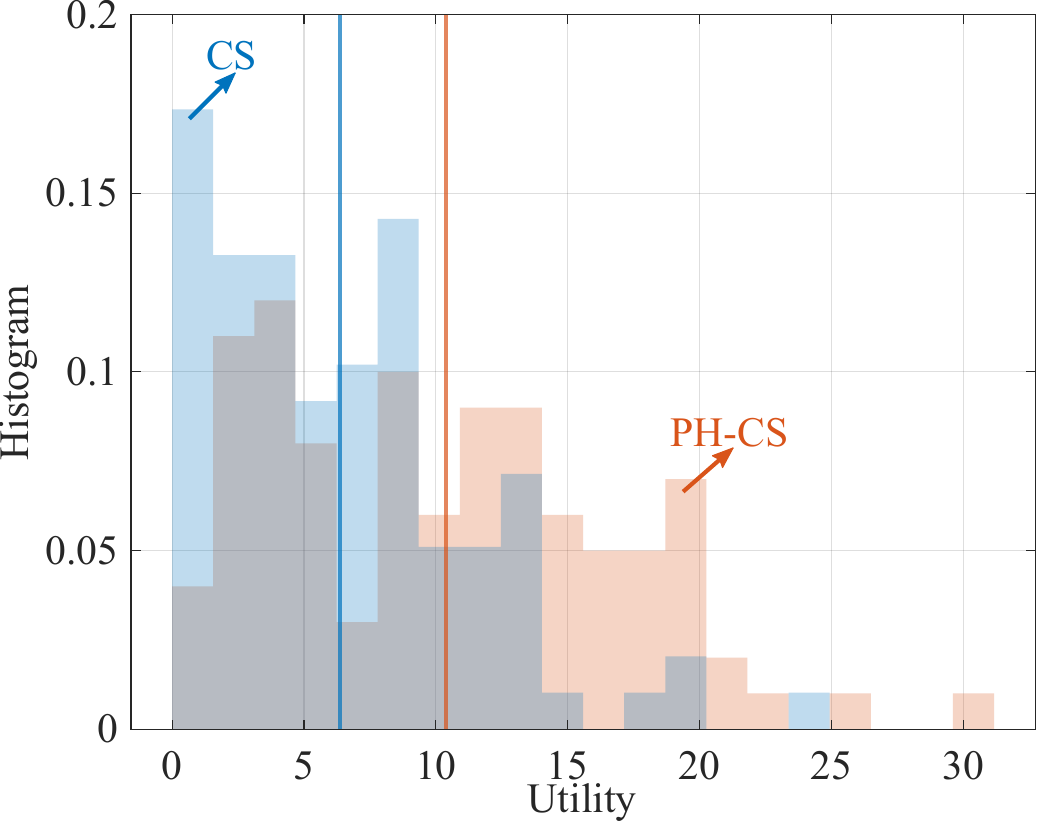}
    \includegraphics[width = 0.3\textwidth]{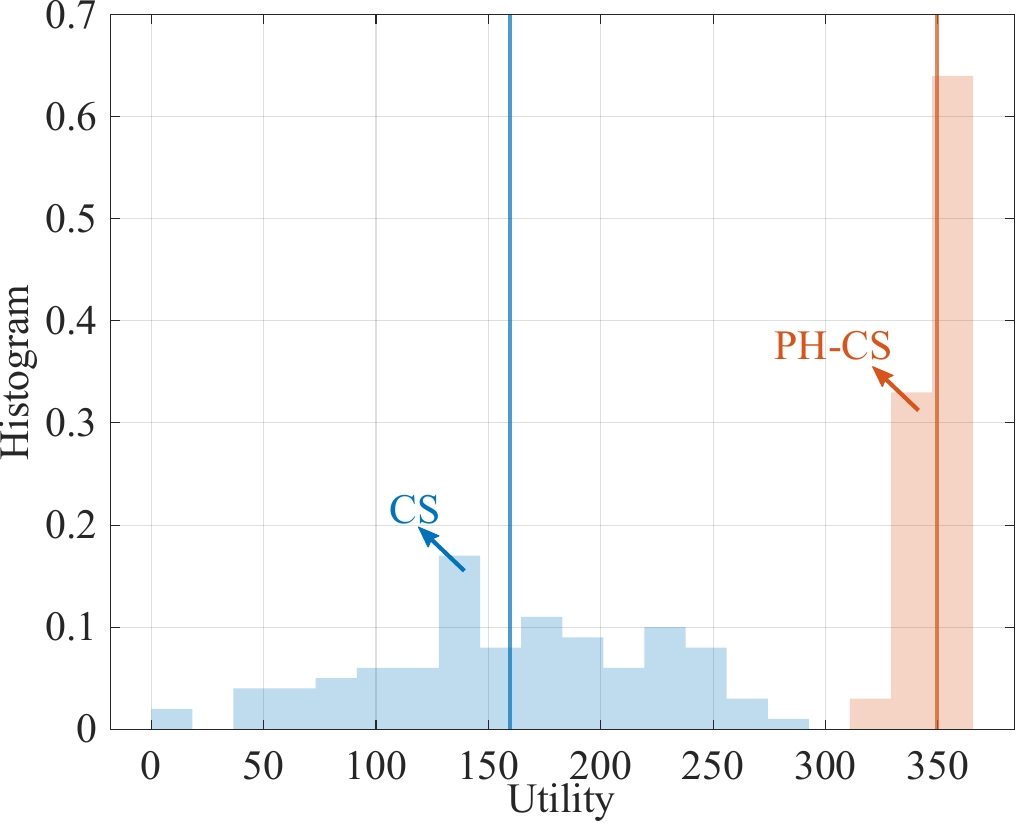}
    }
    \vspace{-1mm}
    \caption{{\color{blue}Histograms of the realized utility under the additive trade-off utility \eqref{eq_u_add} on the synthetic data with homoscedastic (top left) and heteroscedastic (top right) noise using gradient boosting, and on the Recruitment (bottom left) and Musk (bottom right) datasets. The blue and red vertical lines indicate the average utility of CS and PH-CS, respectively.}}
    \vspace{-3.5mm}
    \label{fig_add_apdx}
\end{figure*}
\begin{figure*}[t]
    \centering
    {
    \hspace{-4mm}\includegraphics[width = 0.3\textwidth]{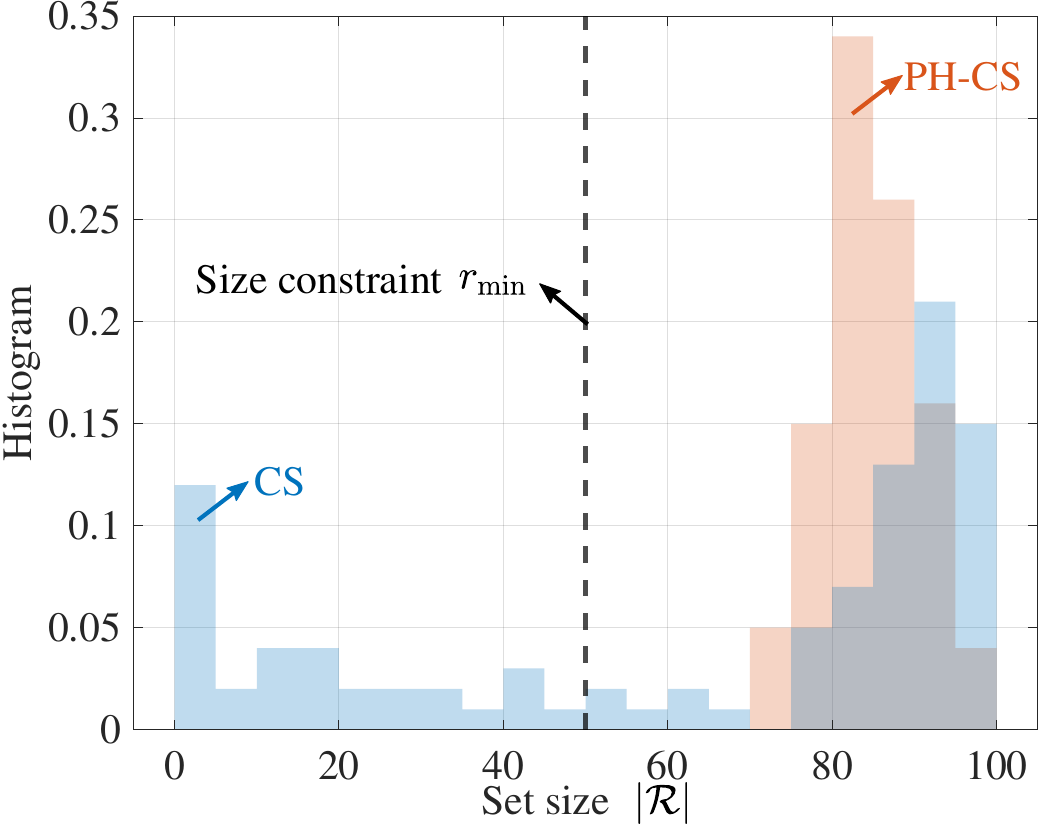}
    \includegraphics[width = 0.3\textwidth]{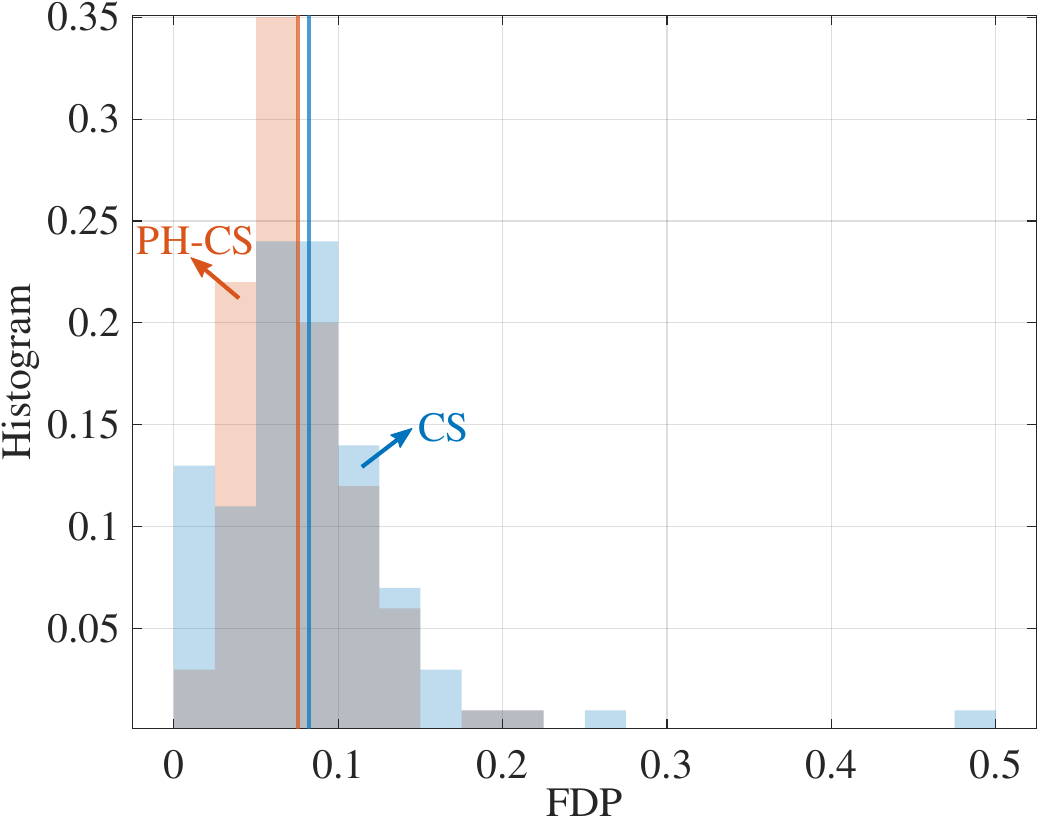}
    \includegraphics[width = 0.3\textwidth]{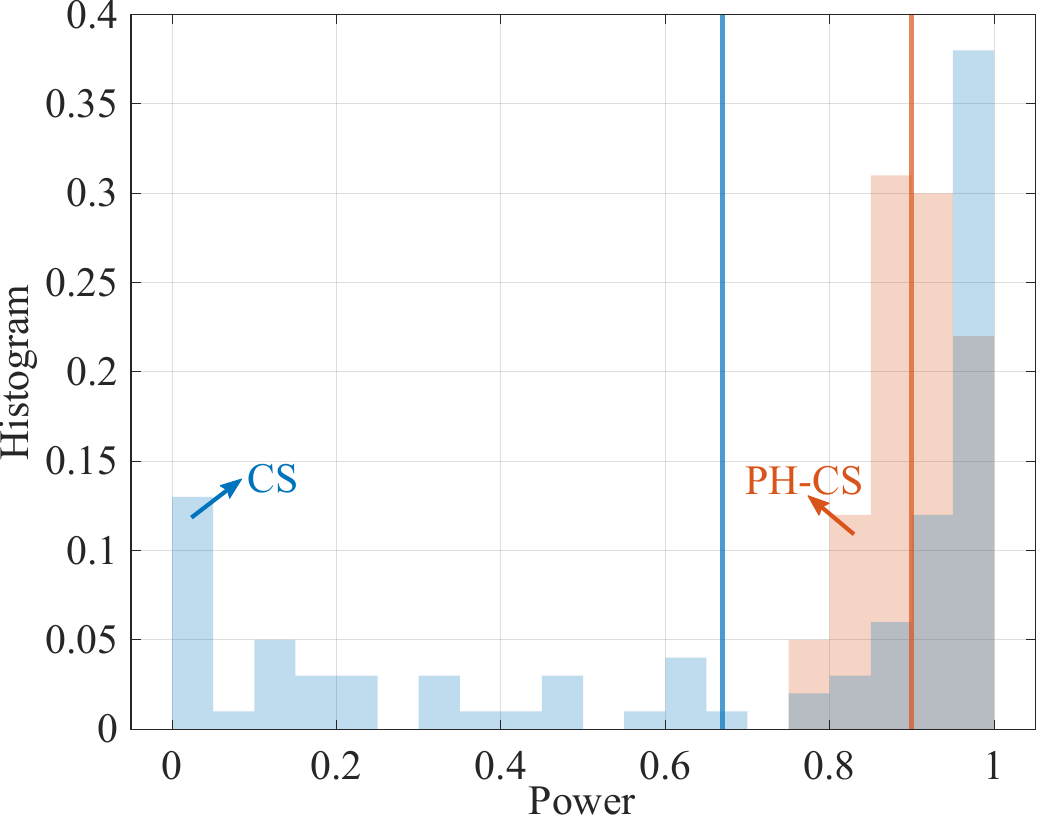}
    \includegraphics[width = 0.3\textwidth]{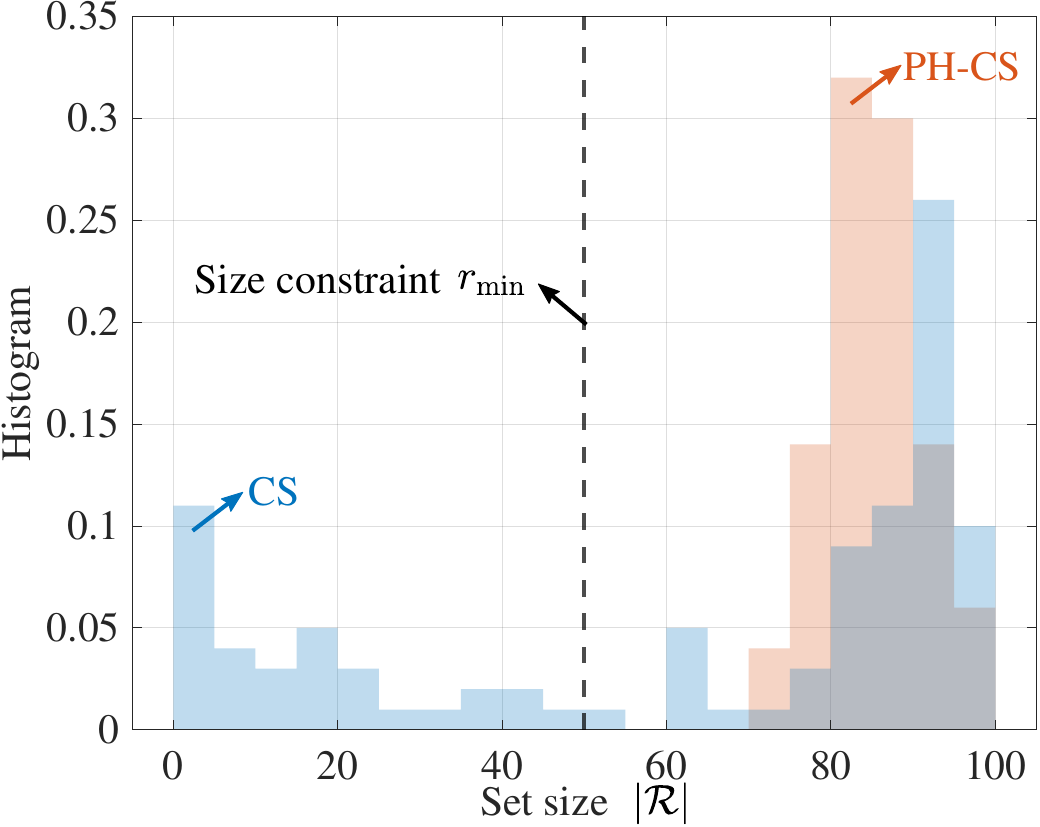}
    \includegraphics[width = 0.308\textwidth]{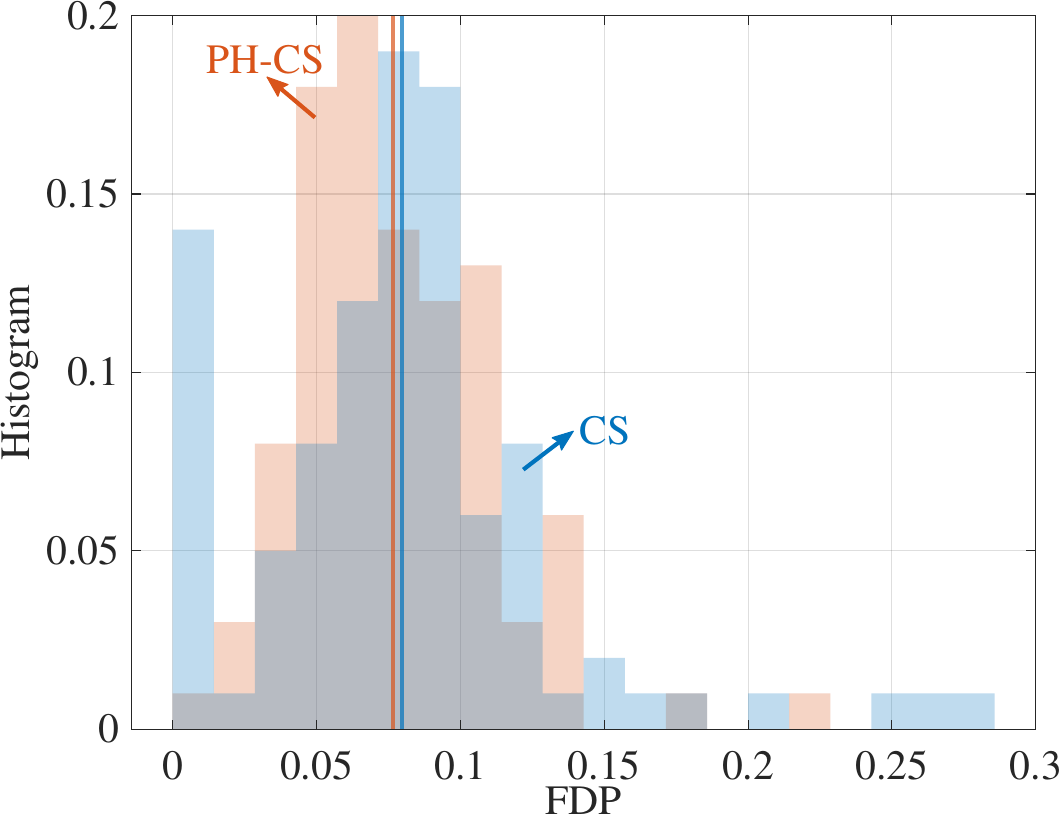}
    \includegraphics[width = 0.3\textwidth]
    {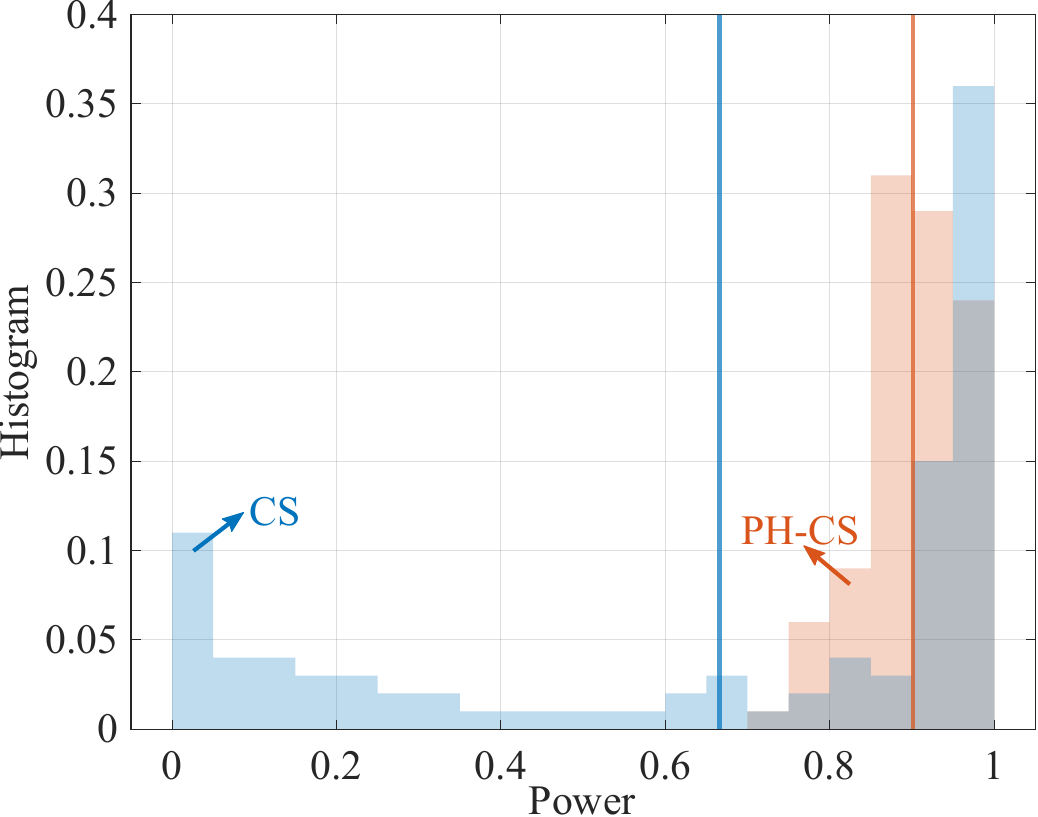}
    }
    \vspace{-1mm}
    \caption{{\color{blue}Histograms of the selected set size (left), realized FDP (middle), and power (right) under the constrained-size utility \eqref{eq_utility_size_first} on synthetic data with homoscedastic (top row) and heteroscedastic (bottom row) noise using random forest. In the middle and right panels, the blue and red vertical lines indicate the average values for CS and PH-CS, respectively.}}
    \label{fig_syn_Csize_rf_apdx}
    \vspace{-3.5mm}
\end{figure*}
\begin{figure*}[t]
    \centering
    {
	\includegraphics[width = 0.3\textwidth]{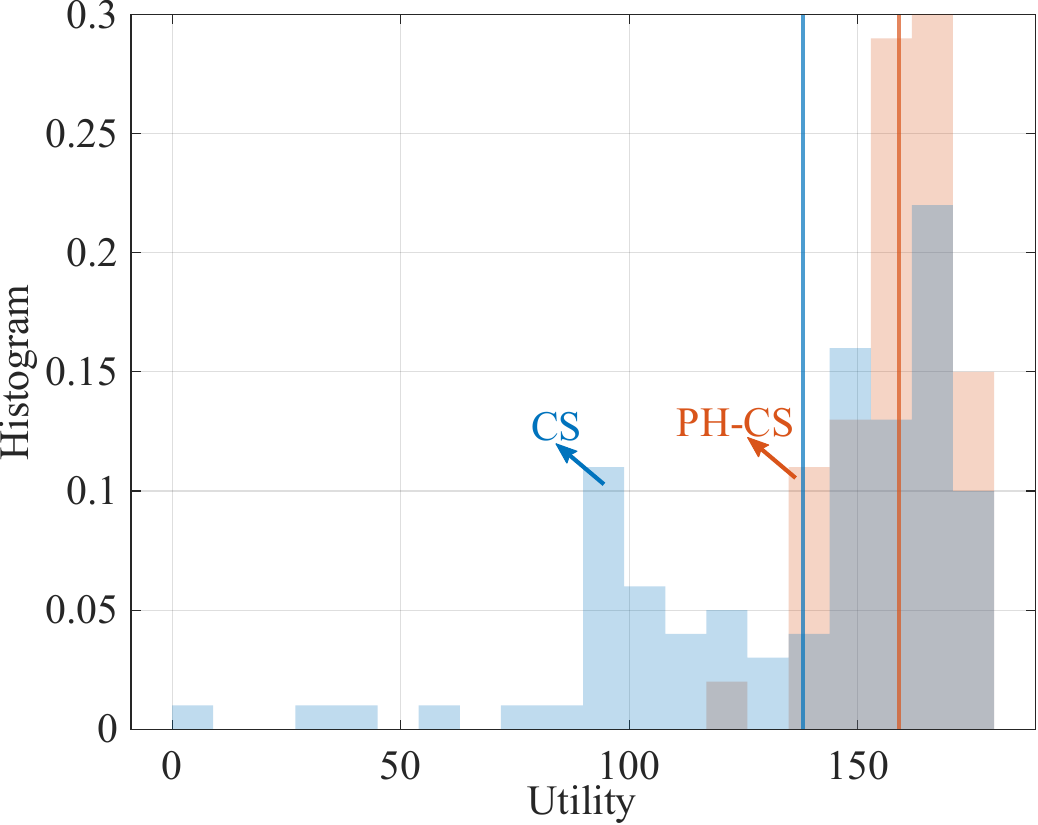}
    \includegraphics[width = 0.295\textwidth]{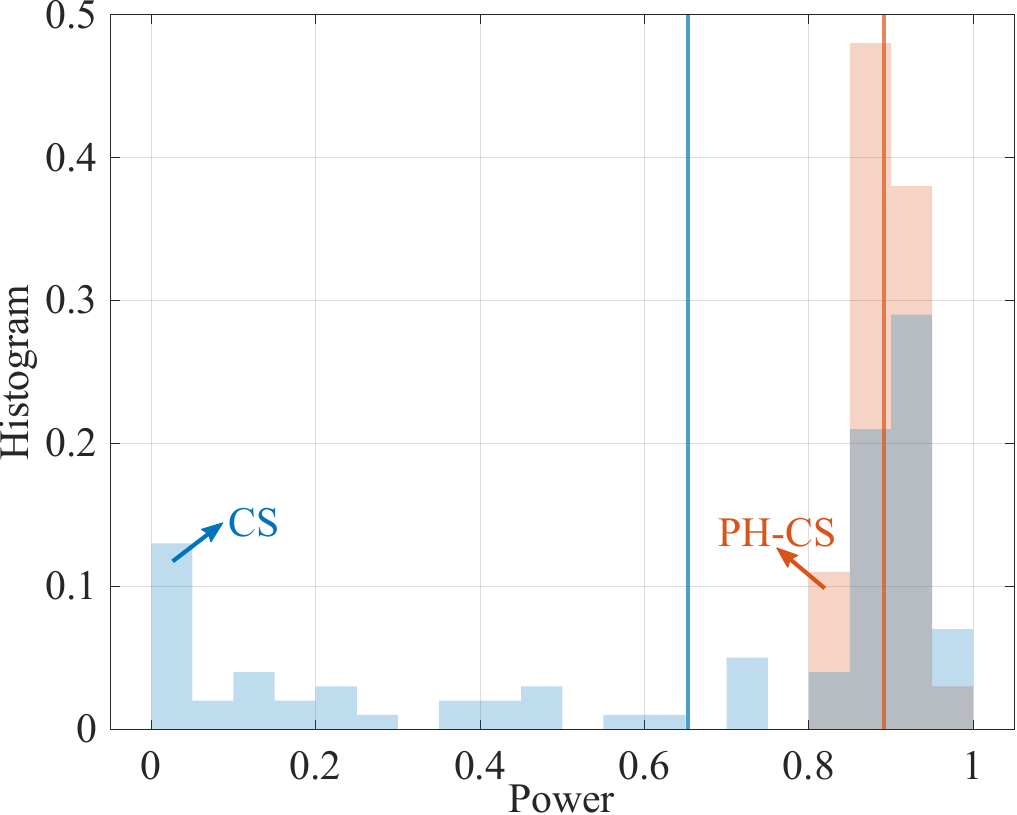}
    \includegraphics[width = 0.3\textwidth]{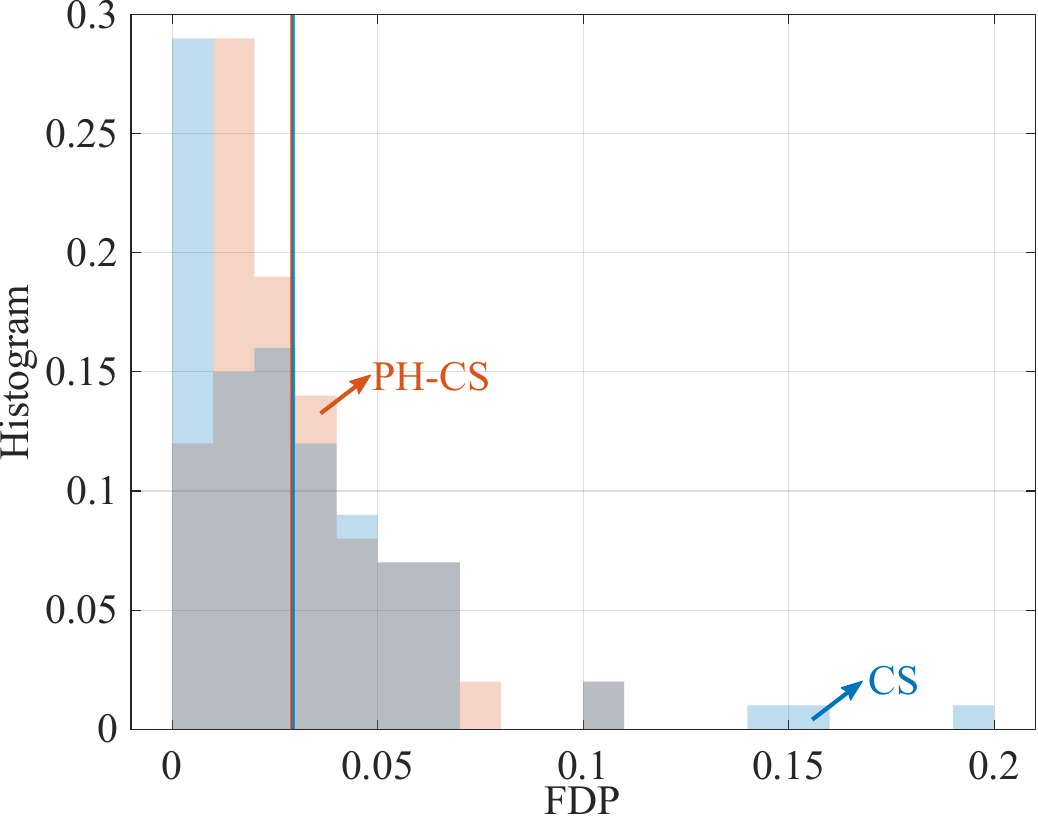}
    }
    \vspace{-1mm}
    \caption{{\color{blue}Histograms of the realized utility (left), power (middle), and FDP (right) under the additive utility \eqref{eq_u_add} with weighted set size \eqref{eq_w_size} on the synthetic data with homoscedastic noise using support vector regression. In each panel, the blue and red vertical lines indicate the average values for CS and PH-CS, respectively.}}
    \vspace{-3.5mm}
    \label{fig_syn_Paware_svm_apdx}
\end{figure*}

\subsection{Proof of Theorem \ref{thm_ph_rcs}}\label{apdx_proof_rcs}
The proof follows the same strategy as Step 3 of the proof of Theorem \ref{theo_posthoc} in Appendix \ref{apx_prof_posthoc}, but replaces the oracle argument in Steps 1--2 with the risk-adjusted e-variable condition \eqref{eq_risk_adj_eval}.

Fix $\alpha \in (0,1)$ and let $\mathcal{R} = \mathcal{R}(\alpha)$ denote the e-BH output with $E_j$ replaced by $E_j^{\textrm{g}}$. As in \eqref{eq_self_consist}, the self-consistency of e-BH gives
\begin{align}\label{eq_self_consist_rcs}
    j \in \mathcal{R} \Longleftrightarrow E_j^{\textrm{g}} \geq \frac{m}{\alpha |\mathcal{R}|}.
\end{align}
Applying \eqref{eq_self_consist_rcs} and the elementary bound $\mathds{1}\{E \geq t\} \leq E/t$ to the generalized FDP \eqref{eq_gfdp}, we obtain
\begin{align}\label{eq_gfdp_bound}
    \textrm{FDP}^{\textrm{g}}(\mathcal{R}, \mathcal{Y}^{\textrm{test}})
    &= \frac{\sum_{j=1}^{m} L_j\mathds{1}\{j\in\mathcal{R}\}}{\max\{1, |\mathcal{R}|\}}\nonumber\\
    &\leq \hspace{-1mm} \frac{1}{|\mathcal{R}|} \hspace{-1mm} \sum_{j=1}^m L_j  \frac{\alpha |\mathcal{R}|}{m} E_j^{\textrm{g}}
    = \frac{\alpha}{m}\sum_{j=1}^{m} L_j E_j^{\textrm{g}}.
\end{align}
Dividing both sides by $\alpha$ yields the level-uniform bound
\begin{align}\label{eq_level_uniform_rcs}
    \frac{\textrm{FDP}^{\textrm{g}}(\mathcal{R}, \mathcal{Y}^{\textrm{test}})}{\alpha} \leq \frac{1}{m}\sum_{j=1}^{m} L_j E_j^{\textrm{g}}.
\end{align}
Taking expectations and applying \eqref{eq_risk_adj_eval} gives
\begin{align}\label{eq_exp_bound_rcs}
    \mathbb{E}\left[\frac{\textrm{FDP}^{\textrm{g}}(\mathcal{R}, \mathcal{Y}^{\textrm{test}})}{\alpha}\right] \leq \frac{1}{m}\sum_{j=1}^{m} \mathbb{E}\big[L_j E_j^{\textrm{g}}\big] \leq 1.
\end{align}
Since the right-hand side of \eqref{eq_level_uniform_rcs} does not depend on $\alpha$, substituting $\alpha = \alpha^{\textrm{PH-RCS}}$ and $\mathcal{R} = \mathcal{R}^{\textrm{PH-RCS}}$ yields \eqref{eq_ph_rcs_guarantee}.

\subsection{Priority-Weighted Selection} \label{apdx_proof_weighted}
In many applications, test inputs are not equally important. For instance, in drug discovery certain molecular scaffolds may warrant closer examination, or in clinical settings high-risk patients may deserve priority \cite{wang2022false, ignatiadis2024asymptotic}. Assigning higher weights to more promising inputs also serves as a power-boosting mechanism by concentrating the testing budget where it is most needed \cite{wang2022false, ignatiadis2024asymptotic}. To incorporate such preferences while preserving the post-hoc reliability guarantee, we allow the user to assign a nonnegative weight $w_j$ to each test input $X_{n+j}$, subject to the budget constraint
\begin{align}\label{eq_weight_budget}
    \sum_{j=1}^{m} w_j \leq m.
\end{align}

The weights are applied to risk-adjusted e-variables $\{E_j^{\textrm{g}}\}_{j=1}^m$ (see \eqref{eq_risk_adj_eval}), yielding weighted e-variables \cite{wang2022false, ignatiadis2024asymptotic}
\begin{align}\label{eq_weighted_eval}
    \tilde{E}_j^{\textrm{g}} = w_j E_j^{\textrm{g}}, \quad j = 1, \ldots, m.
\end{align}
Algorithm \ref{algo_PHCS} is then applied with the e-variables $\{E_j^{\textrm{g}}\}_{j=1}^m$ replaced by the quantities $\{\tilde{E}_j^{\textrm{g}}\}_{j=1}^m$ in \eqref{eq_weighted_eval}, producing a weighted candidate path and a utility-selected operating point $(\mathcal{R}^{\textrm{PH-RCS}}, \alpha^{\textrm{PH-RCS}})$. The following result confirms that the post-hoc guarantee of Theorem \ref{thm_ph_rcs} is preserved under weighting.

\begin{corollary}[\textbf{Weighted post-hoc guarantee}]\label{cor_weighted}
    Let $\{w_j\}_{j=1}^m$ be nonnegative weights satisfying \eqref{eq_weight_budget}. Then, under the same assumption of Theorem \ref{thm_ph_rcs}, the output of Algorithm \ref{algo_PHCS} with $E_j^{\text{\rm{g}}}$ replaced by $\tilde{E}_j^{\text{\rm{g}}}$ in \eqref{eq_weighted_eval} satisfies the inequality
    \begin{equation}\label{eq_weighted_guarantee}
        \mathbb{E}\left[\frac{\text{\rm{FDP}}^{\text{\rm{g}}}(\mathcal{R}^{\text{\rm{PH-RCS}}}, \mathcal{Y}^{\text{\rm{test}}})}{\alpha^{\text{\rm{PH-RCS}}}}\right] \leq 1.
    \end{equation}
\end{corollary}

\textit{Proof:}
% See Appendix \ref{apdx_proof_weighted}.
By the proof of Theorem \ref{thm_ph_rcs}, the level-uniform bound \eqref{eq_level_uniform_rcs} holds under \eqref{eq_risk_adj_eval}. Replacing $E_j^{\textrm{g}}$ by $\tilde{E}_j^{\textrm{g}} = w_j E_j^{\textrm{g}}$ in Algorithm \ref{algo_PHCS} replaces each summand $L_j E_j^{\textrm{g}}$ in \eqref{eq_level_uniform_rcs} by $w_j L_j E_j^{\textrm{g}}$, giving
\begin{align}\label{eq_weighted_level_uniform}
    \frac{\textrm{FDP}^{\textrm{g}}(\mathcal{R}, \mathcal{Y}^{\textrm{test}})}{\alpha} \leq \frac{1}{m}\sum_{j=1}^{m} w_j L_j E_j^{\textrm{g}}.
\end{align}
Taking expectations and applying the budget constraint \eqref{eq_weight_budget} yields
\begin{align}
    \mathbb{E}\left[\frac{\textrm{FDP}^{\textrm{g}}(\mathcal{R}, \mathcal{Y}^{\textrm{test}})}{\alpha}\right] \leq \frac{1}{m}\sum_{j=1}^{m} w_j \mathbb{E}[L_j E_j^{\textrm{g}}] \leq \frac{1}{m}\sum_{j=1}^{m} w_j \leq 1.
\end{align}
Since the right-hand side of \eqref{eq_weighted_level_uniform} is free of $\alpha$, substituting $\alpha = \alpha^{\textrm{PH-RCS}}$ and $\mathcal{R} = \mathcal{R}^{\textrm{PH-RCS}}$ yields \eqref{eq_weighted_guarantee}.

When the loss is selected as $L_j = \mathds{1}\{Y_{n+j} \leq c_j\}$, the guarantee \eqref{eq_weighted_guarantee} recovers Theorem \ref{theo_posthoc}.

{\color{blue}
\subsection{Details of the Baseline Methods}\label{apx_baselines}
% Beside conventional CS \cite{jin2023selection}, which targets a pre-specified FDR level $\alpha_{\textrm{max}}$, we compare PH-CS against three data-driven baselines that also address the post-hoc objective \eqref{eq_utility_obj}.
This appendix details the three baselines introduced in Sec.~\ref{sec_exp_baselines}.
All three baselines build on the nested family of selected sets obtained from the conformal p-variables \eqref{eq_conf_pval} via the BH procedure \cite{benjamini1995controlling, benjamini2001control}, i.e.,
\begin{align}\label{eq_bh_Rk}
    \mathcal{R}^{\textrm{BH}}_k \hspace{-0.5mm}=\hspace{-0.5mm} \{j\in\{1,\ldots,m\} \hspace{-0.5mm}: \hspace{-0.5mm} P_j\leq P_{(k)}\}, ~ k=0,1,\ldots,m,
\end{align}
where $P_{(1)}\leq\cdots\leq P_{(m)}$ are the sorted conformal p-variables \eqref{eq_conf_pval} and $\mathcal{R}^{\textrm{BH}}_0 = \varnothing$. This family is the BH counterpart of the e-BH path \eqref{eq_Rk_def} used by PH-CS, and spans the same range of set sizes. The three baselines report an FDP estimate for the selected set, and differ only in how this estimate is constructed.

\subsubsection{Na\"ive PH-CS}
This baseline applies conventional CS across a range of nominal levels and retains the most useful operating point. For each candidate set $\mathcal{R}^{\textrm{BH}}_k$, inverting the BH self-consistency condition \eqref{eq_BH_k} gives the nominal level associated with the set, which we adopt as its FDP estimate
\begin{align}\label{eq_naive_alpha}
    \hat{\alpha}^{\textrm{n}}(\mathcal{R}^{\textrm{BH}}_k) = \min\Big\{1, \frac{mP_{(k)}}{k}\Big\},
\end{align}
where, as in \eqref{eq_alpha_k_def}, the minimum ensures that the estimate does not exceed $1$. Na\"ive PH-CS then solves the problem
\begin{align}\label{eq_naive_choice}
    \mathcal{R}^{\textrm{n}} = \arg\max_{\mathcal{R}\in\{\mathcal{R}^{\textrm{BH}}_0, \ldots, \mathcal{R}^{\textrm{BH}}_m\}} U\big(r(\mathcal{R}), \hat{\alpha}^{\textrm{n}}(\mathcal{R})\big),
\end{align}
and reports the associated level $\alpha^{\textrm{n}} = \hat{\alpha}^{\textrm{n}}(\mathcal{R}^{\textrm{n}})$. Since this level is selected after inspecting the data, it is no longer fixed in advance, and the fixed-level guarantee of Theorem \ref{theo_conv_fdr} does not apply.

\subsubsection{Split PH-CS}
This baseline restores a fixed-level guarantee through data splitting. The calibration set is partitioned into two disjoint halves $\mathcal{D}^{\textrm{cal}} = \mathcal{D}_1\cup\mathcal{D}_2$, which are used in two steps: \textit{(i)} the Na\"ive PH-CS procedure \eqref{eq_naive_choice} is applied with calibration restricted to dataset $\mathcal{D}_1$, guaranteeing an FDR level $\alpha^{\textrm{s}}$; \textit{(ii)} conventional CS \eqref{eq_BH_set} is run on dataset $\mathcal{D}_2$ with $\alpha^{\textrm{s}}$ held fixed as the target level and as the reported FDP estimate. Because the level $\alpha^{\textrm{s}}$ is selected from $\mathcal{D}_1$ and is therefore independent of $\mathcal{D}_2$, Theorem \ref{theo_conv_fdr} applies conditionally on the selected level. This reliability is obtained at the cost of calibrating on only a fraction of the available data.

\subsubsection{PAC PH-CS}
Rather than reporting a level whose reliability holds only on average, this approach upper bounds, with high probability, the realized FDP \eqref{eq_FDP_def} of every candidate set at once \cite{song2026everywhere, gazin2024transductive}. Specifically, it provides an FDP estimate $\hat{\alpha}^{\textrm{PAC}}(\mathcal{R}^{\textrm{BH}}_k)$ that, by construction, upper bounds the realized FDP of all candidate sets $\mathcal{R}_k^{\textrm{BH}}$ in \eqref{eq_bh_Rk} simultaneously as per the PAC requirement
\begin{align}\label{eq_sb_guarantee}
    \mathbb{P}\Big(\textrm{FDP}(\mathcal{R}^{\textrm{BH}}_k, \mathcal{Y}^{\textrm{test}})\leq \hat{\alpha}^{\textrm{PAC}}(\mathcal{R}^{\textrm{BH}}_k),~\forall k\Big)\geq 1-\eta.
\end{align}
Because this bound holds for all $k$ at once, the operating point can be chosen after inspecting the data. Accordingly, this baseline selects the set
\begin{align}\label{eq_sb_choice}
    \mathcal{R}^{\textrm{PAC}} = \arg\max_{\mathcal{R}\in\{\mathcal{R}^{\textrm{BH}}_0, \ldots, \mathcal{R}^{\textrm{BH}}_m\}} U\big(r(\mathcal{R}), \hat{\alpha}^{\textrm{PAC}}(\mathcal{R})\big),
\end{align}
and reports the associated bound $\alpha^{\textrm{PAC}} = \hat{\alpha}^{\textrm{PAC}}(\mathcal{R}^{\textrm{PAC}})$. We emphasize that the guarantee \eqref{eq_sb_guarantee} is of a different nature from the post-hoc reliability \eqref{eq_post_FDP_first} targeted by PH-CS. It controls the realized FDP with high probability rather than in expectation, requiring the specification of the additional confidence level $\eta$. Furthermore, being valid simultaneously over all sets, the estimate $\hat{\alpha}^{\textrm{PAC}}(\mathcal{R}^{\textrm{PAC}})$ may be loose at the single selected operating point.

The estimate $\hat{\alpha}^{\textrm{PAC}}$ is constructed as follows, following \cite{song2026everywhere, gazin2024transductive}.
The realized FDP \eqref{eq_FDP_def} of the set $\mathcal{R}^{\textrm{BH}}_k$ can be written as
\begin{align}\label{eq_sb_numerator}
    \textrm{FDP}(\mathcal{R}^{\textrm{BH}}_k, \mathcal{Y}^{\textrm{test}}) = \frac{m\hat{F}^0_m(P_{(k)})}{\max\{1, |\mathcal{R}^{\textrm{BH}}_k|\}},
\end{align}
where $\hat{F}^0_m(t)=(\sum_{j=1}^m \mathds{1}\{Y_{n+j}\leq c_j\} \mathds{1}\{P_j\leq t\})/m$ is the empirical distribution of the null p-variables. Only the numerator is unknown, as it depends on the unobserved test labels.

The distribution of $\hat{F}^0_m$, however, is known. A conformal p-variable \eqref{eq_conf_pval} enters only through the ranks of the conformity scores, and under exchangeability all orderings of the scores are equally likely. The distribution of $\hat{F}^0_m$ is therefore the same as if the scores were i.i.d. uniform variables, independently of the data \cite{gazin2024transductive}. Repeatedly sampling the p-variables from this distribution shows how far $\hat{F}^0_m$ can rise, and an envelope $G:[0,1]\to[0,1]$ exceeded with probability at most $\eta$ then bounds it with probability at least $1-\eta$ \cite{song2026everywhere},
\begin{align}\label{eq_sb_envelope}
    \mathbb{P}\big(\hat{F}^0_m(t)\leq G(t),~\forall t\in[0,1]\big)\geq 1-\eta.
\end{align}
The shape of $G$ depends on the deviation statistic used, ranging from the linear Dvoretzky-Kiefer-Wolfowitz (DKW) envelope \cite{gazin2024transductive} to the tighter, variance-adaptive higher-criticism (HC) envelope \cite{song2026everywhere}.

Substituting the bound \eqref{eq_sb_envelope} into \eqref{eq_sb_numerator} replaces the unknown numerator $m\hat{F}^0_m(P_{(k)})$ by $mG(P_{(k)})$, giving the FDP estimate
\begin{align}\label{eq_sb_alpha}
    \hat{\alpha}^{\textrm{PAC}}(\mathcal{R}^{\textrm{BH}}_k) = \frac{mG(P_{(k)})}{\max\{1, |\mathcal{R}^{\textrm{BH}}_k|\}}.
\end{align}
Since \eqref{eq_sb_envelope} holds for all $t\in[0,1]$ simultaneously, evaluating it at $t=P_{(k)}$ for every $k$ and substituting into \eqref{eq_sb_numerator} yields the guarantee \eqref{eq_sb_guarantee}
% stated in Sec. \ref{sec_exp_baselines}
, where the same probability $1-\eta$ is inherited from \eqref{eq_sb_envelope}.
}

{\color{blue}
\subsection{Additional Experimental Results}\label{apdx_add_exp}
This appendix collects additional experimental results supplementing Sec.~\ref{sec_exp}, under the constrained-size utility \eqref{eq_utility_size_first} and the additive trade-off utility \eqref{eq_u_add}, and across alternative regressors.}

\subsubsection{Constrained-Size Utility}\label{apdx_Csize}
We provide additional results under the constrained-size utility \eqref{eq_utility_size_first}, supplementing Sec. \ref{sec_exp_csize} with the heteroscedastic synthetic data, the Recruitment dataset, and the Shuttle dataset. As shown in Fig. \ref{fig_Csize_apdx}, the same conclusion holds across all three settings: PH-CS always satisfies the minimum-size requirement, whereas CS falls below it in a non-negligible fraction of runs{\color{blue}, while PH-CS also attains a higher power at a comparable FDP}.

\subsubsection{Additive Trade-off Utility}\label{apdx_add}
We consider the additive trade-off utility in \eqref{eq_u_add}, instantiated as $U(r, \alpha) = r - \lambda \alpha$ with $r=|\mathcal{R}|$. PH-CS selects, for each test batch, the candidate set that maximizes this utility, and CS is matched to its average declared level as in Sec.~\ref{sec_exp_csize}. To illustrate the effect of the trade-off parameter $\lambda$, Fig.~\ref{fig_real_utility} reports the realized utility, selected set size, and FDP on the Shuttle dataset for $\lambda = 12800$ and $\lambda = 14000$.

PH-CS achieves a larger average realized utility than CS, although not necessarily for every individual realization. This is expected because PH-CS maximizes the utility in \eqref{eq_set_rule} using the declared level $\alpha^{\textrm{PH-CS}}$, whereas the plotted utility in \eqref{eq_utility_obj} is evaluated using the true realized FDP, so finite-sample deviations between the two can occasionally favor CS. As $\lambda$ increases from the top row to the bottom row, the utility places more emphasis on reducing the FDP. Accordingly, PH-CS becomes more conservative, yielding a smaller selected set and a smaller realized FDP, while the change for CS is induced only through the matched target level. The realized utility remains comparable across the two choices of $\lambda$, confirming that $\lambda$ provides a direct way to adjust the size-reliability trade-off within the same utility framework.

We further report results on additional datasets. For the synthetic data, we use $U(r,\alpha)=r-\lambda\alpha$ with $\lambda=500$. For the Recruitment dataset, we use $U(r,\alpha)=\log r-\lambda\log(1/(1-\alpha))$ with $\lambda=15$, and for the Musk dataset, we use $U(r,\alpha)=r-\lambda\alpha$ with $\lambda=1690$. {\color{blue}On both real datasets we set the score hyperparameter to $\gamma=5$ rather than the value used in Sec.~\ref{sec_exp}. As shown in Fig.~\ref{fig_add_apdx}, PH-CS achieves a larger average realized utility than CS across all four settings, consistent with the Shuttle results above, so its advantage carries over to a different $\gamma$.}

\subsubsection{Alternative Regressors} \label{apdx_more_exp_syn}
We present additional synthetic-data results obtained using random forest and support vector regression, supplementing the results in Sec. \ref{sec_exp} based on gradient boosting.
Under the constrained-size utility, Fig. \ref{fig_syn_Csize_rf_apdx} shows that the same qualitative conclusion holds for random forest: PH-CS consistently satisfies the minimum-size requirement in each realization, whereas CS violates this constraint in a non-negligible fraction of runs. {\color{blue}PH-CS also attains a higher power while keeping the FDP comparable.}
% Under the additive trade-off utility, Fig. \ref{fig_syn_Paware_svm_apdx} confirms that PH-CS continues to achieve larger average realized utility than CS when using support vector regression.
{\color{blue}Under a power-aware weighted utility, where each test input is weighted by its estimated true-discovery probability $\pi_j$ via \eqref{eq_pi_gaussian}, Fig.~\ref{fig_syn_Paware_svm_apdx} confirms that PH-CS achieves higher average utility and higher power than CS at a comparable FDP when using support vector regression.}
Overall, these results indicate that the empirical behavior of PH-CS is robust to the choice of regressor.

\begin{figure}[t]
    \centering
    {
	\includegraphics[width = 0.3\textwidth]{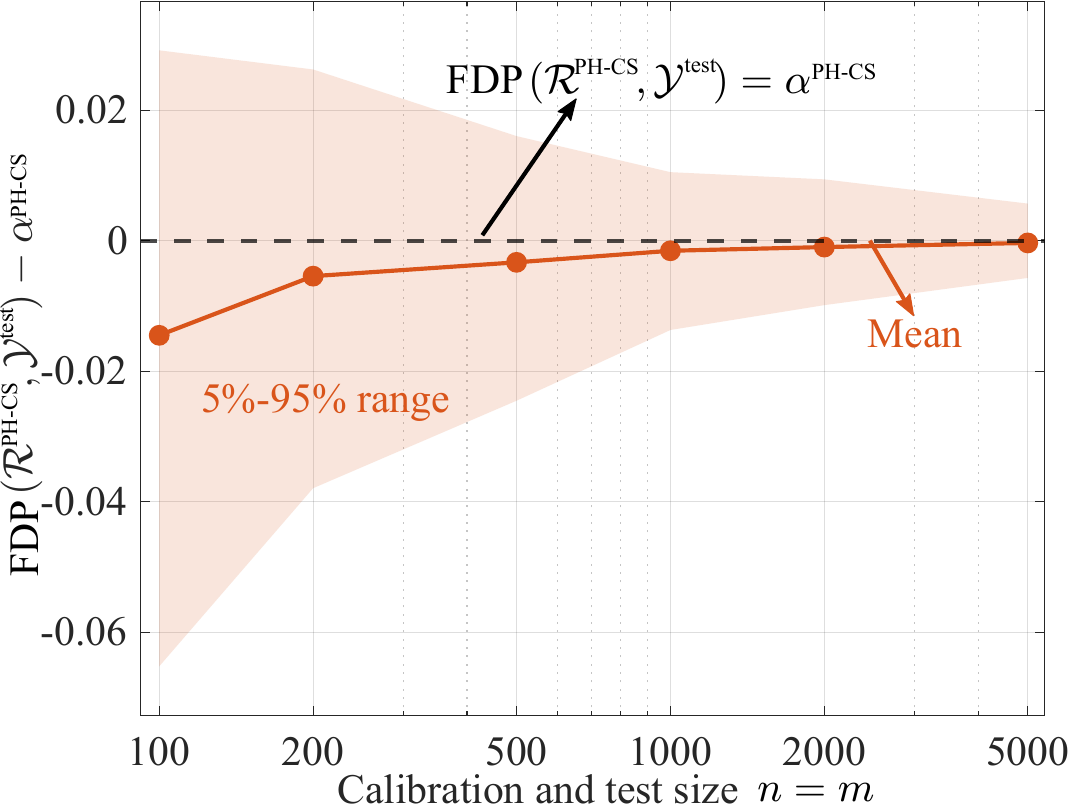}
    }
    \vspace{-1mm}
    \caption{{\color{blue}Gap $\textrm{FDP}(\mathcal{R}^{\textrm{PH-CS}}, \mathcal{Y}^{\textrm{test}}) - \alpha^{\textrm{PH-CS}}$ versus the calibration and test size $n=m$, on the synthetic data with homoscedastic noise under the constrained-size utility \eqref{eq_utility_size_first}. The solid line is the mean over $200$ trials and the shaded band is the $5\%$--$95\%$ range.}}
    \vspace{-3.5mm}
    \label{fig_asymp}
\end{figure}
\begin{figure}[t]
    \centering
    {
	\includegraphics[width = 0.3\textwidth]{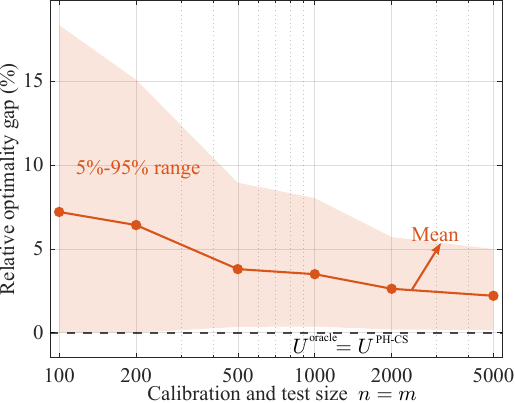}
    }
    \vspace{-1mm}
    \caption{{\color{blue}Relative optimality gap $(U^{\textrm{oracle}}-U^{\textrm{PH-CS}})/U^{\textrm{oracle}}$ between PH-CS and an oracle that selects on the same candidate path using the true FDP, versus the calibration and test size $n=m$, on the synthetic data with homoscedastic noise under the additive utility $U=r-\lambda\alpha$ with $\lambda=5m$. The solid line is the mean over $200$ trials and the shaded band is the $5\%$--$95\%$ range.}}
    \vspace{-3.5mm}
    \label{fig_opt}
\end{figure}

{\color{blue}
\subsection{Empirical Validation of the Asymptotic Behavior}\label{apdx_asymp_behavior}
This appendix illustrates the large-sample behavior of PH-CS on the synthetic data of Sec.~\ref{sec_exp} with homoscedastic noise. In both experiments below, the calibration and test sets grow jointly, with $n=m$ ranging over $\{100, 200, 500, 1000, 2000, 5000\}$ and $200$ trials per size, so that the joint growth matches the asymptotic regime being illustrated.

\subsubsection{Asymptotic Reliability}\label{apdx_asymp}
This experiment illustrates Proposition~\ref{prop_asymp} under the constrained-size utility \eqref{eq_utility_size_first}, with the size floor set to $r_{\min}=0.5m$ to keep the operating point comparable across sizes. We track the gap $\textrm{FDP}(\mathcal{R}^{\textrm{PH-CS}}, \mathcal{Y}^{\textrm{test}}) - \alpha^{\textrm{PH-CS}}$ between the realized FDP and the reported level. As shown in Fig.~\ref{fig_asymp}, the $5\%$--$95\%$ range narrows toward a point as $n=m$ grows, so the scatter of the realized FDP around $\alpha^{\textrm{PH-CS}}$ in Fig.~\ref{fig_FDP_estimate} is a finite-sample effect. Its upper edge, positive at small samples, falls below any fixed margin as the sample grows, in agreement with the asymptotic bound of Proposition~\ref{prop_asymp}. The range concentrates around a small negative mean gap, showing that $\alpha^{\textrm{PH-CS}}$ is on average a slightly conservative estimate of the realized FDP, consistent with the average guarantee of Theorem~\ref{theo_posthoc}.

\subsubsection{Optimality Gap}\label{apdx_opt}
This experiment compares PH-CS with an oracle that searches the same score-ordered candidate path \eqref{eq_Rk_def} but selects the operating point using the true FDP rather than the estimate $\alpha^{\textrm{PH-CS}}$, with both utilities evaluated at the true FDP. We use the additive utility $U(r,\alpha)=r-\lambda\alpha$ with $\lambda=5m$, so that the weight scales with the problem size. Fig.~\ref{fig_opt} reports the relative gap $(U^{\textrm{oracle}}-U^{\textrm{PH-CS}})/U^{\textrm{oracle}}$. As $n=m$ grows, both the mean gap and its $5\%$--$95\%$ range shrink, indicating that the operating point selected by PH-CS through $\alpha^{\textrm{PH-CS}}$ approaches the one the oracle selects with the true FDP. On this well-separated data, where the score keeps the estimate close to the realized FDP, the utility attained by PH-CS is therefore close to the best achievable on the candidate path.
}

\else
\makeatletter
\newcommand{\fakeref}[2]{\phantomsection\def\@currentlabel{#2}\label{#1}}
\makeatother
\fakeref{apdx_Taylor}{A}
\fakeref{apx_prof_posthoc}{B}
\fakeref{apx_prop_asymp}{C}
\fakeref{apdx_proof_rcs}{D}
\fakeref{apdx_proof_weighted}{E}
\fakeref{apx_baselines}{F}
\fakeref{apdx_add_exp}{G}
% \fakeref{apdx_Csize}{G.1}
% \fakeref{apdx_add}{G.2}
% \fakeref{apdx_more_exp_syn}{G.3}
\fakeref{apdx_asymp_behavior}{H}
% \fakeref{apdx_asymp}{H.1}
% \fakeref{apdx_opt}{H.2}

\fi

\end{document}